\documentclass{article}
\PassOptionsToPackage{table}{xcolor}
\usepackage{PRIMEarxiv}

\usepackage[utf8]{inputenc}
\usepackage[T1]{fontenc}
\usepackage[round]{natbib}
\usepackage{hyperref}
\usepackage{url}
\usepackage{amsmath}
\usepackage{amssymb}
\usepackage{amsthm}
\usepackage{booktabs}
\usepackage{tabularx}
\usepackage{wrapfig}
\usepackage{multirow}
\usepackage{enumitem}
\usepackage[table]{xcolor}
\usepackage{bbm}
\usepackage{dsfont}
\usepackage{nicefrac}
\usepackage{microtype}
\usepackage{graphicx}
\usepackage{placeins}
\usepackage{capt-of}
\usepackage{flafter}
\usepackage{algorithm}
\usepackage{algpseudocode}
\graphicspath{{figs/}}

\hypersetup{colorlinks=true,linkcolor=blue!60!black,citecolor=blue!60!black,urlcolor=blue!60!black}

\setlist[itemize]{noitemsep}

\theoremstyle{plain}
\newtheorem{theorem}{Theorem}
\newtheorem{lemma}{Lemma}
\newtheorem{corollary}{Corollary}
\newtheorem{proposition}{Proposition}
\theoremstyle{definition}
\newtheorem{definition}{Definition}

\newtheorem{problem}{Problem}
\theoremstyle{remark}

\newcommand{\R}{\mathbb{R}}
\newcommand{\E}{\mathbb{E}}
\newcommand{\Hcal}{\mathcal{H}}

\newcommand{\op}{\mathrm{op}}

\newcommand{\1}{\mathbf{1}}
\DeclareMathOperator*{\argmin}{arg\,min}

\title{Learning Higher-Order Interactions from \\Incomplete Spatiotemporal Data with \\Mixtures of Hypergraph Laplacians}

\author{
  Keshu Wu$^{1,2}$, Sixu Li$^1$, Zihao Li$^3$, Zhiwen Fan$^1$, Xiaopeng Li$^2$, Yang Zhou$^{1}$\thanks{Corresponding author: \texttt{yangzhou295@tamu.edu}}\\
  $^1$Texas A\&M University, 
  $^2$University of Wisconsin-Madison,
  $^3$Marquette University\\
}

\begin{document}

\maketitle

\begin{abstract}
Traffic sensor data exhibit strong temporal and cross-sensor dependencies, yet long gaps erase both a sensor’s local history and the joint observations needed to learn its relationships with other sensors. We introduce MoHL, a mixture of hypergraph Laplacians that learns group-wise relational structure directly from incomplete spatiotemporal data. MoHL discovers overlapping sensor groups from observed co-variation, constructs a Laplacian component for each interaction order, and learns their mixture through holdouts that mimic the missingness pattern. By sharing weights across group-induced couplings, MoHL pools scattered evidence instead of estimating each coupling independently. We characterize the induced weight-sharing structure and analyze the conditions under which sharing reduces estimation risk. Experiments on three traffic networks show that higher-order groups are most beneficial under long gaps, while support-matched comparisons demonstrate the advantage of shared weights when relational evidence is scarce. MoHL achieves the lowest error among the evaluated baseline methods under day-long gaps, with orders of magnitude fewer learned parameters than deep imputers. These results show how group-wise sharing can sustain reliable reconstruction precisely when missingness depletes the evidence needed to learn cross-sensor dependence.
\end{abstract}


\section{Introduction}
\label{sec:intro}

Traffic networks are instrumented with roadside sensors, such as inductive loops and radar detectors, that continuously measure speed, flow and occupancy~\citep{klein2006traffic}. These measurements support traffic monitoring, forecasting and control~\citep{li2018dcrnn_traffic}, as well as safety analysis and data-driven automated driving~\citep{wu2025ai2, wang2026generative}, but they are often incomplete because of equipment faults, communication failures and other disruptions~\citep{fekade2017probabilistic, chen2018spatial, qu2009ppca, li2014missing}. The data are strongly structured: each sensor follows recurring daily patterns, and sensors exposed to common demand, congestion or traffic propagation vary together~\citep{li2020tensor, CHEN201966, chen2022latc}. This dependence also acts at the level of groups, in which several sensors depart from their typical behavior in a coordinated way~\citep{benson2016higher, battiston2021physics}. Cross-sensor relations matter most during long gaps, when a sensor's own history is unavailable. Yet the same gaps remove the joint observations from which these relations must be learned. The central challenge is therefore to learn reliable relational structure from incomplete evidence.

A natural way to encode relations is the road graph, whose edges reflect physical connectivity or proximity~\citep{li2018dcrnn_traffic, yu2018stgcn}. Road topology, however, can miss dependencies induced by shared demand and changing traffic conditions, which has motivated methods that learn spatial relations from data~\citep{wu2019graphwavenet, wu2020mtgnn, bai2020adaptive, xu2024hierarchical, wang2025task}. Learning a free weight for every sensor pair is the most flexible option, but it needs enough joint observations for hundreds or thousands of couplings. Structured missingness removes these joint observations. We therefore learn relations at the level of groups, so that related couplings share statistical evidence.

Hypergraphs represent such group relations directly. A hyperedge joins sensors that are expected to vary together and governs all of their pairwise couplings through one group weight~\citep{zhou2006hypergraph, feng2019hgnn, bai2021hypergraph}. For the quadratic Laplacian penalty used here, a hyperedge is exactly a graph whose edges share one weight. A group therefore pools the evidence of all its member pairs, which reduces estimation variance when evidence is scarce. This pooling pays off whenever the groups are accurate, so MoHL discovers them from the data and learns how much to trust each order.

We introduce \textbf{MoHL}, a \emph{mixture of hypergraph Laplacians} that learns these group relations from incomplete data (Figure~\ref{fig:framework}). MoHL removes each sensor's daily profile, estimates pairwise affinities from the observed residuals, and grows overlapping groups of increasing order around each sensor. Each order contributes one Laplacian component with one shared weight, and the weights are learned by masking observed cells in gaps that mimic the missingness pattern. The mixture regularizes a spatiotemporal reconstruction that is solved as one positive-definite system, and an evidence-adaptive gate adjusts the regularization to local missingness.

\begin{figure}[t]
\centering
\includegraphics[width=\linewidth]{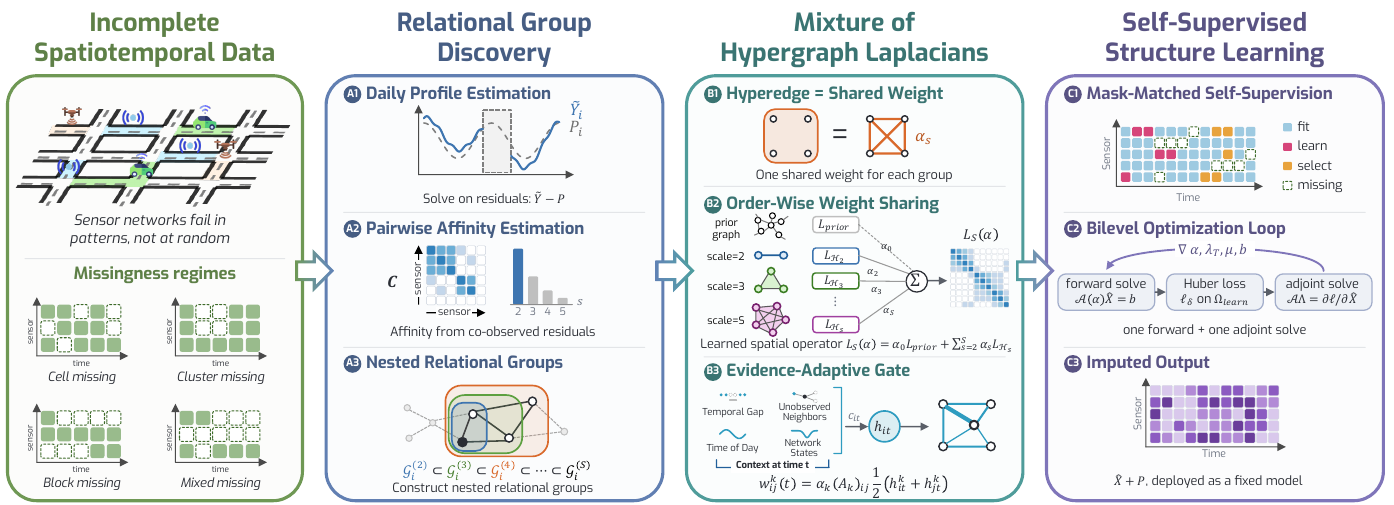}
\caption{\textbf{Overview of MoHL}: affinities computed from incomplete observations define relational groups; a learned mixture of their hypergraph Laplacians enters a spatiotemporal imputation solve.}
\label{fig:framework}
\end{figure}

Our analysis asks when and why sharing helps. We characterize the pairwise couplings induced by the mixture, show that the shared-weight estimator is Bayes-optimal at the true weights under a Gaussian group model, and derive the sample size at which free weights overtake shared ones. To isolate sharing itself, we compare MoHL's shared weights with freely learned pairwise weights on the same support. Higher-order groups help most during long gaps, and on PEMS-BAY shared weights outperform free ones in every regime and at both subnetwork sizes. The advantage depends on discovering meaningful groups, which we verify on hidden data.

\paragraph{Contributions.}
\textbf{(1) Group-wise structure from incomplete data.} MoHL discovers overlapping multiscale sensor groups from pairwise evidence alone and represents them by a mixture of hypergraph Laplacians. \textbf{(2) A learnable mixture across interaction orders.} MoHL assigns one Laplacian component to each interaction order and learns their weights end to end, with one adjoint solve per step. \textbf{(3) A support-matched test of group sharing.} Theory and experiments on data drawn from the model show that group-wise sharing is most effective when relational evidence is scarce, and support-matched experiments on PEMS-BAY show that it beats free weights in every regime. \textbf{(4) Interpretable and resilient structure.} On the road network, the discovered groups join detectors across freeways that share their commuting peak, and they keep MoHL accurate when related detectors fail together.

\section{Related Work}
\label{sec:related}
\label{app:related}

\paragraph{Spatiotemporal imputation.} PEMS-BAY~\citep{li2018dcrnn_traffic} and the PEMS flow networks are standard test beds. Pooling and interpolation methods use no structure beyond a neighborhood mean. Low-rank tensor completion~\citep{6138863, CHEN2020102673, chen2022latc, NIE2023104190} minimizes a nuclear norm over the sensor--time array and is strong under scattered dropout. Deep imputers based on graphs~\citep{cini2022grin, wu2021ignnk}, recurrence~\citep{cao2018brits} or attention~\citep{du2023saits, marisca2022spin, nie2024imputeformer}, general time-series backbones~\citep{wu2023timesnet} and generative models~\citep{tashiro2021csdi, liu2023pristi} express richer dependence at the cost of training for each deployment~\citep{wang2025survey}. Graph-regularized tensor completion~\citep{li2020tensor} is the closest prior method in spirit and our most informative classical baseline. Most methods are evaluated on one missingness pattern at one rate; we evaluate one model across six regimes that differ in which evidence survives.

\paragraph{Hypergraph regularization.} Pairwise Laplacian regularization~\citep{ando2006learning, yin2015laplacian} penalizes differences between neighbors. The hypergraph Laplacian~\citep{zhou2006hypergraph} penalizes within-group variance, and hypergraph neural networks~\citep{feng2019hgnn, yadati2019hypergcn, bai2021hypergraph, chien2022allset} pass messages over hyperedges in the same spirit, usually with the hypergraph given. Hypergraphs now model higher-order relations from multi-agent traffic interactions~\citep{wu2025hypergraph} to retrieval for language models~\citep{wu2026knowledge}. Hypergraph learning methods~\citep{10366871, 10.1145/3605776} infer or embed hypergraph structure from fully observed signals. Our starting point is the pairwise-expansion identity~\citep{zhou2006hypergraph, agarwal2006higher, chitra2019random}: as a quadratic form, every hypergraph Laplacian is a graph Laplacian in which the pairs of a group share one weight. This identity is the mechanism of sharing, and our support-matched comparison isolates its effect.

\paragraph{Structured parameter sharing and implicit differentiation.} Sharing parameters across related components trades bias for variance, from hierarchical shrinkage~\citep{efron1975data} and group-structured penalties~\citep{yuan2006model} to weight sharing in convolutional and graph networks~\citep{lecun1998gradient, kipf2016semi}. We measure this trade-off directly, comparing two parameterizations of one operator on one support under the same learning procedure. Gradients through the solution of an optimization problem~\citep{bengio2000gradient, pedregosa2016hyperparameter, franceschi2018bilevel, lorraine2020optimizing, amos2017optnet, blondel2022efficient} are usually taken for a handful of hyperparameters; because our estimator solves one symmetric positive-definite system, a single adjoint solve~\citep{plessix2006review} gives the gradient for every structural weight, gate parameter and scalar at once.

\section{Problem Statement and Preliminaries}
\label{sec:setup}

\subsection{Problem statement}
\label{sec:setup-problem}

\paragraph{Data and observation model.} Consider a network of $N$ sensors that record a traffic variable, such as speed or flow, at $T$ equally spaced times (every five minutes in our data). The true values form a matrix $X^*\in\R^{N\times T}$. Row $i$ is the time series of sensor $i$, column $t$ is a snapshot of the network at time $t$, and each entry $(i,t)$ is a \emph{cell}. A binary mask $M\in\{0,1\}^{N\times T}$ records which cells were measured, with $M_{i,t}=1$ if sensor $i$ reported at time $t$ and $M_{i,t}=0$ otherwise. It partitions the cells into observed and missing ones,
\begin{equation}
\Omega := \{(i,t) : M_{i,t}=1\}, \qquad \Omega^c := \{(i,t) : M_{i,t}=0\}.
\label{eq:obs-sets}
\end{equation}
On an observed cell we obtain a noisy reading
\begin{equation}
Y_{i,t} = X^*_{i,t} + \xi_{i,t}, \qquad (i,t)\in\Omega,
\label{eq:obs-model}
\end{equation}
where the measurement noise $\xi_{i,t}$ has mean zero and sub-Gaussian tails with scale $\sigma$, that is, tails no heavier than those of a Gaussian with standard deviation $\sigma$. Each cell is observed with probability $\pi_{i,t}:=\E[M_{i,t}]$. We do not assume that missing cells are scattered. A single gap may cover hours or days of one sensor, or several neighboring sensors at once, and this structured missingness is what makes the problem hard. Finally, a symmetric prior adjacency $A\in\R_{\ge0}^{N\times N}$ with zero diagonal encodes known relations between sensors, with $A_{ij}$ large when sensors $i$ and $j$ are close on the road network.

\paragraph{Problem formulation.} The task is to recover the missing values from the observed ones.

\begin{problem}[Imputation under structured missingness]
\label{prob:imputation}
Given the mask $M$, the readings $\{Y_{i,t}\}_{(i,t)\in\Omega}$ and the prior adjacency $A$, construct an estimate $\widehat X^{\rm out}_{i,t}$ of $X^*_{i,t}$ for every missing cell $(i,t)\in\Omega^c$, in the original units of the readings, so as to minimize the mean absolute error on a set $\Omega_{\rm eval}\subseteq\Omega^c$ of withheld cells,
\begin{equation}
\mathrm{MAE}\bigl(\widehat X^{\rm out}\bigr) := \frac{1}{|\Omega_{\rm eval}|}\sum_{(i,t)\in\Omega_{\rm eval}}\bigl|\widehat X^{\rm out}_{i,t}-X^*_{i,t}\bigr| .
\label{eq:mae}
\end{equation}
The estimate may use only $M$, the observed readings and $A$: no true value on $\Omega^c$ and no external training data.
\end{problem}

The criterion \eqref{eq:mae} cannot be optimized directly, because its true values are unavailable to the method. Every parameter of the imputer must therefore be learned from the observed cells alone, and two difficulties follow. First, the cross-sensor structure that long gaps require is not given: the prior adjacency reflects road proximity, not necessarily shared traffic behavior. Second, the missingness also depletes the joint observations from which such structure could be learned.

\paragraph{Temporal gap and relational evidence.} Two quantities make these difficulties precise and separate an isolated missing cell from a long gap:
\begin{equation}
g_{i,t} := \min\bigl\{\,|t-t'| : M_{i,t'}=1\,\bigr\}, \qquad n_{ij} := \sum_{t=1}^{T} M_{i,t}\,M_{j,t} .
\label{eq:gap-evidence}
\end{equation}
The \emph{temporal gap} $g_{i,t}$ is the number of steps from cell $(i,t)$ to the nearest reading of the same sensor. It is $0$ on observed cells, small inside a short dropout and large inside a day-long gap. The \emph{relational evidence} $n_{ij}$ counts the times at which sensors $i$ and $j$ were observed together, the sample size for learning how they relate. A short gap leaves the sensor's own neighbors in time, so the estimate can interpolate. A long gap removes them, so the estimate must rely on other sensors. Yet long gaps also shrink $n_{ij}$. The regimes that need relational structure most thus provide the least evidence for it.

\paragraph{Notation.} For a matrix $X\in\R^{N\times T}$, $X_{i,:}$ is the time series of sensor $i$ and $x_t$ the snapshot at time $t$. We write $\langle U,V\rangle_F$ and $\|U\|_F$ for the Frobenius inner product and norm, $\odot$ for the entrywise product and $\1$ for the all-ones vector. Table~\ref{tab:notation} collects the main symbols.

\subsection{Graph and hypergraph Laplacians}
\label{sec:setup-laplacian}

\paragraph{Graph Laplacian.} A weighted graph on the sensors is described by a symmetric adjacency $B\in\R_{\ge0}^{N\times N}$ with zero diagonal, where $B_{ij}$ is the coupling strength between sensors $i$ and $j$. Its degree matrix $D_B:=\operatorname{diag}(B\1)$ holds the total coupling $\sum_jB_{ij}$ of each sensor, and its graph Laplacian is $D_B-B$. Applied to a matrix of time series, the Laplacian quadratic form measures disagreement across coupled sensors,
\begin{equation}
\langle X,(D_B-B)X\rangle_F = \sum_{i<j} B_{ij}\,\|X_{i,:}-X_{j,:}\|_2^2 = \sum_{i<j}\sum_{t=1}^{T} B_{ij}\,(X_{i,t}-X_{j,t})^2 .
\label{eq:graph-dirichlet}
\end{equation}
Used as a penalty, it pulls the series of strongly coupled sensors toward each other~\citep{zhou2004learning, belkin2006manifold}. A graph is, however, a pairwise description. A fixed graph such as road proximity can miss sensors that respond to the same traffic conditions, and a free coupling for every pair among $s$ related sensors requires $\binom{s}{2}$ parameters.

\paragraph{Hypergraph Laplacian.} A hypergraph $\Hcal=(V,E)$ describes groups directly~\citep{zhou2006hypergraph}. Its nodes are the sensors, $V=[N]:=\{1,\dots,N\}$, and each hyperedge $e\in E$ is a group of $|e|\ge2$ sensors with a weight $w_e\ge0$. We call $|e|$ the interaction \emph{order}; an ordinary graph is the order-two case. A hyperedge penalizes the deviation of its members from their group average,
\begin{equation}
\bar X_{e,t} := \frac{1}{|e|}\sum_{i\in e}X_{i,t}, \qquad \langle X, P_e X\rangle_F = \sum_{i \in e}\sum_{t=1}^{T}(X_{i,t} - \bar X_{e,t})^2 ,
\label{eq:hypergraph-penalty}
\end{equation}
which is produced by the hyperedge Laplacian $P_e$, and the hypergraph Laplacian $L_{\Hcal}$ sums these penalties with their weights:
\begin{equation}
P_e := I_e - \frac{1}{|e|}\,\1_e \1_e^\top \in\R^{N\times N}, \qquad L_{\Hcal} := \sum_{e \in E} w_e\, P_e .
\label{eq:hypergraph-lap}
\end{equation}
Here $\1_e\in\{0,1\}^N$ is the indicator vector of $e$ and $I_e:=\operatorname{diag}(\1_e)$. On the coordinates of $e$, $P_e$ is the orthogonal projector that removes the group average. It leaves motion shared by the whole group unpenalized, has rank $|e|-1$ and equals half the pairwise Laplacian when $e=\{i,j\}$.

\paragraph{Pairwise expansion.} The deviation from the group average equals the average disagreement over all pairs of the group,
\begin{equation}
\sum_{i\in e}\sum_{t=1}^{T}(X_{i,t}-\bar X_{e,t})^2 = \frac{1}{|e|}\sum_{\{i,j\}\subseteq e}\,\sum_{t=1}^{T}(X_{i,t}-X_{j,t})^2 .
\label{eq:pair-expansion}
\end{equation}
By \eqref{eq:graph-dirichlet}, $P_e$ is therefore the graph Laplacian of the adjacency $A^{(e)}:=|e|^{-1}(\1_e\1_e^\top-I_e)$, which places the same weight $1/|e|$ on every pair of $e$. Since $B\mapsto D_B-B$ is linear, every hypergraph Laplacian is a graph Laplacian, $L_{\Hcal}=D_{A_{\Hcal}}-A_{\Hcal}$, with induced adjacency
\begin{equation}
(A_{\Hcal})_{ij} = \sum_{e\,\ni\, i,j}\frac{w_e}{|e|}, \qquad i\ne j
\label{eq:induced-adjacency}
\end{equation}
(Lemma~\ref{lem:clique}). A single hyperedge weight $w_e$ thus sets the coupling of all $\binom{|e|}{2}$ pairs of its group. Rather than estimating each pair from its own joint observations, a group pools the evidence of all its pairs into one parameter. This pooling is the statistical role that higher-order structure plays in this paper.

\subsection{Laplacian-regularized spatiotemporal reconstruction}
\label{sec:setup-backbone}

\paragraph{Reconstruction estimator.} Graph Laplacians enter imputation as regularizers of a least-squares reconstruction, the backbone of graph-based semi-supervised learning~\citep{zhou2004learning, belkin2006manifold, ando2006learning} and of classical spatiotemporal imputers~\citep{li2020tensor, NIE2023104190}. Temporal smoothness is expressed in the same way, through $L_T$, one quarter of the Laplacian of the path graph that links each time step to the next. For readings $Y$ on the cells with observation weight $W_{i,t}=1$ (and $W_{i,t}=0$ elsewhere), a spatial operator $L\succeq0$, a temporal weight $\lambda_T>0$ and a ridge $\mu>0$, the reconstruction is
\begin{equation}
\widehat X := \argmin_{X\in\R^{N\times T}}\;
\frac12\|W\odot(X-Y)\|_F^2
+ \frac12\langle X, LX\rangle_F
+ \frac{\lambda_T}{2}\langle X, XL_T\rangle_F
+ \frac{\mu}{2}\|X\|_F^2 .
\label{eq:tikh}
\end{equation}
The four terms ask the estimate to fit the observed cells, to agree across sensors coupled by $L$, to vary smoothly in time and to stay near zero where nothing else constrains it.

\paragraph{Normal equations.} The objective is a strictly convex quadratic, so its unique minimizer is where its gradient vanishes:
\begin{equation}
\mathcal A\,\widehat X = W\odot Y,
\qquad
\mathcal A Z := W\odot Z + LZ + \lambda_T\, ZL_T + \mu Z \quad\text{for } Z\in\R^{N\times T}.
\label{eq:backbone-normal}
\end{equation}
The spatial term $LZ$ multiplies from the left and so mixes sensors within each snapshot, whereas the temporal term $ZL_T$ multiplies from the right and so mixes times within each sensor's series. The operator $\mathcal A$ is symmetric positive definite for every mask, with all eigenvalues at least $\mu$ (Lemma~\ref{lem:wellposed}). The system is solved by preconditioned conjugate gradients, an iterative method that only needs to apply $\mathcal A$ to a matrix, at a cost of $O(\mathrm{nnz}(L)\,T+NT)$ per iteration, where $\mathrm{nnz}$ counts nonzero entries.

\paragraph{Estimate at a missing cell.} Let $L=D_B-B$ be a graph Laplacian, write $d_i:=\sum_jB_{ij}$, and consider a missing cell with $1<t<T$. There $W_{i,t}=0$, and the $(i,t)$ entry of \eqref{eq:backbone-normal} rearranges to
\begin{equation}
\widehat X_{i,t} = \frac{\sum_{j} B_{ij}\, \widehat X_{j,t} + \tfrac{\lambda_T}{4}\,(\widehat X_{i,t-1} + \widehat X_{i,t+1})}{d_i + \lambda_T/2 + \mu}\,.
\label{eq:missing-cell}
\end{equation}
The estimate is a weighted average of the coupled sensors at the same time (first sum) and of the same sensor at the adjacent times (second term), shrunk toward zero by $\mu$. Under scattered dropout the adjacent times are observed and dominate the average, so the estimate hardly depends on $B$. Inside a long missing interval the adjacent times are themselves unobserved; only the readings at the two ends of the gap carry information, and their influence decays into its interior. For a sensor that is never observed, the temporal term carries no data at all. In these regimes, sensor $j$ enters the estimate of sensor $i$ with a weight proportional to $B_{ij}$. The quality of the reconstruction under long gaps is therefore decided by the spatial operator $L$. Which operator to use, and how to learn it when the joint observations are scarce, is the question MoHL answers.

\section{MoHL: Learning Mixtures of Hypergraph Laplacians}
\label{sec:method}

MoHL answers two questions in turn: \emph{which sensors belong together}, and \emph{how much each interaction order should matter}. It first discovers multiscale groups from the remaining observations and inserts their Laplacians into a spatiotemporal estimator. It then learns the mixture weights from observed cells that it masks, so that every quantity it uses comes from the observed data.

\subsection{Relational interaction groups from incomplete data}
\label{sec:method-profile}
\label{sec:method-affinity}
\label{sec:method-groups}

Everything in this subsection uses observed readings only. The standardization uses all observed cells; the other steps use an observation mask $M'\in\{0,1\}^{N\times T}$, which is the fitting mask during learning and the full mask $M$ at deployment.

\paragraph{Step 1: standardization.} Sensors measure different quantities at different levels, so each sensor is standardized by the mean $\bar Y_i$ and standard deviation $s_i$ of all its observed readings. This gives $\tilde Y_{i,t}=(Y_{i,t}-\bar Y_i)/s_i$ on observed cells and $\tilde Y_{i,t}=0$ elsewhere (Appendix~\ref{app:groups-full}).

\paragraph{Step 2: daily profile and residual.} Two sensors can share a daily cycle, with morning and evening peaks, without responding alike to the same incident. Relations should therefore be learned from departures from each sensor's typical day. The daily profile $P_{i,t}$ averages the standardized readings of sensor $i$ in 30-minute bins of the day, using its observed cells only. The \emph{residual}
\begin{equation}
R_{i,t} := M'_{i,t}\,(\tilde Y_{i,t}-P_{i,t})
\label{eq:residual}
\end{equation}
is the observed departure from the typical day, for example a congestion episode or an unusually light morning; it is zero on unobserved cells. MoHL discovers groups from $R$ and imputes $R$; the profile is added back at the end.

\paragraph{Step 3: pairwise affinity.} Requiring all members of a group to be observed together would waste most of the remaining data, so affinities are estimated one pair at a time. Each residual series is first smoothed with a one-hour moving average, giving $\bar R$. The affinity $C_{ij}\in[-1,1]$ is the Pearson correlation of $\bar R_{i,:}$ and $\bar R_{j,:}$ over the set $\mathcal T_{ij}$ of times at which both are defined, with $C_{ij}:=0$ when $|\mathcal T_{ij}|<n_{\min}$. It is high when the two sensors' departures from their typical day rise and fall together, and pairs with too few joint observations get no affinity rather than an unreliable one.

Pairs are used rather than whole groups because a group is observed in full far less often than any of its pairs. If cells are observed independently with rate $\pi$, a pair is observed together $T\pi^2$ times in expectation and a group of $s$ sensors only $T\pi^s$ times, a ratio of $\pi^{-(s-2)}$ (Lemma~\ref{lem:pairwise-rate}). At $\pi=0.5$ over seven days ($T=2016$), a pair has $504$ joint observations and a five-sensor group only $63$. The count $n_e$ of times at which every member of a group $e$ is observed never exceeds that of any of its pairs, for any mask, because a time at which all members are observed is one at which every pair is. Pairwise affinities are therefore the better-supported basis for discovering higher-order groups.

\paragraph{Step 4: nested groups.} Around each sensor, we add neighbors in order of decreasing affinity, which produces overlapping groups of every order.

\begin{definition}[Relational interaction groups]
\label{def:rig}
Fix an anchor sensor $i$ and rank the other sensors by affinity, $C_{i j_1}\ge C_{i j_2}\ge\cdots$. For each order $s=2,\dots,S$, the anchor's order-$s$ group is the anchor together with its $s-1$ strongest neighbors,
\begin{equation}
\mathcal G_i^{(s)} := \{i,\, j_1,\, \dots,\, j_{s-1}\}, \qquad \mathcal G_i^{(2)}\subset\mathcal G_i^{(3)}\subset\cdots\subset\mathcal G_i^{(S)} .
\label{eq:rig}
\end{equation}
The group is retained only if its weakest neighbor is still strongly related, $C_{i j_{s-1}}\ge c_{\min}$. The set $\Hcal_s$ collects the distinct retained order-$s$ groups over all anchors.
\end{definition}

Nesting grows each local relation from its strongest members, and the threshold $c_{\min}$ excludes weak groups that would impose incorrect smoothing. A sensor with no sufficiently related neighbor anchors no group, but it can still belong to the groups of other anchors. The groups $\Hcal_s$ and the prior adjacency $A$ determine the spatial operator of MoHL, introduced next.

\begin{figure}[!ht]
\centering
\includegraphics[width=0.9\linewidth]{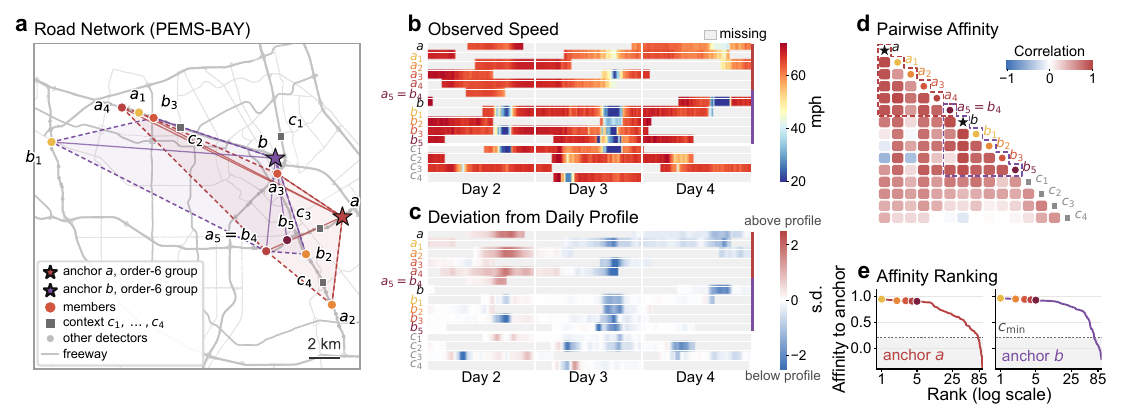}
\caption{\textbf{Finding the groups of two detectors} on PEMS-BAY under day-long gaps, with anchors $a$ and $b$, their neighbors $a_k$ and $b_k$, and context detectors $c_k$: (a) both order-6 groups on the road network; (b) observed speed; (c) deviation from the daily profile; (d) pairwise affinity; (e) affinity ranking of each anchor.}
\label{fig:discovery}
\end{figure}

\paragraph{Example on PEMS-BAY.} Figure~\ref{fig:discovery} follows the four steps for two detectors, $a$ and $b$, under day-long gaps. Anchor $a$ and its strongest neighbor are observed together at only $384$ of $2016$ steps, and the whole order-6 group of $a$ at only $53$, so groups are built from pairs. Where both are observed, their deviations rise and fall together (Figure~\ref{fig:discovery}c), giving an affinity of $0.95$. Affinity averages $0.82$ within each group and $0.48$ between them (Figure~\ref{fig:discovery}d), and every neighbor clears $c_{\min}$ (Figure~\ref{fig:discovery}e). The two groups overlap in the detector $a_5=b_4$, and both reach beyond one freeway: $a$ lies on I-280, but its strongest neighbor lies $14$~km away on US~101 (Figure~\ref{fig:discovery}a).

\subsection{The MoHL estimator}
\label{sec:method-tying}
\label{sec:method-structure}

\paragraph{Mixture operator.} One weight per discovered group would still be costly to learn. MoHL therefore shares one weight across all groups of the same order and keeps the prior graph as one more component. Orders with no retained group are dropped. To make the components comparable, each is divided by its largest eigenvalue, $\rho_s>0$ for $\sum_{e\in\Hcal_s}P_e$ and $\rho_0>0$ for $D_A-A$, which gives the unit-scaled components
\begin{equation}
L_{\rm prior} := \frac{1}{\rho_0}\,(D_A-A), \qquad L_{\Hcal_s} := \frac{1}{\rho_s}\sum_{e\in\Hcal_s}P_e, \quad s=2,\dots,S .
\label{eq:components}
\end{equation}
The spatial operator of MoHL is their conic mixture,
\begin{equation}
L_S(\alpha) := \alpha_0\, L_{\rm prior} + \sum_{s=2}^{S} \alpha_s\, L_{\Hcal_s}, \qquad \alpha=(\alpha_0,\alpha_2,\dots,\alpha_S) \ge 0 .
\label{eq:mohl-family}
\end{equation}
The weight $\alpha_0$ expresses how much to trust the prior graph, and $\alpha_s$ how much to trust the order-$s$ groups. Because every component has unit spectral radius, the weights share a common scale and can be compared across orders. Setting $\alpha_s=0$ for all $s\ge2$ recovers the prior graph. The learned weights let the mixture downweight any order whose groups do not help, while the groups themselves are discovered from the data.

\paragraph{MoHL estimator.} MoHL applies the reconstruction \eqref{eq:tikh} to the residuals $R$ in place of $Y$, with observation weights $W=M'$ and the mixture operator $L=L_S(\alpha)$. A global spatial weight would be redundant because it is absorbed into $\alpha$. The learned parameters of this static estimator are $\theta:=(\alpha,\lambda_T,\mu)$, and by \eqref{eq:backbone-normal} the estimate $\widehat X(\theta)$ solves
\begin{equation}
\mathcal A(\theta)\,\widehat X = W \odot R,
\label{eq:mohl-estimator}
\end{equation}
with the operator
\begin{equation}
\mathcal A(\theta)Z := W\odot Z + L_S(\alpha)\, Z + \lambda_T\, Z L_T + \mu Z .
\label{eq:operator}
\end{equation}
Since every component has unit spectral radius, all eigenvalues of $\mathcal A(\theta)$ lie in $[\mu,\;1+\sum_k\alpha_k+\lambda_T+\mu]$ (Lemma~\ref{lem:wellposed}). In the residual units, the ridge shrinks an unconstrained cell toward zero, that is, toward the typical day.

\paragraph{Imputation.} Adding back the typical day and undoing the standardization returns the estimate in the original units,
\begin{equation}
\widehat X^{\rm out}_{i,t} := \bar Y_i + s_i\,(\widehat X_{i,t}+P_{i,t}), \qquad (i,t)\in\Omega^c .
\label{eq:output}
\end{equation}
All spatial structure enters through $L_S(\alpha)$, so different group structures can be compared within one estimator.

\paragraph{Induced pairwise coupling.} When every order is retained, the static estimator has $S$ mixture weights, $\alpha_0$ and $\alpha_2,\dots,\alpha_S$, and with $\lambda_T$ and $\mu$ it learns $S+2$ scalars, eight at $S=6$. Each $\alpha_s$ acts on every pair of every retained order-$s$ group, and overlapping groups add their contributions. The following result makes the induced pairwise coupling explicit.

\begin{theorem}[Structure and identifiability of the shared-weight operator]
\label{thm:comembership}
Let $K^{(s)}_{ij} := |\{e \in \Hcal_s : \{i,j\} \subseteq e\}|$ count the order-$s$ groups containing the pair $\{i,j\}$, and let $\varphi(i,j) := (A_{ij}, K^{(2)}_{ij}, \dots, K^{(S)}_{ij})$ be the pair's \emph{co-membership vector}. For every $\alpha \ge 0$: \textnormal{(i)} $L_S(\alpha)=D_{A(\alpha)}-A(\alpha)$ is the graph Laplacian of the weighted adjacency
\begin{equation}
A(\alpha)_{ij} = \frac{\alpha_0}{\rho_0}\, A_{ij} + \sum_{s=2}^{S} \frac{\alpha_s}{s\,\rho_s}\, K^{(s)}_{ij}, \qquad i \ne j;
\label{eq:comembership}
\end{equation}
\textnormal{(ii)} two sensor pairs receive the same coupling for every $\alpha$ if and only if their co-membership vectors are identical; \textnormal{(iii)} if some prior edge belongs to no group and each nonempty order contains a pair that is off the prior graph and in no group of lower order, then $\alpha\mapsto L_S(\alpha)$ is injective.
\end{theorem}

The first term is the prior graph and the sum collects the groups. The coupling of two sensors thus grows with the number of groups they share at each order, $K^{(s)}_{ij}$, weighted by how much that order is trusted, $\alpha_s$. The discovered groups determine which pairs share parameters, and the mixture weights determine how strongly each order contributes. Under the conditions of (iii), different weights always give different operators, so the order weights can be recovered from the operator.

\paragraph{Weights inside long gaps.} By Theorem~\ref{thm:comembership}, $L_S(\alpha)$ is the graph Laplacian of $A(\alpha)$, so the per-cell reading \eqref{eq:missing-cell} applies with $B=A(\alpha)$. Inside long gaps, where the spatial term dominates, sensor $j$ enters the estimate of sensor $i$ with a weight proportional to $A(\alpha)_{ij}$. The learned coupling \eqref{eq:comembership}, and hence the discovered groups and their order weights, therefore determine the imputation inside long gaps.

\paragraph{Coupling profile.} Along one anchor's chain of nested groups, the member at neighbor rank $k-1$ is pulled toward the closer members by every order that contains it, that is, by orders $k$ to $S$. Its pull toward them is therefore the sum of their scaled weights,
\begin{equation}
\Lambda_k := \sum_{s=k}^{S}\frac{\alpha_s}{\rho_s}, \qquad k=2,\dots,S,
\label{eq:coupling-profile}
\end{equation}
which we call the \emph{coupling profile}. It is nonincreasing in neighbor rank, each order weight is the drop from one rank to the next, $\alpha_s/\rho_s=\Lambda_s-\Lambda_{s+1}$ with $\Lambda_{S+1}:=0$, and the ratio $\Lambda_k/\Lambda_2$ says how strongly the $(k{-}1)$-th neighbor is still pulled toward the closer members relative to the strongest one. This is exact on a single chain (Appendix~\ref{app:proof-spectral}).

\paragraph{Operators on toy and real data.} Figure~\ref{fig:operators} shows these operators. On a toy network (Figure~\ref{fig:operators}a--d), each order's component couples every pair of a group through one shared weight, and nested groups give nested blocks. On PEMS03 under day-long gaps, with the learned static weights, the discovered groups link sensors that the prior graph leaves apart and carry most of the learned coupling (Figure~\ref{fig:operators}e--h).

\begin{figure}[!ht]
\centering
\includegraphics[width=0.9\linewidth]{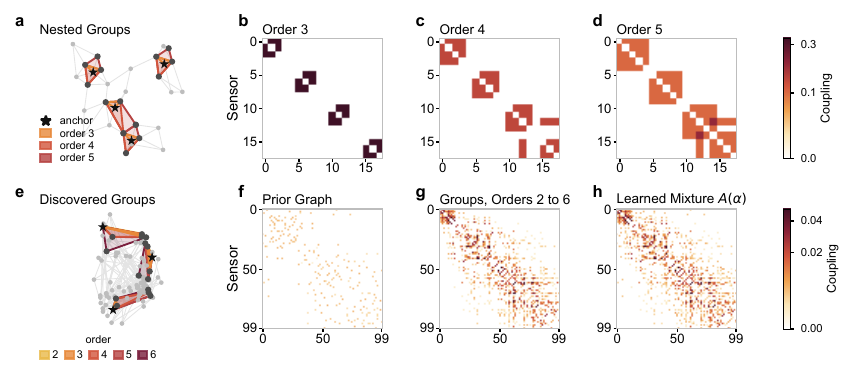}
\caption{\textbf{The operators of MoHL}: (a--d) nested groups on a toy network and the adjacency of each order; (e--h) on PEMS03 under day-long gaps, the discovered groups, the weighted prior graph, the weighted groups and the learned mixture $A(\alpha)$.}
\label{fig:operators}
\end{figure}

\paragraph{Latent-factor interpretation.} A hyperedge also has a probabilistic meaning: its members fluctuate around a shared, unobserved group level.

\begin{theorem}[Latent-factor representation]
\label{thm:factor}
\textnormal{(a)} For a hyperedge $e$ and $\alpha > 0$, $\exp(-\tfrac{\alpha}{2}x^\top P_e x) \propto \int_{\R}\exp\bigl(-\tfrac{\alpha}{2}\sum_{i \in e}(x_i - \nu)^2\bigr)\,d\nu$. A hyperedge penalty is thus the marginal likelihood of $x_i = \nu + \varepsilon_i$ with independent $\varepsilon_i \sim \mathcal N(0, 1/\alpha)$ and a flat prior on the group level $\nu$. \textnormal{(b)} Let the residual signal have the Gaussian prior with density proportional to $\exp\bigl(-\tfrac12\langle X,L_S(\alpha)X\rangle_F-\tfrac{\lambda_T}{2}\langle X,XL_T\rangle_F-\tfrac{\mu}{2}\|X\|_F^2\bigr)$, and let the observed residuals carry independent Gaussian noise with precision $W_{i,t}$. Then the solution of \eqref{eq:mohl-estimator} at the true parameters is the posterior mean and the cellwise posterior median, and hence Bayes-optimal for both squared and absolute error.
\end{theorem}

Part (a) says that penalizing deviations from the group average is equivalent to assuming that all members share one latent level $\nu$, such as a common congestion state. Each member adds independent noise of variance $1/\alpha$. A single weight $\alpha$ then describes the whole group, which is why sensors in a valid group can borrow strength from each other. Part (b) says that if the data follow this model, the MoHL estimate at the true weights is the best possible imputation. Proofs are in Appendix~\ref{app:mohl-proofs}.

\subsection{Evidence-adaptive gate}
\label{sec:gating}

\paragraph{Context features.} A cell deep inside a day-long gap needs more spatial coupling than a cell next to an observation. MoHL therefore lets the strength of each component vary from cell to cell with an observed context vector $c_{i,t}\in\R^7$, computed from the mask and the observed cells only. Its features are the temporal gap $g_{i,t}$, capped at one day and log-scaled, the share of sensor $i$'s prior-graph neighbors that are unobserved at time $t$, four harmonics of the time of day and the network state at time $t$ (Appendix~\ref{app:gate-adjoint}).

\paragraph{Gated couplings.} Each component $k\in\{0,2,\dots,S\}$ has a unit-scaled adjacency $A_k$, namely $A_0:=A/\rho_0$ for the prior graph and $A_s:=K^{(s)}/(s\rho_s)$ for order $s$, and its own slope vector $b_k\in\R^7$. The gate of sensor $i$ for component $k$ at time $t$, and the resulting total coupling between sensors $i$ and $j$, are
\begin{equation}
h^{k}_{i,t} := \exp\bigl(\mathrm{clip}(b_k^\top c_{i,t},\,-\kappa,\,\kappa)\bigr), \qquad
w_{ij}(t) := \sum_{k} \alpha_k\, (A_k)_{ij}\,\tfrac12\bigl(h^{k}_{i,t} + h^{k}_{j,t}\bigr),
\label{eq:egate}
\end{equation}
where $\mathrm{clip}(u,-\kappa,\kappa):=\max(-\kappa,\min(\kappa,u))$ bounds the gate to $[e^{-\kappa},e^{\kappa}]$. A gate value $h^k_{i,t}>1$ asks for more coupling of component $k$ at cell $(i,t)$, and $h^k_{i,t}<1$ for less. The coupling of a pair uses the average gate of its two endpoints, so it stays symmetric and nonnegative, and an edge is strong if either endpoint asks for it. The temporal link between $(i,t)$ and $(i,t+1)$ has one gate $h^{T}_{i,t}$, computed in the same way from the average context of its two times with its own slopes $b_T\in\R^7$.

\paragraph{Gated estimator.} The gated estimator replaces the spatial and temporal terms of \eqref{eq:tikh} by their time-varying versions, collected in the gated regularizer
\begin{equation}
\mathcal R_{\rm gate}(X) := \frac12\sum_{t=1}^{T}\sum_{i<j} w_{ij}(t)\,(X_{i,t}-X_{j,t})^2 + \frac{\lambda_T}{8}\sum_{i=1}^{N}\sum_{t=1}^{T-1} h^{T}_{i,t}\,(X_{i,t+1}-X_{i,t})^2 .
\label{eq:gated-reg}
\end{equation}
With all gates equal to one, $w_{ij}(t)=A(\alpha)_{ij}$ and $h^{T}_{i,t}=1$, and $\mathcal R_{\rm gate}$ reduces to the static terms by \eqref{eq:graph-dirichlet} and \eqref{eq:temporal-dirichlet}. At every time $t$ the couplings $w(t)$ are symmetric and nonnegative, so they define a graph. The gated objective is therefore still a convex quadratic, and its normal equations keep the form \eqref{eq:mohl-estimator} with a symmetric positive-definite operator. Zero slopes, $b_k=0$ for every $k$ and $b_T=0$, recover the static estimator exactly. We call the gated model \textbf{MoHL} and the static core \textbf{MoHL-Static}. With $a_k:=\log\alpha_k$, MoHL learns
\begin{equation}
\theta = \bigl(a_0,a_2,\dots,a_S;\ b_0,b_2,\dots,b_S;\ b_T;\ \log\lambda_T,\ \log\mu\bigr),
\label{eq:theta}
\end{equation}
that is, $S$ log-weights, $7S$ spatial slopes, $7$ temporal slopes and two scalars: $8S+9=57$ parameters at $S=6$.

\subsection{Self-supervised learning of the mixture}
\label{sec:method-learn}
\label{sec:method-validation}

\paragraph{Mask-matched masking.} Missing cells have no labels, so MoHL learns from observed cells that it hides from itself and then tries to recover. Masking random cells would mostly teach short-gap interpolation, so MoHL instead hides observed cells as events shaped like the mask's own gaps. It sets aside two such sets of observed cells: the \emph{masked cells} $\Omega_{\rm learn}$, on which the mixture is learned, and the \emph{test cells within the observed data} $\Omega_{\rm sel}$, which are withheld from both fitting and learning and serve only to score the learned model. The remaining observed cells form the fitting set,
\begin{equation}
\Omega_{\rm fit} := \Omega\setminus(\Omega_{\rm learn}\cup\Omega_{\rm sel}), \qquad M^{\rm fit}_{i,t} := \mathds{1}\bigl[(i,t)\in\Omega_{\rm fit}\bigr].
\label{eq:split}
\end{equation}
The profile, affinities, groups, components and gate context are all rebuilt from $M^{\rm fit}$, so the learned structure never sees the value of a masked cell; only the per-sensor mean and scale of the standardization use all observed cells.

\paragraph{Validation objective.} For given parameters $\theta$, let $\widehat X(\theta)$ be the estimate computed from the fitting cells alone: the solution of $\mathcal A(\theta)\widehat X=M^{\rm fit}\odot R^{\rm fit}$, where $P^{\rm fit}$, $R^{\rm fit}:=M^{\rm fit}\odot(\tilde Y-P^{\rm fit})$ and the operator $\mathcal A(\theta)$ are all built from $M^{\rm fit}$. Its error on a masked cell is measured against the true residual of that cell, which the solve never saw:
\begin{equation}
u_{i,t}(\theta) := \widehat X_{i,t}(\theta) - (\tilde Y_{i,t} - P^{\rm fit}_{i,t}), \qquad (i,t)\in\Omega_{\rm learn}.
\label{eq:validation-error}
\end{equation}
MoHL minimizes the validation loss plus a penalty that prunes unhelpful orders and a ridge that keeps the gates near one,
\begin{equation}
\mathcal J(\theta) := \mathcal L(\theta) + \gamma\sum_{s=2}^{S}\alpha_s + \frac{\gamma_b}{2}\Bigl(\sum_{k}\|b_k\|_2^2+\|b_T\|_2^2\Bigr),
\label{eq:mohl-bilevel}
\end{equation}
where $\mathcal L(\theta)$ averages the Huber loss $\ell_\delta$ of the errors \eqref{eq:validation-error} over $\Omega_{\rm learn}$. The Huber loss is quadratic for small errors and linear for large ones, so a few unusual masked cells cannot dominate the fit. The penalty $\gamma\sum_s\alpha_s$ pushes the weight of an order toward zero unless that order lowers the validation error. The ridge with weight $\gamma_b$ keeps the gate slopes small, so the static mixture is the default. Writing $\alpha_k=e^{a_k}$ and learning $a_k$, $\log\lambda_T$ and $\log\mu$ keeps all of these quantities positive without positivity constraints.

\paragraph{Adjoint gradient.} The parameters affect the validation loss only through the solution $\widehat X(\theta)$ of a linear system, and the gradient follows from differentiating that system. Write $\mathcal A=\mathcal A(\theta)$ and let $\theta_r$ be any one parameter. Differentiating $\mathcal A\widehat X=M^{\rm fit}\odot R^{\rm fit}$, whose right-hand side does not depend on $\theta$, and applying the chain rule gives
\begin{equation}
\frac{\partial\widehat X}{\partial\theta_r} = -\mathcal A^{-1}\Bigl(\frac{\partial\mathcal A}{\partial\theta_r}\,\widehat X\Bigr), \qquad
\frac{\partial\mathcal L}{\partial\theta_r} = \Bigl\langle G,\ \frac{\partial\widehat X}{\partial\theta_r}\Bigr\rangle_F = -\Bigl\langle \Psi,\ \frac{\partial\mathcal A}{\partial\theta_r}\,\widehat X\Bigr\rangle_F ,
\label{eq:adjoint-derivation}
\end{equation}
where $G:=\partial\mathcal L/\partial\widehat X$ is the gradient of the validation loss with respect to the estimate, and the \emph{adjoint state} $\Psi$ solves one more system with the same operator,
\begin{equation}
\mathcal A\,\Psi = G, \qquad G_{i,t} := \begin{cases} \ell'_\delta\bigl(u_{i,t}\bigr)/|\Omega_{\rm learn}|, & (i,t)\in\Omega_{\rm learn},\\ 0, & \text{otherwise}, \end{cases}
\label{eq:adjoint}
\end{equation}
with $\ell'_\delta(u)=\max(-\delta,\min(\delta,u))$, the clipped error. The last equality in \eqref{eq:adjoint-derivation} moves $\mathcal A^{-1}$ onto $G$, which is allowed because $\mathcal A$ is symmetric. Intuitively, $\Psi$ is the validation error propagated back through the solve: it measures how the validation loss changes when the data entering the solve change. Because $\Psi$ does not depend on which parameter is differentiated, one adjoint solve serves all of them (Proposition~\ref{prop:adjoint}). For the static parameters, $\partial\mathcal A/\partial a_k$ multiplies by $\alpha_k\bar L_k$, with $\bar L_0:=L_{\rm prior}$ and $\bar L_s:=L_{\Hcal_s}$, so
\begin{equation}
\frac{\partial\mathcal L}{\partial a_k} = -\alpha_k\,\langle \Psi,\, \bar L_k\widehat X\rangle_F, \qquad
\frac{\partial\mathcal L}{\partial\log\lambda_T} = -\lambda_T\,\langle \Psi,\, \widehat X L_T\rangle_F, \qquad
\frac{\partial\mathcal L}{\partial\log\mu} = -\mu\,\langle \Psi,\, \widehat X\rangle_F ,
\label{eq:mohl-adjoint}
\end{equation}
for $k=0,2,\dots,S$. The penalties add $\gamma\alpha_s$ to $\partial\mathcal J/\partial a_s$ for $s\ge2$, and the gate slopes are differentiated with the same $\Psi$. One learning step therefore costs one forward solve and one adjoint solve, however many orders and gate parameters there are. We minimize $\mathcal J$ with Adam in stages that add one part of the model at a time, keeping each parameter in a box $\Theta$; Appendix~\ref{app:optimization} gives the box and the complete procedure (Algorithm~\ref{alg:mohl}), and Appendix~\ref{app:learning-curves} traces the staged fit.

\subsection{Support-matched test of group sharing}
\label{sec:method-untied}
\label{sec:method-theory}

\paragraph{Shared versus free weights.} To isolate the effect of sharing, we compare MoHL-Static with a freely weighted graph on exactly the same sensor pairs. The group adjacency with all order weights set to one, and its support (the set of coupled pairs), are
\begin{equation}
A_{\rm shared} := \sum_{s=2}^S \frac{K^{(s)}}{s\,\rho_s}, \qquad \mathcal E := \bigl\{\{i,j\} : (A_{\rm shared})_{ij}>0\bigr\}, \qquad m:=|\mathcal E| .
\label{eq:shared-support}
\end{equation}
MoHL-Free rescales every supported pair by its own positive factor, parameterized by a free logit $\vartheta_{ij}\in\R$ through $\mathrm{softplus}(u):=\log(1+e^u)$:
\begin{equation}
\begin{aligned}
(A_{\rm free})_{ij} &:= \frac{\mathrm{softplus}(\vartheta_{ij})}{\mathrm{softplus}(0)}\,(A_{\rm shared})_{ij}, \qquad \{i,j\}\in \mathcal E,\\
L_S &:= \alpha_0\,L_{\rm prior} + \alpha_1\,(D_{A_{\rm free}}-A_{\rm free}).
\end{aligned}
\label{eq:untied}
\end{equation}
At $\vartheta=0$ every factor is one, so MoHL-Free starts as MoHL-Static with equal order weights $\alpha_2=\dots=\alpha_S=\alpha_1$. Every relational adjacency of MoHL-Static is a limit of free adjacencies (Proposition~\ref{prop:cone}), so the free model can represent anything the shared one can. The two models share the support, masked cells, solver, objective, optimizer and step budget; only the parameterization differs. MoHL-Static spends $S-1$ order weights on the $m$ couplings, whereas MoHL-Free learns each coupling separately, with $p=m+4$ parameters $(a_0,a_1,\vartheta,\log\lambda_T,\log\mu)$.

\paragraph{Bias-variance trade-off.} The free model is more flexible but must estimate many more couplings from the same incomplete observations. Sharing reduces this estimation burden, at the cost of bias when the couplings within an order differ. The comparison has a simple analog in linear regression with Gaussian inputs. Consider $m$ coefficients, one per coupling, estimated from $n$ samples; let $\beta^*\in\R^m$ be the true coefficients and $\sigma^2$ the noise variance. Sharing restricts the coefficients to an $r$-dimensional subspace $U\subset\R^m$, for example the vectors that are constant within each group of couplings, with $r<m$.

\begin{proposition}[Shared versus free weights]
\label{prop:tying}
Let $x_l \sim \mathcal N(0, I_m)$ and $y_l = x_l^\top\beta^* + \varepsilon_l$ with $\varepsilon_l \sim \mathcal N(0,\sigma^2)$, independently for $l = 1,\dots,n$. Let $\hat\beta^{\rm free}$ be least squares over $\R^m$ and $\hat\beta^{\rm shared}$ least squares over $U$, write $\mathcal R_{\rm ex}(\hat\beta):=\|\hat\beta-\beta^*\|_2^2$ for the excess prediction risk at a new design point, and let $\zeta^2 := \|\beta^* - \Pi_U\beta^*\|_2^2$ be the \emph{group-structure mismatch}. For $n > m+1$,
\begin{equation}
\E\mathcal R_{\rm ex}(\hat\beta^{\rm free}) = \frac{m\,\sigma^2}{n-m-1}, \qquad \E\mathcal R_{\rm ex}(\hat\beta^{\rm shared}) = \zeta^2 + \frac{r\,(\sigma^2 + \zeta^2)}{n-r-1},
\label{eq:tying-risks}
\end{equation}
and sharing has the lower risk if and only if $\zeta^2 < (m-r)\,\sigma^2/(n-m-1)$.
\end{proposition}

Here $\Pi_U\beta^*$ is the closest shared coefficient vector to the truth, so $\zeta^2$ measures how far the true couplings are from being equal within groups. The free estimator has no bias but pays variance for all $m$ coefficients. The shared estimator trades a bias $\zeta^2$, which is small when the couplings within a group are nearly equal, for a variance on only $r$ coefficients. Sharing therefore wins until the variance it removes, which shrinks with $n$, falls below this bias, a sample-size threshold:

\begin{corollary}[Scarcity crossover]
\label{cor:scarcity}
For $n>m+1$, sharing has lower expected risk than the free model if and only if
\begin{equation}
n<n_{\rm cross}:=m+1+\frac{(m-r)\,\sigma^2}{\zeta^2},
\label{eq:scarcity}
\end{equation}
with $n_{\rm cross}=\infty$ when $\zeta^2=0$. Under independent Bernoulli observation with rate $\pi$, a pair is observed together $T\pi^2$ times in expectation; identifying $n$ with this count places the crossover at the window length $T_{\rm cross}=n_{\rm cross}/\pi^2$. It grows with the number of couplings $m$ and the noise $\sigma^2$, grows as the observation rate $\pi$ falls, and shrinks as the mismatch $\zeta^2$ grows.
\end{corollary}

In the comparison of MoHL-Static with MoHL-Free, $m$ is the number of couplings on the discovered support and $r=S-1$ is the number of order weights. The vector $\beta^*$ plays the role of the ideal pair weights on that support, and $n$ is set by the window length and the observation rate. The more couplings one order weight replaces, the later the crossover. The generalization guarantee of the learned operator (Theorem~\ref{thm:generalization}, Appendix~\ref{app:proof-generalization}) points the same way: its estimation term grows with the number of learned scalars, so its bound favors sharing while the number of masked cells is small (Corollary~\ref{cor:crossover}).

\section{Experiments}
\label{sec:exp}

Our analysis makes three predictions. Relational structure should matter most inside long gaps, where a sensor's own history is unavailable \eqref{eq:missing-cell}. Shared weights should beat free ones while joint observations are scarce (Proposition~\ref{prop:tying}). And groups discovered from pairwise affinities should remain related on cells that MoHL never saw.

\subsection{Setup}
\label{sec:exp-setup}

\paragraph{Data and missingness.}
We evaluate PEMS-BAY (speed; $325$ sensors), PEMS03 (flow; $358$ sensors) and PEMS04 (flow; $307$ sensors), all sampled every five minutes. Native missingness is retained, and zero readings are treated as missing on PEMS03 and PEMS04. The main comparison uses the $100$ highest-degree sensors of the largest connected component, the densest seven-day window ($T=2016$) and $50\%$ injected missingness, with identical masks across methods. The regimes differ in which evidence survives. The main results report scattered cell dropout (\emph{Cell}), five nearby sensors missing together for 30 minutes (\emph{Cluster}) and one 24-hour gap per sensor (\emph{Block}). Additional experiments include 30-minute and six-hour blocks (Block-6 and Block-72; Block is Block-288 in this notation) and a \emph{Mixed} regime that combines cell dropout, long gaps and fully missing sensors.

\paragraph{Methods and controls.}
We compare MoHL with the sensor mean, $k$NN-spatial, HaLRTC~\citep{6138863}, LETC~\citep{NIE2023104190}, WDGTC~\citep{li2020tensor} and three deep imputers: GRIN~\citep{cini2022grin}, SPIN~\citep{marisca2022spin} and ImputeFormer~\citep{nie2024imputeformer}. Deep models are trained from scratch on the same seven-day window. The ablation adds MoHL's components one at a time: \textbf{MoHL-Base} uses the tuned prior graph alone, and each subsequent model adds the daily profile, pairs, higher orders (\textbf{MoHL-Static}) and the gate (\textbf{MoHL}). \textbf{MoHL-Free} replaces the shared weights with free weights on the same support \eqref{eq:untied}. All MoHL variants use the same masks, masked cells, solver, optimizer and step budget.

\subsection{Performance under structured missingness}
\label{sec:exp-compare}
\label{sec:exp-deep}

\paragraph{Accuracy against the baselines.}

\begin{table}[!ht]
\centering\footnotesize
\setlength{\tabcolsep}{4.2pt}
\caption{\textbf{MAE on three road networks under three missingness regimes}: best in bold, second best underlined.}
\label{tab:compare}
\label{tab:deep}
\begin{tabular}{@{}l *{3}{c} c *{3}{c} c *{3}{c}@{}}
\toprule
 & \multicolumn{3}{c}{PEMS-BAY (mph)} &  & \multicolumn{3}{c}{PEMS03 (veh/5 min)} &  & \multicolumn{3}{c}{PEMS04 (veh/5 min)} \\
\cmidrule(lr){2-4}\cmidrule(lr){6-8}\cmidrule(lr){10-12}
Method & Cell & Cluster & Block &  & Cell & Cluster & Block &  & Cell & Cluster & Block \\
\midrule
$k$NN-spatial & 3.450 & 4.428 & 3.469 &  & 84.08 & 84.50 & 85.89 &  & 94.46 & 107.7 & 92.41 \\
HaLRTC & 1.116 & 1.818 & 3.203 &  & \underline{9.800} & \underline{11.55} & 25.11 &  & 22.95 & 29.18 & 81.98 \\
LETC & 1.748 & 2.147 & 3.178 &  & 36.79 & 54.49 & 89.63 &  & 52.33 & 69.26 & 118.7 \\
WDGTC & 2.142 & 3.577 & 13.73 &  & 13.42 & 14.06 & 40.03 &  & 26.81 & 31.88 & 103.0 \\
\midrule
GRIN & \underline{0.706} & \underline{1.291} & \underline{3.029} &  & 10.36 & 12.61 & 21.57 &  & \underline{20.26} & 25.97 & 50.03 \\
SPIN & 0.840 & 1.319 & 3.162 &  & 11.65 & 13.09 & 20.03 &  & 20.69 & 24.47 & 39.62 \\
ImputeFormer & 0.873 & 1.326 & 3.202 &  & 10.75 & 11.58 & \underline{14.72} &  & 20.75 & \underline{23.16} & \underline{31.66} \\
\midrule
\textbf{MoHL} & \textbf{0.699} & \textbf{1.262} & \textbf{2.751} &  & \textbf{9.626} & \textbf{11.43} & \textbf{14.06} &  & \textbf{19.53} & \textbf{23.08} & \textbf{31.31} \\
\midrule
Improvement & $+1.0\%$ & $+2.2\%$ & $+9.2\%$ &  & $+1.8\%$ & $+1.0\%$ & $+4.5\%$ &  & $+3.6\%$ & $+0.3\%$ & $+1.1\%$ \\
\bottomrule
\end{tabular}

\end{table}

\begin{figure}[!ht]
\centering
\includegraphics[width=0.96\linewidth]{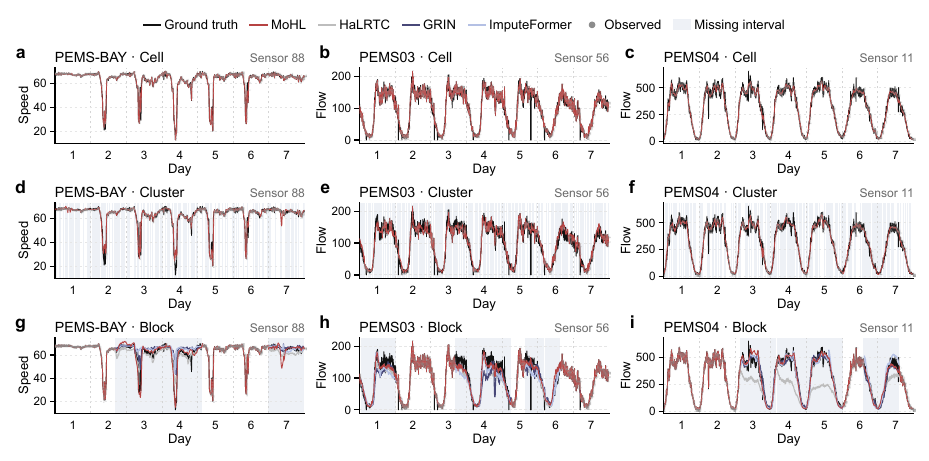}
\caption{\textbf{MoHL imputations over the seven-day window} under Cell (top), Cluster (middle) and Block (bottom), one sensor per network (speed in mph, flow in vehicles per 5~min), with HaLRTC, GRIN and ImputeFormer inside the Block gaps.}
\label{fig:traces}
\end{figure}

Table~\ref{tab:compare} compares MoHL with four classical and three deep imputers. MoHL had the lowest error on every network and in every regime of the table, and its margin over the best baseline was largest under day-long gaps on PEMS-BAY and PEMS03, $9.2\%$ and $4.5\%$. Under Cell, a missing cell keeps observed neighbors in time, so most methods interpolate well. Under Block, the estimate must come from other sensors, and the methods differ in how they find those sensors. On the flow networks, low-rank and tensor methods lost the most accuracy, the error of HaLRTC roughly tripling from Cell to Block, because day-long gaps remove the scattered entries on which low-rank completion relies. With one week of incomplete data, efficiently learned relational structure proved more valuable than the capacity of the deep imputers. Appendix~\ref{app:all-regimes} reports all six regimes.

\begin{figure}[!ht]
\centering
\includegraphics[width=0.96\linewidth]{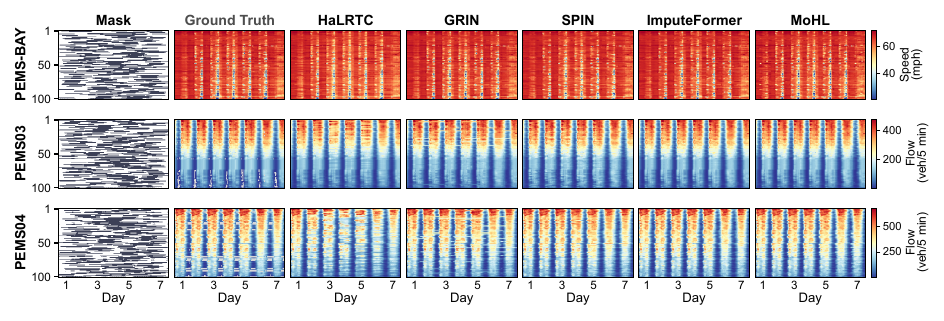}
\caption{\textbf{Reconstructions under Block missingness}: sensor-by-time heatmaps for MoHL and baselines, together with the mask and ground truth.}
\label{fig:heatmap-baselines}
\end{figure}

\begin{figure}[!ht]
\centering
\includegraphics[width=\linewidth]{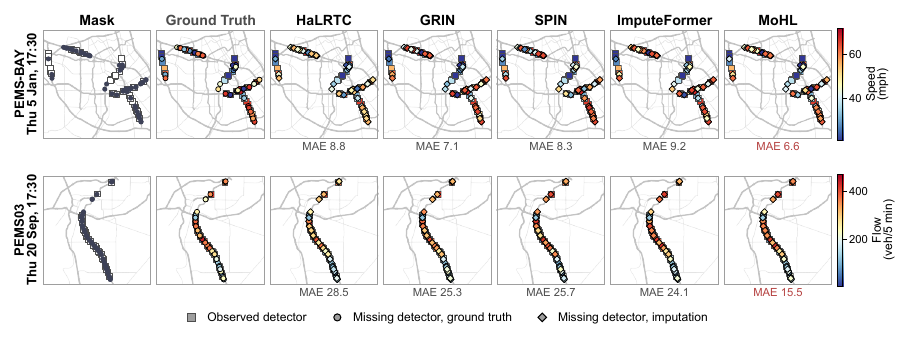}
\caption{\textbf{Reconstructions on the road network under Block}: PEMS-BAY and PEMS03 detectors at the weekday evening-peak step with the most missing detectors, with each model's MAE over the missing detectors.}
\label{fig:sensor-maps}
\end{figure}

\paragraph{Imputations inside long gaps.}
Figures~\ref{fig:traces} and~\ref{fig:heatmap-baselines} show the imputations behind these errors. Under Cell and Cluster, MoHL's imputations stayed close to the truth throughout the week. Under Block, MoHL followed the daily cycle and most congestion dips of the specific day inside the gaps (Figure~\ref{fig:traces}g--i), whereas HaLRTC drifted far below the truth on PEMS04. The heatmaps show the same contrast for every sensor: MoHL reproduced the daily bands and the differences between sensors, while the low-rank reconstruction of HaLRTC washed out the daily bands of PEMS04 over several days (Figure~\ref{fig:heatmap-baselines}). Appendix~\ref{app:traces-all} extends the traces to more sensors.

\paragraph{Detector-level reconstructions.}
Figure~\ref{fig:sensor-maps} places the reconstructions on the road network at the weekday evening-peak step with the most missing detectors, a step chosen by rule rather than by inspection. MoHL had the lowest error over the missing detectors on both networks, $6.6$~mph on PEMS-BAY against $7.1$--$9.2$~mph for the baselines and $15.5$ against $24.1$--$28.5$ vehicles per 5~min on PEMS03. The step shows the averages of Table~\ref{tab:compare} at the level of individual detectors, including stretches of freeway that are missing at once.

\paragraph{Recovery of day-specific patterns.}
The traces suggest that MoHL imputes the specific day rather than a typical one. Figure~\ref{fig:tsne-days} tests this on every day-long stretch inside a Block gap, embedding the true stretch and each method's imputation of it in one t-SNE per network. Two measures, computed in the original space rather than in the embedding, quantify the picture. The \emph{match} is the share of stretches whose imputation is closer to its own true day than to any other, and the \emph{spread} is the total variance of the imputed stretches relative to that of the true ones. On all three networks, MoHL had the highest match, $43\%$, $91\%$ and $80\%$ against at most $38\%$, $85\%$ and $79\%$ for the baselines, with a spread close to that of the truth ($0.73$, $0.96$ and $0.93$). The sensor mean, in contrast, collapsed onto one typical day (spread $0.06$ on PEMS-BAY), and WDGTC scattered its imputations far beyond the truth (spread $2.90$). MoHL thus imputes the specific day: the daily profile supplies the recurrent shape of the day, and the groups supply its departures, read from related sensors observed at the same moment.

\begin{figure}[!ht]
\centering
\includegraphics[width=0.96\linewidth]{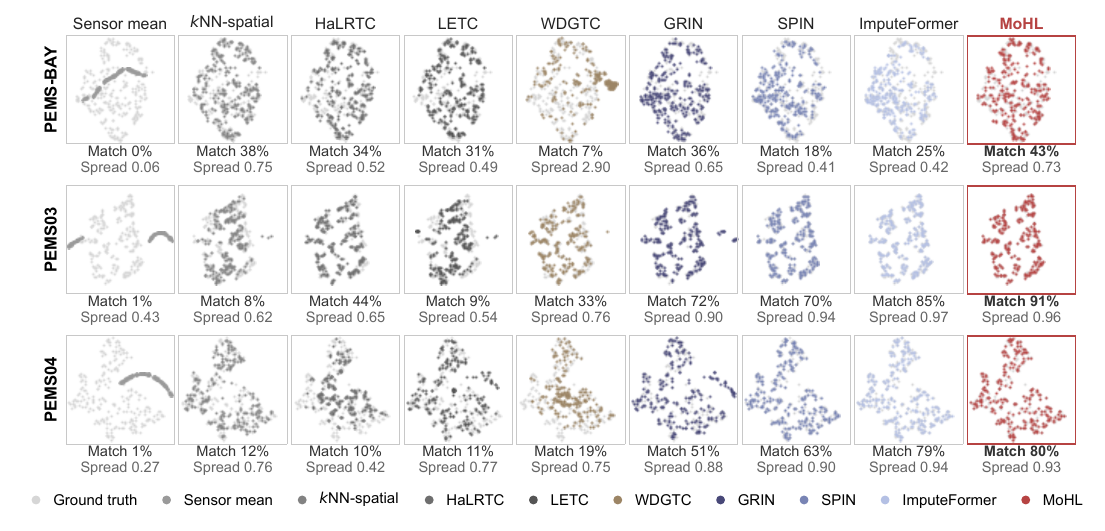}
\caption{\textbf{Imputed days against ground truth}: t-SNE of each day-long gap as ground truth (gray) and as each method's imputation.}
\label{fig:tsne-days}
\end{figure}

\subsection{Ablation studies}
\label{sec:exp-order}
\label{sec:exp-where}

\begin{table}[!h]
\centering\footnotesize
\setlength{\tabcolsep}{3.2pt}
\caption{\textbf{Ablation of MoHL}: MAE as each component is added to the prior-graph backbone, with the support-matched control MoHL-Free; best in bold, second best underlined, ranked on unrounded values.}
\label{tab:ablation}
\begin{tabular}{@{}l l *{3}{c} c *{3}{c} c *{3}{c}@{}}
\toprule
 &  & \multicolumn{3}{c}{PEMS-BAY (mph)} &  & \multicolumn{3}{c}{PEMS03 (veh/5 min)} &  & \multicolumn{3}{c}{PEMS04 (veh/5 min)} \\
\cmidrule(lr){3-5}\cmidrule(lr){7-9}\cmidrule(lr){11-13}
Model & Component added & Cell & Cluster & Block &  & Cell & Cluster & Block &  & Cell & Cluster & Block \\
\midrule
\multicolumn{13}{@{}l}{\textit{Cumulative ladder}} \\
MoHL-Base & backbone (prior graph) & 0.732 & 1.273 & 3.128 &  & 11.11 & 13.31 & 22.98 &  & 19.82 & 25.17 & 52.71 \\
+ profile & daily profile & 0.732 & \underline{1.267} & 2.882 &  & 10.99 & 12.97 & 16.07 &  & 19.63 & 23.61 & 34.57 \\
+ pairs & groups of order 2 (pairs) & 0.703 & 1.294 & 2.814 &  & 9.984 & 12.19 & 15.32 &  & 19.62 & 23.30 & 32.28 \\
MoHL-Static & groups of orders 3--6 & \underline{0.700} & 1.268 & 2.808 &  & 9.652 & 11.45 & \underline{14.12} &  & \underline{19.61} & \underline{23.14} & \underline{31.38} \\
\textbf{MoHL} & evidence-adaptive gate & \textbf{0.699} & \textbf{1.262} & \underline{2.751} &  & \underline{9.626} & \underline{11.43} & \textbf{14.06} &  & \textbf{19.53} & \textbf{23.08} & \textbf{31.31} \\
\addlinespace[2pt]
\textit{Gain (\%)} & MoHL vs.\ MoHL-Base & $4.5$ & $0.9$ & $12.1$ &  & $13.4$ & $14.1$ & $38.8$ &  & $1.5$ & $8.3$ & $40.6$ \\
\midrule
\multicolumn{13}{@{}l}{\textit{Support-matched control}} \\
MoHL-Free & free pair weights, not shared & 0.719 & 1.470 & \textbf{2.695} &  & \textbf{9.077} & \textbf{11.24} & 14.25 &  & 19.61 & 23.71 & 32.13 \\
\bottomrule
\end{tabular}

\end{table}

Table~\ref{tab:ablation} adds MoHL's components one at a time (all six regimes in Appendix~\ref{app:all-ablation}), and Table~\ref{tab:order} varies the largest interaction order.

\paragraph{Effect of the daily profile.}
The profile matters most when a sensor's own history is gone. Under Block it lowered the error of MoHL-Base by $8\%$ on PEMS-BAY and by $30\%$ and $34\%$ on PEMS03 and PEMS04, whose flows follow a strong daily cycle. Under Cell it changed the error by about $1\%$ or less, because the readings just before and after a missing cell already fix its place in the day. The profile thus supplies the recurrent part of the signal that long gaps remove.

\paragraph{Effect of the relational groups.}
Pairs and higher-order groups then lowered the error by up to $12\%$ beyond the profile (PEMS03), with the largest absolute gains under Block on every network; the gain grows with the distance to the nearest reading (Appendix~\ref{app:strata}). They add what no profile can provide: how the other sensors depart from their typical day at the same moment. Beyond pairs, higher orders added most under Block on the flow networks, $7.8\%$ on PEMS03 and $2.8\%$ on PEMS04. The one column without a gain from groups is PEMS-BAY under Cluster, where the profile alone was already as accurate as MoHL-Static.

\paragraph{Effect of the gate.}
Because zero slopes recover MoHL-Static, the gate refines the static mixture rather than replacing it. It lowered the error most on PEMS-BAY under Block ($2.0\%$) and by at most $0.5\%$ elsewhere. A single static mixture per network thus captured most of the relational structure.

\paragraph{Free weights on the same support.}
In Table~\ref{tab:ablation}, the comparison with MoHL-Free is mixed: free weights were better on PEMS03 under Cell and Cluster and on PEMS-BAY under Block, and no better in the other columns. Whether sharing pays thus depends on the evidence available to each coupling.

\begin{table}[!ht]
\centering\footnotesize
\setlength{\tabcolsep}{5pt}
\caption{\textbf{Gain by interaction order}: MAE on PEMS-BAY ($N=60$, two seeds) as orders are added, with the change $\Delta$ from pairs.}
\label{tab:order}
\begin{tabular}{@{}l l *{5}{c}@{}}
\toprule
 & & \multicolumn{5}{c}{relational interaction groups, by order} \\
\cmidrule(lr){3-7}
Regime & & Pairs ($s{=}2$) & $s{=}3$ & $s{=}4$ & $s{=}5$ & $s{=}6$ \\
\midrule
\multirow{2}{*}{Cell} & MAE & \textbf{0.740} & 0.741 & 0.741 & 0.741 & 0.741 \\
 & $\Delta$ & --- & $+0.001$ & $+0.001$ & $+0.001$ & $+0.001$ \\[1pt]
\multirow{2}{*}{Cluster} & MAE & 1.288 & \textbf{1.278} & 1.286 & 1.283 & 1.282 \\
 & $\Delta$ & --- & $-0.010$ & $-0.002$ & $-0.005$ & $-0.006$ \\[1pt]
\multirow{2}{*}{Block} & MAE & 2.889 & 2.742 & 2.685 & \textbf{2.683} & 2.684 \\
 & $\Delta$ & --- & $-0.147$ & $-0.204$ & $-0.206$ & $-0.205$ \\[1pt]
\bottomrule
\end{tabular}

\end{table}

\paragraph{Effect of the interaction order.}
Table~\ref{tab:order} varies the largest order on a 60-sensor PEMS-BAY subnetwork. The useful depth follows the gap. Under Cell, a missing cell has observed neighbors in time, and pairs sufficed. Under Cluster, nearby sensors fail together, so a sensor's closest relatives may be missing with it, and order three, which reaches one more neighbor, gave the main gain. Under Block, a sensor loses its own history for a day, and orders three and four carried most of a $7\%$ gain. Under long gaps, each sensor thus draws on more distant neighbors in its affinity ranking, and the same orders help on the test cells within the observed data (Appendix~\ref{app:more-tables}).

\subsection{Evidence-efficient group sharing}
\label{sec:exp-tied}

Proposition~\ref{prop:tying} predicts that shared weights beat free ones while joint observations are scarce, provided the couplings within a group are nearly equal. We test the comparison on real data, its mechanism against a Bayes oracle, and the homogeneity condition on the discovered groups.

\begin{figure}[!ht]
\centering
\includegraphics[width=0.85\linewidth]{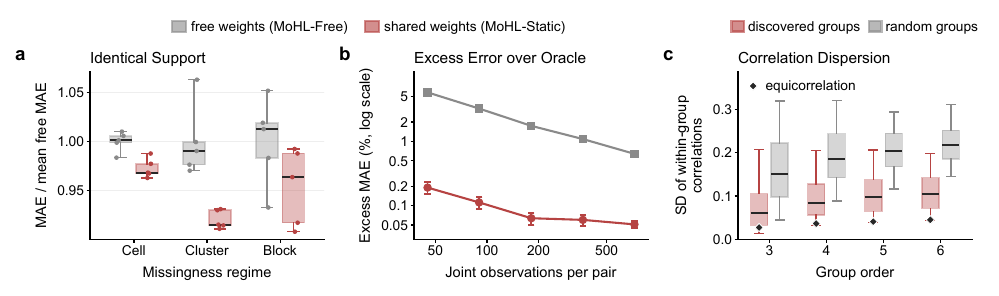}
\caption{\textbf{Shared versus free weights}: (a) shared versus free weights on identical support (PEMS-BAY, $N = 100$, five seeds); (b) excess error over the oracle on data from MoHL's model, mean and $95\%$ interval; (c) correlation dispersion within discovered and random groups.}
\label{fig:tied}
\end{figure}

\paragraph{Shared versus free weights on real data.}
MoHL-Static and MoHL-Free differ only in how the same couplings are parameterized \eqref{eq:untied}. Sharing lowered MAE in all three regimes, by up to $8.0\%$ under Cluster (Figure~\ref{fig:tied}a), and the advantage held in all six regimes and at both subnetwork sizes (Table~\ref{tab:tied}). MoHL-Free has the same support and more parameters, so the gain cannot come from extra connectivity. The gain comes from pooling evidence: each shared weight is estimated from the joint observations of every pair it governs, whereas each free weight sees only those of its own pair.

\paragraph{Gap to the Bayes-optimal estimator.} Theorem~\ref{thm:factor}(b) provides a benchmark: under MoHL's Gaussian group model, no estimator beats MoHL at the true weights. On data drawn from this model, learned shared weights came within $0.2\%$ of this oracle, whereas free weights stayed above it (Figure~\ref{fig:tied}b). The free model's support contains the true coupling, so its excess error is pure estimation error. That excess shrank as joint observations grew, as Proposition~\ref{prop:tying} predicts, but it remained large even at the few hundred joint observations per pair that a week of half-missing data provides.

\paragraph{Homogeneity of the discovered groups.} The group model implies that the members of a group are equally correlated. Real discovered groups came close to this ideal: their within-group correlations varied far less than those of random groups (Figure~\ref{fig:tied}c). The mismatch $\zeta^2$ of Proposition~\ref{prop:tying} is thus small, the setting in which sharing pays while evidence is limited.

\subsection{What the discovered groups capture}
\label{sec:exp-road}

\begin{figure}[!ht]
\centering
\includegraphics[width=0.96\linewidth]{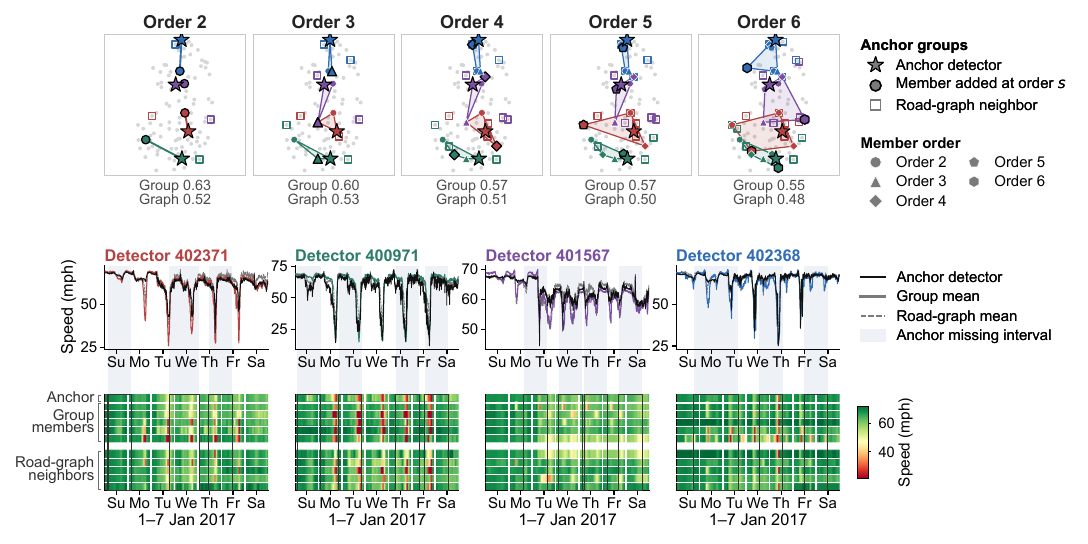}
\caption{\textbf{Discovered groups versus the road graph} on PEMS-BAY under Block: (top) t-SNE of detectors from hidden-cell affinity, with each anchor's order-$s$ group and nearest road-graph neighbors; (middle) anchor speed with the means of its group and road-graph neighbors; (bottom) speed heatmaps of the anchor, its group and its road-graph neighbors.}
\label{fig:tsne-groups}
\end{figure}

\paragraph{Group relations on hidden data.}
If the discovered groups reflect real relations rather than noise in the observed cells, they should hold on cells that MoHL never saw. MoHL builds its groups from observed cells only, so the hidden cells provide this test. Under Block, the discovered groups were more strongly related on hidden data than the same number of road-graph neighbors, at every order on PEMS-BAY (Figure~\ref{fig:tsne-groups}) and at order six on PEMS03 and PEMS04, with mean hidden-period correlations of $0.84$ against $0.65$ on PEMS03 and $0.70$ against $0.55$ on PEMS04. Road-graph neighbors often tracked an anchor as well, but the group of detector~401567 tracked a lasting drop on day three that its road-graph neighbors largely missed. A pair is observed together $3.5$--$16$ times as often as a group of six (Appendix~\ref{app:proof-rate}), so pairwise affinities suffice to find these groups, as Lemma~\ref{lem:pairwise-rate} anticipates. The road graph is thus a good but incomplete proxy for which sensors move together, and the part it misses is what the discovered groups add.

\begin{figure}[!ht]
\centering
\includegraphics[width=\linewidth]{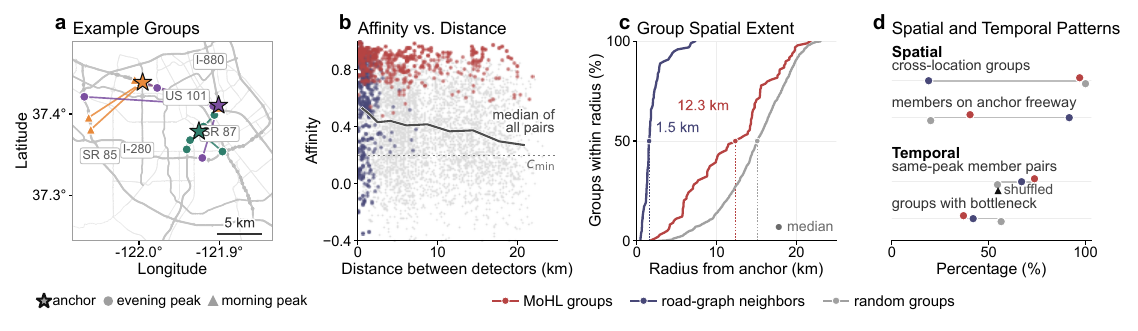}
\caption{\textbf{Discovered groups on the road network} (PEMS-BAY, order-6 groups under Block): (a) example groups; (b) affinity against distance; (c) spatial extent; (d) spatial and temporal patterns of MoHL groups, road-graph neighbors and random groups.}
\label{fig:group-road}
\end{figure}

\paragraph{Spatial extent of the groups.}
Figure~\ref{fig:group-road} places the order-6 groups on the freeway network, assigning each detector to its nearest freeway. We call a group \emph{cross-location} if it spans several freeways and is not confined to within $3$~km of one interchange. Almost all MoHL groups ($97\%$) were cross-location groups, against $19\%$ of the sets of road-graph neighbors, and the median group extended $12.3$~km from its anchor (Figure~\ref{fig:group-road}c). Affinity explains this reach: it decayed only slowly with distance, so strongly related detectors could lie far apart (Figure~\ref{fig:group-road}b). The groups are not arbitrary, however. Two in five members lay on the anchor's own freeway, twice the share for random groups, so a typical group combined a stretch of the anchor's own freeway with related detectors on other freeways.

\paragraph{Peak periods and bottlenecks within groups.}
We label each detector by its peak period, the time of day of its lowest median speed. Pairs within a MoHL group shared their peak period $74\%$ of the time, against $55\%$ expected by chance ($p<0.001$; Figure~\ref{fig:group-road}d). The groups thus joined detectors that load at the same time of day, as detectors serving the same commuting flows do, even across freeways. Recurrent bottlenecks, in contrast, were not over-represented. MoHL groups therefore encode shared demand rather than shared congestion points, with peak timing standing in for the direction of demand, which the public data do not record. For imputation, a detector's most informative neighbors are thus those that share its demand pattern, which distance on the road network only partly reflects.

\subsection{Resilience under detector failures}
\label{sec:exp-stress}

\begin{figure}[!ht]
\centering
\includegraphics[width=\linewidth]{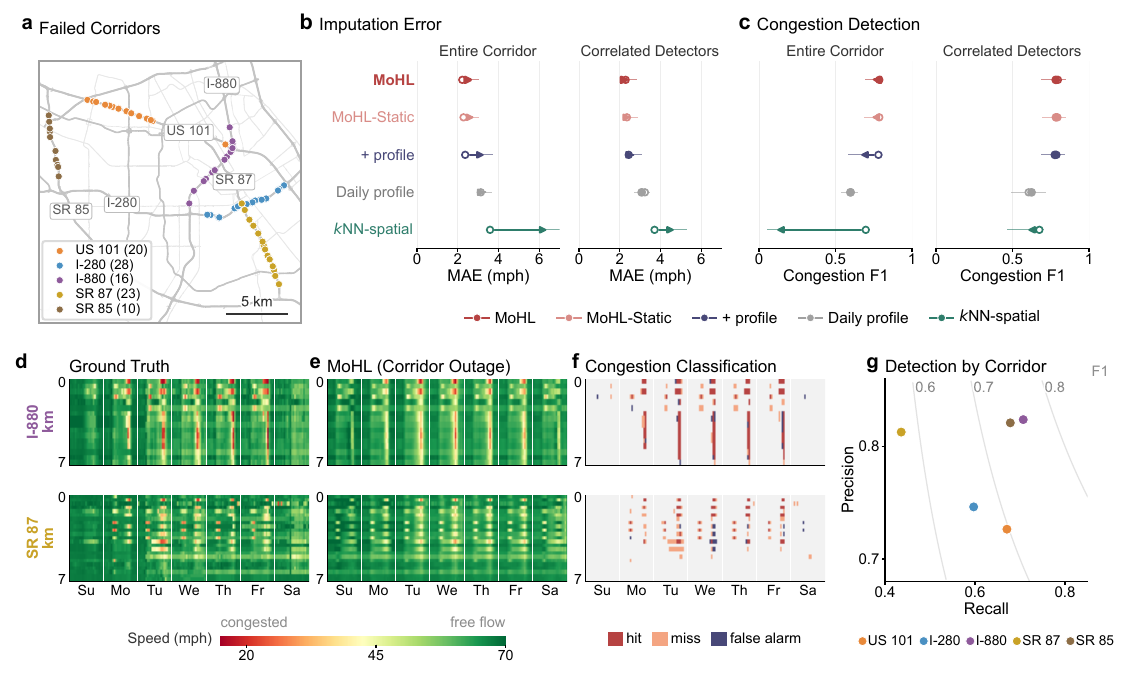}
\caption{\textbf{Resilience to detector failures} (PEMS-BAY): (a) the five failed corridors; (b) MAE and (c) congestion F1 on failed cells, from random failures (circle) to correlated failures of the same size (arrowhead, $95\%$ interval); (d--f) ground truth, MoHL's reconstruction and congestion classification on I-880 and SR~87; (g) congestion precision and recall per corridor.}
\label{fig:stress}
\end{figure}

\paragraph{Failure scenarios.}
MoHL's groups reach beyond a detector's own corridor. MoHL should therefore stay accurate when an entire corridor fails, whereas methods that borrow from nearby detectors should fail with it. Starting from the native PEMS-BAY mask of the seven-day window, we remove the failed detector-times and impute them. We test two disruptions: every detector of one of five freeways failing for a full weekday, and six correlated detectors failing together. Appendix~\ref{app:stress} gives the protocol and a third, progressive scenario in which detectors fail one by one until the whole network is down. MoHL is compared with MoHL-Static, MoHL-Base with the daily profile ($+$\,profile), the daily profile alone, $k$NN-spatial and linear interpolation in time. Errors are measured on the failed cells, and congestion, a speed below $45$~mph, is scored by F1.

\paragraph{Correlated versus random failures.}
Losing an entire corridor raised MoHL's error only modestly, from $2.22$ to $2.79$~mph, whereas the error of $k$NN-spatial rose to $6.41$~mph and its congestion F1 fell to $0.11$ (Figure~\ref{fig:stress}b,~c). Under corridor outages, MoHL kept both the lowest error and the best congestion detection. The contrast reflects where each method finds its evidence. $k$NN-spatial relies on nearby detectors, which fail with the corridor, whereas MoHL's groups reach detectors on other freeways. When six correlated detectors failed together, MoHL's error rose even less, because overlapping groups still cover a detector whose own group is down. Correlated failures are the realistic case, since power and communication outages can take out neighboring detectors together, so robustness to them matters more than accuracy under random dropout.

\paragraph{Congestion detection on failed corridors.}
Figure~\ref{fig:stress}d--f reconstructs two corridors over the week, each missing for one day at a time. On I-880, MoHL placed the recurrent evening queues at the right positions and times, using only detectors on other freeways. On SR~87, whose slowdowns are short and scattered, the congestion it reconstructed was still correct $81\%$ of the time. Across all five corridors, precision stayed between $0.73$ and $0.82$ (Figure~\ref{fig:stress}g), so a queue that MoHL reconstructs on a failed corridor is usually real, which is what traffic monitoring needs from an imputation.

\FloatBarrier
\section{Conclusion}
\label{sec:conclusion}
\label{sec:discussion}
\label{app:limits}

We introduced \textbf{MoHL}, a mixture of hypergraph Laplacians that learns group-wise relational structure from incomplete spatiotemporal data. MoHL discovers overlapping sensor groups from pairwise evidence, gives each interaction order its own Laplacian component and learns the weight of each order by masking observed cells in gaps that mimic the missingness pattern. Across three traffic networks, it attained the lowest error among the evaluated methods under day-long gaps, with orders of magnitude fewer learned parameters than deep imputers. Higher-order groups mattered most when temporal evidence was weak. On PEMS-BAY, shared weights outperformed free weights on identical support in every regime and at both subnetwork sizes. The discovered groups join detectors on different freeways that share their commuting peak. MoHL also kept the lowest error when a whole corridor failed. Higher-order structure is thus most useful when missingness makes relational evidence scarce, and sharing statistical strength within groups offers a compact way to learn cross-sensor dependence for imputation.

Several directions follow from these results. MoHL learns its structure from the incomplete window it imputes. A natural extension is to combine this evidence with longer histories, related deployments or network topology, and to learn how much each source deserves. Profiles by day type would extend the daily profile, which is learned mostly from weekdays, to holidays and weekends. Groups that form, dissolve or reweight over time would let the mixture follow relations that shift with incidents, weather or the time of day. Shared and free weights are two ends of a continuum. Hierarchical priors that shrink each pair weight toward its group weight, according to the pair's joint observations, would let the degree of sharing adapt to each pair. The learned mixture Laplacian could also serve deep imputers as their graph or as a structural prior. Finally, discovering groups from pairwise evidence is not specific to traffic and applies to other networked measurements with structured missingness, such as air-quality and energy monitoring.

\bibliographystyle{plainnat}
\bibliography{references}

\clearpage
\appendix
\section*{\Large Appendix}

\section{Additional results}
\label{app:results}

This appendix extends the experiments of Section~\ref{sec:exp}.

\subsection{All six regimes}
\label{app:all-regimes}

Table~\ref{tab:all-baselines} extends Table~\ref{tab:compare} to all six regimes ($N = 100$, five seeds, masks shared across rows). The columns run from scattered dropout (Cell) through Cluster to blocks of 30 minutes, six hours and a day. The block regimes are named by their length in five-minute steps (Block-6, Block-72, Block-288), with Block-288 the Block of the main text. The Mixed mask adds sensors that never report.

\paragraph{Comparison across regimes.} MoHL led under day-long gaps, with the lowest error on all three networks under Block-288. Six-hour gaps left enough evidence for the baselines: under Block-72, HaLRTC, GRIN, SPIN and ImputeFormer beat MoHL on PEMS-BAY, and HaLRTC and ImputeFormer on PEMS03; HaLRTC also won under Block-6 on PEMS03, and SPIN and ImputeFormer under Mixed on PEMS04. Low-rank completion was the most sensitive to the mask: under Mixed on PEMS-BAY, HaLRTC and WDGTC ended above even the sensor mean, because missing days and silent sensors remove the scattered entries they rely on.

\begin{table}[!ht]
\centering\footnotesize
\setlength{\tabcolsep}{5pt}
\caption{\textbf{MoHL against the baselines in all six regimes}: MAE on three networks; best in bold, second best underlined.}
\label{tab:all-baselines}
\begin{tabular}{@{}l *{6}{c}@{}}
\toprule
Method & Cell & Cluster & Block-6 & Block-72 & Block-288 & Mixed \\
\midrule
\multicolumn{7}{@{}l}{\textit{PEMS-BAY (mph)}} \\
Sensor mean & 5.044 & 5.124 & 5.023 & 5.221 & 5.251 & 5.136 \\
$k$NN-spatial & 3.450 & 4.428 & 3.320 & 3.395 & 3.469 & 3.423 \\
HaLRTC & 1.116 & 1.818 & 1.373 & 2.407 & 3.203 & 19.81 \\
LETC & 1.748 & 2.147 & 1.891 & 2.847 & 3.178 & 2.581 \\
WDGTC & 2.142 & 3.577 & 2.241 & 3.291 & 13.73 & 18.21 \\
\addlinespace[2pt]
GRIN & \underline{0.706} & \underline{1.291} & \underline{1.071} & 2.403 & \underline{3.029} & 2.499 \\
SPIN & 0.840 & 1.319 & 1.073 & \underline{2.267} & 3.162 & \underline{2.455} \\
ImputeFormer & 0.873 & 1.326 & 1.119 & \textbf{2.207} & 3.202 & 2.674 \\
\addlinespace[2pt]
\textbf{MoHL} & \textbf{0.699} & \textbf{1.262} & \textbf{1.069} & 2.461 & \textbf{2.751} & \textbf{2.453} \\
\midrule
\multicolumn{7}{@{}l}{\textit{PEMS03 (vehicles per 5 min)}} \\
Sensor mean & 70.03 & 69.55 & 70.32 & 73.29 & 71.02 & 84.57 \\
$k$NN-spatial & 84.08 & 84.50 & 82.53 & 82.02 & 85.89 & 84.66 \\
HaLRTC & \underline{9.800} & \underline{11.55} & \textbf{10.02} & \textbf{12.44} & 25.11 & 59.93 \\
LETC & 36.79 & 54.49 & 49.40 & 82.50 & 89.63 & 68.96 \\
WDGTC & 13.42 & 14.06 & 13.50 & 18.46 & 40.03 & 60.89 \\
\addlinespace[2pt]
GRIN & 10.36 & 12.61 & 11.53 & 20.15 & 21.57 & 50.18 \\
SPIN & 11.65 & 13.09 & 12.65 & 17.27 & 20.03 & \underline{43.70} \\
ImputeFormer & 10.75 & 11.58 & 10.91 & \underline{13.31} & \underline{14.72} & 44.78 \\
\addlinespace[2pt]
\textbf{MoHL} & \textbf{9.626} & \textbf{11.43} & \underline{10.77} & 13.61 & \textbf{14.06} & \textbf{42.81} \\
\midrule
\multicolumn{7}{@{}l}{\textit{PEMS04 (vehicles per 5 min)}} \\
Sensor mean & 117.6 & 118.6 & 118.5 & 121.2 & 117.9 & 125.9 \\
$k$NN-spatial & 94.46 & 107.7 & 90.67 & 88.06 & 92.41 & 93.70 \\
HaLRTC & 22.95 & 29.18 & 24.10 & 32.33 & 81.98 & 94.77 \\
LETC & 52.33 & 69.26 & 69.84 & 110.6 & 118.7 & 89.10 \\
WDGTC & 26.81 & 31.88 & 27.51 & 31.77 & 103.0 & 90.58 \\
\addlinespace[2pt]
GRIN & \underline{20.26} & 25.97 & 22.37 & 42.63 & 50.03 & 62.82 \\
SPIN & 20.69 & 24.47 & 22.34 & 33.41 & 39.62 & \textbf{51.16} \\
ImputeFormer & 20.75 & \underline{23.16} & \underline{21.83} & \underline{30.17} & \underline{31.66} & \underline{54.84} \\
\addlinespace[2pt]
\textbf{MoHL} & \textbf{19.53} & \textbf{23.08} & \textbf{21.47} & \textbf{29.26} & \textbf{31.31} & 58.72 \\
\bottomrule
\end{tabular}

\end{table}

\FloatBarrier
\subsection{Imputations at more sensors}
\label{app:traces-all}

\begin{figure}[!ht]
\centering
\includegraphics[width=\linewidth]{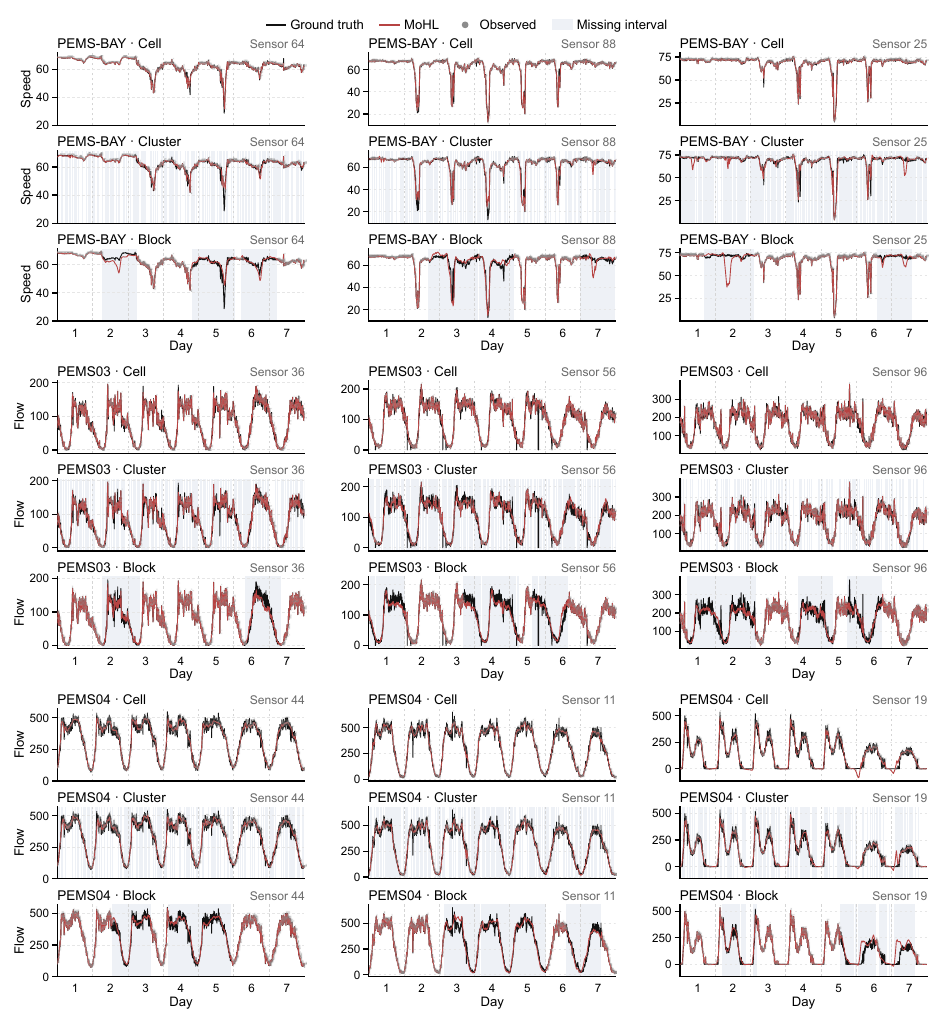}
\caption{\textbf{MoHL's imputations under Cell, Cluster and Block}, three sensors per network (speed in mph, flow in vehicles per 5~min).}
\label{fig:traces-all}
\end{figure}

\paragraph{Traces at other sensors.} Figure~\ref{fig:traces-all} extends Figure~\ref{fig:traces} to two more sensors per network, at the 25th and 75th percentiles of MoHL's per-sensor error under Block. The median sensor is that of Figure~\ref{fig:traces}, and each sensor keeps its column across regimes. Under Cell and Cluster the imputations stayed close to the truth throughout. Inside day-long gaps they kept the daily cycle and, on PEMS-BAY, many congestion dips of the day they replace, even at the 75th percentile. The accuracy of Table~\ref{tab:compare} is therefore not carried by a few easy sensors.

\begin{figure}[!htbp]
\centering
\includegraphics[width=\linewidth]{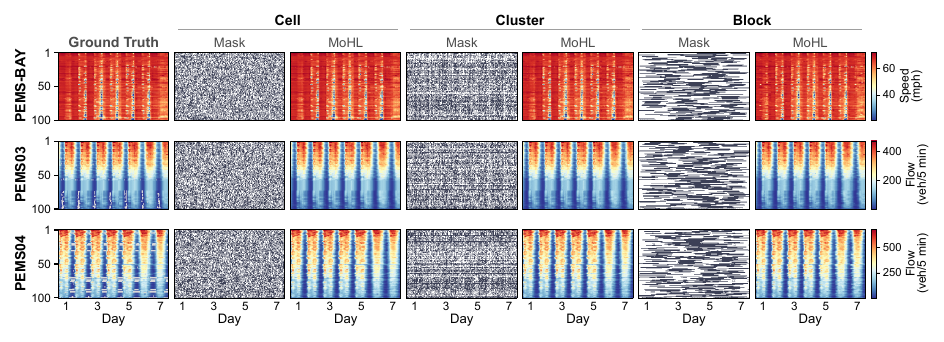}
\caption{\textbf{MoHL's reconstruction of every sensor}: ground truth, mask (dark: missing) and reconstruction under Cell, Cluster and Block.}
\label{fig:heatmap-recon}
\end{figure}

\paragraph{Full reconstructions.} Figure~\ref{fig:heatmap-recon} shows every sensor at once, with each regime's mask beside MoHL's reconstruction, in the sensor order and color scale of Figure~\ref{fig:heatmap-baselines}. In all three regimes the reconstruction reproduces the daily bands and the differences between sensors in the ground truth, and the blocks of the Block mask cannot be made out in it.

\FloatBarrier
\subsection{Ablation in all six regimes}
\label{app:all-ablation}

Table~\ref{tab:all-ablation} extends Table~\ref{tab:ablation} to all six regimes, on the masks of Table~\ref{tab:all-baselines}. Under Block-288, the ladder rows credit most of the gain over MoHL-Base to the daily profile and most of the rest to the discovered groups. Under Mixed, the daily profile raised the error of MoHL-Base on all three networks, on PEMS-BAY from $2.341$ to $2.608$, which left MoHL-Base the best method there.

\begin{table}[!ht]
\centering\footnotesize
\setlength{\tabcolsep}{5pt}
\caption{\textbf{Ablation of MoHL in all six regimes}: MAE of each model of Table~\ref{tab:ablation} on three networks; best in bold, second best underlined, ranked on unrounded values.}
\label{tab:all-ablation}
\begin{tabular}{@{}l *{6}{c}@{}}
\toprule
Method & Cell & Cluster & Block-6 & Block-72 & Block-288 & Mixed \\
\midrule
\multicolumn{7}{@{}l}{\textit{PEMS-BAY (mph)}} \\
MoHL-Base & 0.732 & 1.273 & 1.108 & 2.856 & 3.128 & \textbf{2.341} \\
+ profile & 0.732 & \underline{1.267} & 1.111 & 2.657 & 2.882 & 2.608 \\
+ pairs & 0.703 & 1.294 & 1.071 & 2.519 & 2.814 & 2.520 \\
MoHL-Static & \underline{0.700} & 1.268 & \textbf{1.069} & \textbf{2.424} & 2.808 & 2.479 \\
\textbf{MoHL} & \textbf{0.699} & \textbf{1.262} & \underline{1.069} & 2.461 & \underline{2.751} & \underline{2.453} \\
\addlinespace[2pt]
MoHL-Free & 0.719 & 1.470 & 1.198 & \underline{2.428} & \textbf{2.695} & 2.504 \\
\midrule
\multicolumn{7}{@{}l}{\textit{PEMS03 (vehicles per 5 min)}} \\
MoHL-Base & 11.11 & 13.31 & 12.85 & 24.17 & 22.98 & 43.67 \\
+ profile & 10.99 & 12.97 & 12.46 & 16.12 & 16.07 & 44.39 \\
+ pairs & 9.984 & 12.19 & 11.37 & 14.23 & 15.32 & 43.02 \\
MoHL-Static & 9.652 & 11.45 & 10.86 & \textbf{13.56} & \underline{14.12} & \textbf{42.80} \\
\textbf{MoHL} & \underline{9.626} & \underline{11.43} & \underline{10.77} & \underline{13.61} & \textbf{14.06} & 42.81 \\
\addlinespace[2pt]
MoHL-Free & \textbf{9.077} & \textbf{11.24} & \textbf{10.47} & 13.99 & 14.25 & \underline{42.81} \\
\midrule
\multicolumn{7}{@{}l}{\textit{PEMS04 (vehicles per 5 min)}} \\
MoHL-Base & 19.82 & 25.17 & 22.20 & 46.78 & 52.71 & \textbf{53.33} \\
+ profile & 19.63 & 23.61 & 21.68 & 32.01 & 34.57 & 60.20 \\
+ pairs & 19.62 & 23.30 & 21.57 & 30.11 & 32.28 & 59.03 \\
MoHL-Static & \underline{19.61} & \underline{23.14} & \underline{21.53} & 29.48 & \underline{31.38} & 58.84 \\
\textbf{MoHL} & \textbf{19.53} & \textbf{23.08} & \textbf{21.47} & \textbf{29.26} & \textbf{31.31} & \underline{58.72} \\
\addlinespace[2pt]
MoHL-Free & 19.61 & 23.71 & 22.11 & \underline{29.43} & 32.13 & 59.11 \\
\bottomrule
\end{tabular}

\end{table}

\FloatBarrier
\subsection{Gain by temporal gap}
\label{app:strata}

Equation~\eqref{eq:missing-cell} predicts that long gaps shift the estimate from a sensor's own history to other sensors. Table~\ref{tab:strata} tests this by splitting the evaluation cells by their temporal gap and comparing MoHL with $+$\,profile (the daily profile and the prior graph) on PEMS-BAY. MoHL improved on $+$\,profile by about $4\%$ under Cell but by $11$--$51\%$ under day-long gaps, so the gain followed the length of the regime's gaps. Under Block-72 and Block-288 it was largest at the edges of a gap and still $11$--$12\%$ more than six hours from the nearest reading. Benchmarks built on scattered dropout therefore understate the value of relational structure.

\begin{table}[!ht]
\centering\footnotesize
\setlength{\tabcolsep}{5pt}
\caption{\textbf{Gain over $+$\,profile by temporal gap}: MAE (mph) of $+$\,profile $\to$ MoHL on PEMS-BAY ($N = 60$, five seeds).}
\label{tab:strata}
\begin{tabular}{@{}l *{4}{c}@{}}
\toprule
 & \multicolumn{4}{c}{temporal gap of the evaluation cell (steps)} \\
\cmidrule(lr){2-5}
Regime & $1$--$2$ & $3$--$12$ & $13$--$72$ & $>72$ \\
\midrule
Cell & 0.75$\to$0.72 ($-4\%$) & 1.14$\to$1.10 ($-4\%$) & --- & --- \\
Cluster & 0.97$\to$0.95 ($-2\%$) & 1.57$\to$1.49 ($-5\%$) & 2.56$\to$2.10 ($-18\%$) & --- \\
Block-6 & 0.94$\to$0.91 ($-3\%$) & 1.38$\to$1.34 ($-3\%$) & 2.15$\to$1.91 ($-11\%$) & --- \\
Block-72 & 1.74$\to$1.06 ($-39\%$) & 2.47$\to$2.01 ($-19\%$) & 3.04$\to$2.54 ($-16\%$) & 2.66$\to$2.33 ($-12\%$) \\
Block-288 & 2.77$\to$1.35 ($-51\%$) & 2.98$\to$2.29 ($-23\%$) & 2.97$\to$2.60 ($-12\%$) & 2.92$\to$2.60 ($-11\%$) \\
Mixed & 0.69$\to$0.69 ($-1\%$) & 2.30$\to$1.94 ($-16\%$) & 3.02$\to$2.30 ($-24\%$) & 4.70$\to$4.62 ($-2\%$) \\
\bottomrule
\end{tabular}

\end{table}

\subsection{Shared versus free weights in all six regimes}
\label{app:tied}

Table~\ref{tab:tied} extends Figure~\ref{fig:tied}a to all six regimes and to $N=60$. Shared weights had the lower error in every regime and at both sizes, by about $4\%$ on average.

\begin{table}[!ht]
\centering\footnotesize
\setlength{\tabcolsep}{4pt}
\caption{\textbf{Shared versus free weights in all six regimes}: MAE (mph) on PEMS-BAY over five seeds, with the relative change $\Delta$.}
\label{tab:tied}
\begin{tabular}{@{}l ccc c ccc@{}}
\toprule
 & \multicolumn{3}{c}{$N = 60$} & & \multicolumn{3}{c}{$N = 100$} \\
\cmidrule(lr){2-4}\cmidrule(lr){6-8}
Regime & Free & Shared & $\Delta$ (\%) & & Free & Shared & $\Delta$ (\%) \\
\midrule
Cell & 0.761 & \textbf{0.739} & $-2.8$ & & 0.735 & \textbf{0.715} & $-2.7$ \\
Cluster & 1.429 & \textbf{1.290} & $-9.7$ & & 1.368 & \textbf{1.259} & $-8.0$ \\
Block-6 & 1.239 & \textbf{1.136} & $-8.3$ & & 1.191 & \textbf{1.115} & $-6.4$ \\
Block-72 & 2.453 & \textbf{2.395} & $-2.3$ & & 2.446 & \textbf{2.389} & $-2.3$ \\
Block-288 & 2.725 & \textbf{2.672} & $-1.9$ & & 2.688 & \textbf{2.563} & $-4.6$ \\
Mixed & 2.448 & \textbf{2.438} & $-0.4$ & & 2.604 & \textbf{2.577} & $-1.0$ \\
\midrule
Mean $\Delta$ (\%) & \multicolumn{3}{c}{$-4.2$} & & \multicolumn{3}{c}{$-4.1$} \\
\bottomrule
\end{tabular}

\end{table}

\FloatBarrier
\subsection{Stress tests}
\label{app:stress}

\paragraph{Protocol.} All scenarios start from the native PEMS-BAY mask of the window of Table~\ref{tab:compare} (1--7 January 2017), which has no missing cells, and each detector is assigned to the nearest OpenStreetMap motorway within $300$~m. \emph{Corridor outages:} each of the five freeways with at least six of the $100$ detectors fails on each of four weekdays (3--6 January), paired with three random failures of the same size on the same day. \emph{Six correlated detectors:} each of the strongest pairwise disjoint order-6 groups fails for one weekday, paired with three random sets of six. \emph{Progressive failure} (five seeds): detector $i$ is down while a schedule $q(t)$ exceeds a threshold $u_i\sim\mathcal U(0,0.8)$, so a share $\min(1,q(t)/0.8)$ of the detectors is down; $q$ rises linearly from $0$ at the end of day 1 to $0.8$ at mid-day 4, when every detector is down, and returns to $0$ by the end of day 7. Table~\ref{tab:stress} gives the numbers behind Figure~\ref{fig:stress}b,~c and the MAE and F1 of the progressive scenario.

\begin{table}[!ht]
\centering\scriptsize
\setlength{\tabcolsep}{2.4pt}
\caption{\textbf{Stress tests} on PEMS-BAY: MAE on the failed cells, averaged over scenarios, and congestion F1 pooled over the failed cells; best in bold, ranked on unrounded values.}
\label{tab:stress}
\begin{tabular}{@{}l *{6}{c} c *{6}{c}@{}}
\toprule
 & \multicolumn{6}{c}{MAE (mph) on failed cells} & & \multicolumn{6}{c}{Congestion F1 on failed cells} \\
\cmidrule(lr){2-7}\cmidrule(lr){9-14}
Scenario & MoHL & MoHL-Static & +\,profile & Daily profile & $k$NN-spatial & Interp. & & MoHL & MoHL-Static & +\,profile & Daily profile & $k$NN-spatial & Interp. \\
\midrule
Corridor outage & \textbf{2.79} & 2.83 & 3.33 & 3.37 & 6.41 & 6.36 & & \textbf{0.74} & 0.73 & 0.66 & 0.60 & 0.11 & 0.00 \\
Random, same size & \textbf{2.22} & 2.28 & 2.36 & 3.13 & 3.58 & 5.86 & & \textbf{0.79} & 0.79 & 0.78 & 0.60 & 0.70 & 0.00 \\
Six correlated & \textbf{2.50} & 2.56 & 2.70 & 3.07 & 4.71 & 5.84 & & 0.78 & \textbf{0.78} & 0.78 & 0.62 & 0.60 & 0.00 \\
Random six & \textbf{2.29} & 2.35 & 2.44 & 3.24 & 3.71 & 6.05 & & \textbf{0.80} & 0.79 & 0.79 & 0.60 & 0.68 & 0.00 \\
\addlinespace[2pt]
Progressive, worsening & \textbf{3.72} & 3.74 & 3.78 & 4.51 & 4.40 & 6.73 & & 0.40 & 0.40 & 0.38 & 0.18 & \textbf{0.54} & 0.08 \\
Progressive, recovering & \textbf{3.92} & 3.93 & 4.01 & 4.56 & 4.66 & 6.60 & & 0.47 & 0.47 & 0.45 & 0.27 & \textbf{0.56} & 0.10 \\
\bottomrule
\end{tabular}

\end{table}

\paragraph{Progressive network failure.}
In the progressive scenario (Table~\ref{tab:stress}, last two rows), MoHL kept the lowest error on the failed cells both while failures spread and while they receded. With most detectors down, its congestion recall over the scenario fell below that of $k$NN-spatial ($0.31$ against $0.47$), and on the New Year holiday and the Saturday its error exceeded that of $k$NN-spatial ($4.66$ against $3.38$ and $3.15$ against $2.70$~mph). Without enough live evidence, MoHL's estimates shrink toward a daily profile learned mostly from weekdays, which keeps errors low but misses atypical congestion.

\subsection{Interaction orders on test cells within the observed data}
\label{app:more-tables}

Table~\ref{tab:order-test} repeats Table~\ref{tab:order} in all six regimes on the test cells within the observed data, $\Omega_{\rm sel}$, with the structure built from $\Omega_{\rm fit}$ and the weights learned on $\Omega_{\rm learn}$. Under the block regimes, orders three and four again carried the gain and orders five and six added little, so the choice of orders can be read from observed data alone.

\begin{table}[!ht]
\centering\scriptsize
\setlength{\tabcolsep}{3pt}
\caption{\textbf{Gain by interaction order on test cells within the observed data}: MAE on PEMS-BAY ($N=60$, two seeds) as orders are added, with the change $\Delta$ from pairs.}
\label{tab:order-test}
\begin{tabular}{@{}l l *{5}{c}@{}}
\toprule
 & & \multicolumn{5}{c}{relational interaction groups, by order} \\
\cmidrule(lr){3-7}
Regime & & Pairs ($s{=}2$) & $s{=}3$ & $s{=}4$ & $s{=}5$ & $s{=}6$ \\
\midrule
\multirow{2}{*}{Cell} & MAE & 1.041 & \textbf{1.040} & \textbf{1.040} & \textbf{1.040} & \textbf{1.040} \\
 & $\Delta$ & --- & $-0.001$ & $-0.001$ & $-0.001$ & $-0.001$ \\[1pt]
\multirow{2}{*}{Cluster} & MAE & 1.954 & \textbf{1.920} & 1.925 & 1.922 & 1.922 \\
 & $\Delta$ & --- & $-0.034$ & $-0.029$ & $-0.032$ & $-0.032$ \\[1pt]
\multirow{2}{*}{Block-6} & MAE & 1.628 & 1.607 & 1.604 & 1.604 & \textbf{1.602} \\
 & $\Delta$ & --- & $-0.021$ & $-0.024$ & $-0.024$ & $-0.026$ \\[1pt]
\multirow{2}{*}{Block-72} & MAE & 3.128 & 2.976 & 2.905 & 2.900 & \textbf{2.885} \\
 & $\Delta$ & --- & $-0.152$ & $-0.223$ & $-0.228$ & $-0.243$ \\[1pt]
\multirow{2}{*}{Block-288} & MAE & 3.689 & 3.631 & 3.602 & \textbf{3.599} & \textbf{3.599} \\
 & $\Delta$ & --- & $-0.058$ & $-0.087$ & $-0.090$ & $-0.090$ \\[1pt]
\multirow{2}{*}{Mixed} & MAE & 2.491 & 2.477 & \textbf{2.441} & 2.442 & 2.442 \\
 & $\Delta$ & --- & $-0.014$ & $-0.050$ & $-0.049$ & $-0.049$ \\[1pt]
\bottomrule
\end{tabular}

\end{table}

\subsection{Learning curves of the mixture}
\label{app:learning-curves}

\begin{figure}[!ht]
\centering
\includegraphics[width=0.83\linewidth]{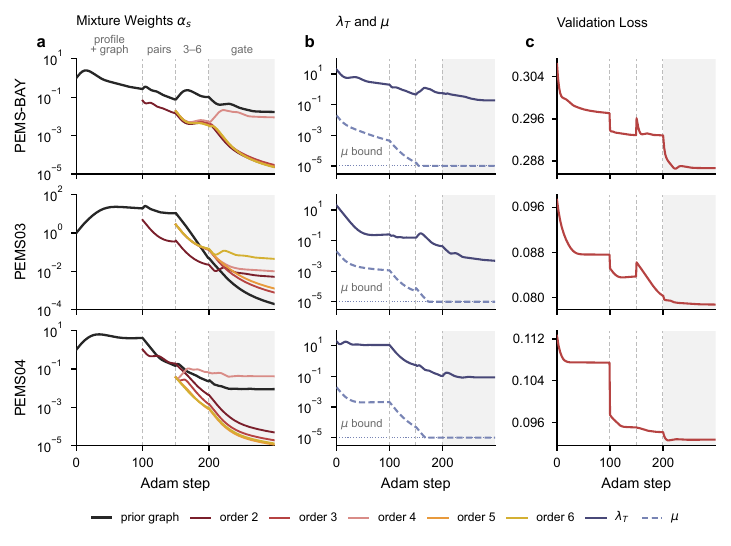}
\caption{\textbf{Learning curves of MoHL under Block}: (a) mixture weights $\alpha_k$, (b) $\lambda_T$ and $\mu$, and (c) validation loss over the staged fit, one row per network; the gate stage is shaded.}
\label{fig:learning-curves}
\end{figure}

\paragraph{Convergence of the staged fit.} Figure~\ref{fig:learning-curves} traces the parameters and the validation loss through the staged fit of the Block runs of Table~\ref{tab:compare}. The validation loss fell at every stage, apart from a brief rise on PEMS-BAY and PEMS03 when orders three to six entered; the kept iterate, the one with the lowest objective, was unaffected.

The networks settled differently. On PEMS03 the groups took over and the prior-graph weight fell by more than two orders of magnitude when orders three to six entered, on PEMS04 order four dominated, and on PEMS-BAY the prior graph kept the largest weight. The data identify these weights only coarsely (Appendix~\ref{app:comembership}), so the curves describe the optimization, not the weights. On every network the ridge $\mu$ fell to its lower bound: the objective prefers no ridge, and the bound keeps the solve positive definite (Lemma~\ref{lem:wellposed}). In the gate stage the $\alpha_k$ are not comparable with earlier stages, since each edge is also scaled by its gate.

\FloatBarrier

\section{Method details}
\label{app:method-details}

This appendix gives the full definitions behind Sections~\ref{sec:setup} and~\ref{sec:method}.

\subsection{Notation}
\label{app:notation-table}

Table~\ref{tab:notation} collects the symbols used across sections; symbols local to one proof are defined there.

\begin{table}[!ht]
\centering\footnotesize
\caption{\textbf{Notation.}}
\label{tab:notation}
\begin{tabularx}{\linewidth}{@{}l>{\raggedright\arraybackslash}X@{}}
\toprule
\textbf{Symbol} & \textbf{Meaning} \\
\midrule
\multicolumn{2}{@{}l}{\textit{Data and missingness}} \\
$N$, $T$ & number of sensors and of time steps \\
$X^*$ & true values, $N\times T$ \\
$M$, $\Omega$ & observation mask and the set of observed cells \\
$M'$, $M^{\rm fit}$ & mask of the current solve and the fitting mask \eqref{eq:split} \\
$\Omega^c$, $\Omega_{\rm eval}$, MAE & missing cells, evaluation cells and the error criterion \eqref{eq:mae} \\
$n_{\rm day}$ & time steps per day ($288$ at five minutes) \\
$\widehat X$, $\widehat X^{\rm out}$ & estimate in residual units and imputation in original units \eqref{eq:output} \\
$Y$, $\xi_{i,t}$ & readings on $\Omega$ and their noise, sub-Gaussian with scale $\sigma$ \\
$\pi_{i,t}$, $W$ & observation propensity, and observation weights of the solve ($W=M'$) \\
$\bar Y_i$, $s_i$ & mean and standard deviation of all of sensor $i$'s observed readings \\
$L_T$ & temporal Laplacian $\tfrac14L_{\rm path}$, with $\rho(L_T)<1$ \\
$g_{i,t}$ & temporal gap: distance to the sensor's nearest observation \\
$n_{ij}$ & relational evidence: times sensors $i$ and $j$ are observed together \\
\addlinespace
\multicolumn{2}{@{}l}{\textit{Graphs and hypergraphs}} \\
$A$ & prior adjacency \\
$D_B$ & degree matrix $\operatorname{diag}(B\1)$ of an adjacency $B$ \\
$e$, $w_e$, $\Hcal$ & hyperedge (group), its weight, and the hypergraph \\
$P_e$ & hyperedge Laplacian $I_e-\1_e\1_e^\top/|e|$ \\
$A_{\Hcal}$ & adjacency induced by $\Hcal$, with $L_{\Hcal}=D_{A_{\Hcal}}-A_{\Hcal}$ \\
$\Hcal_s$ & retained groups of order $s$ \\
$L_{\rm prior}$, $L_{\Hcal_s}$ & $(D_A-A)/\rho_0$ and $\sum_{e\in\Hcal_s}P_e/\rho_s$, each with largest eigenvalue one \\
\addlinespace
\multicolumn{2}{@{}l}{\textit{Group discovery}} \\
$\tilde Y$, $P$ & standardized observations and the daily profile \\
$\bar R$ & one-hour moving average of the residual $R=M'\odot(\tilde Y-P)$ \eqref{eq:residual} \\
$C_{ij}$ & affinity \eqref{eq:affinity} \\
$\mathcal T_{ij}$ & times at which both smoothed residuals are defined \eqref{eq:affinity} \\
$n_{\min}$, $c_{\min}$ & minimum joint support of an affinity and minimum affinity of a group \\
$\mathcal G_i^{(s)}$, $S$ & group of order $s$ around anchor $i$, and the largest order \\
\addlinespace
\multicolumn{2}{@{}l}{\textit{MoHL}} \\
$\alpha$, $a_k$ & mixture weights $(\alpha_0,\alpha_2,\dots,\alpha_S)$ and $a_k=\log\alpha_k$ \\
$\rho_0$, $\rho_s$ & largest eigenvalues that scale the components \eqref{eq:components} \\
$A_k$ & unit-scaled adjacency of component $k$ \\
$L_S(\alpha)$, $A(\alpha)$ & spatial operator \eqref{eq:mohl-family} and its adjacency \eqref{eq:comembership} \\
$K^{(s)}$ & counts of order-$s$ groups containing each pair \\
$\lambda_T$, $\mu$ & temporal weight and ridge \\
$\theta$, $\Theta$ & learned parameters \eqref{eq:theta} and their box \eqref{eq:box} \\
$w_{ij}(t)$ & gated coupling \eqref{eq:egate} \\
$\mathcal A$, $\Psi$ & system operator \eqref{eq:operator} and adjoint state \\
$\Omega_{\rm fit}$, $\Omega_{\rm learn}$, $\Omega_{\rm sel}$ & fitting cells, masked cells and test cells within the observed data \eqref{eq:split} \\
$\mathcal J$, $\mathcal L$ & objective \eqref{eq:mohl-bilevel} and validation loss \\
$\ell_\delta$, $\delta$ & Huber loss and its threshold \\
$\gamma$, $\gamma_b$ & penalty weights \eqref{eq:mohl-bilevel} \\
$u$, $G$ & validation errors \eqref{eq:validation-error} and loss gradient \eqref{eq:adjoint} \\
$c_{i,t}$, $b_k$, $h^k_{i,t}$ & gate context \eqref{eq:context}, slopes and values \eqref{eq:egate} \\
$b_T$, $h^{T}_{i,t}$ & slopes and values of the temporal gate \eqref{eq:tgate} \\
$\kappa$ & clip bound of the gates (a condition number in Lemma~\ref{lem:wellposed}) \\
$\Gamma_\Theta$, $\kappa_\Theta$ & Lipschitz and log-term constants of Theorem~\ref{thm:generalization} \\
$\Lambda_k$ & coupling profile $\sum_{s\ge k}\alpha_s/\rho_s$ \\
$A_{\rm shared}$, $\mathcal E$, $m$ & group adjacency, its support and the number of supported pairs \eqref{eq:shared-support} \\
$A_{\rm free}$, $\vartheta$ & free-weight control \eqref{eq:untied} and its pair logits \\
$\zeta^2$ & group-structure mismatch (Proposition~\ref{prop:tying}) \\
\bottomrule
\end{tabularx}
\end{table}

\subsection{Group discovery in full}
\label{app:groups-full}

\paragraph{Standardization.} Sensors measure different quantities at different levels, so we standardize each one. Let $\Omega_i:=\{t:M_{i,t}=1\}$ be the observed times of sensor $i$, and let
\begin{equation}
\begin{aligned}
\bar Y_i &:= \frac{1}{|\Omega_i|}\sum_{t\in\Omega_i}Y_{i,t}, \qquad
s_i^2 := \frac{1}{|\Omega_i|}\sum_{t\in\Omega_i}(Y_{i,t}-\bar Y_i)^2,\\
\tilde Y_{i,t} &:= \frac{Y_{i,t}-\bar Y_i}{s_i} \ \ \text{for } (i,t)\in\Omega, \qquad \tilde Y_{i,t}:=0 \ \ \text{otherwise}.
\end{aligned}
\label{eq:standardize}
\end{equation}
The standardized reading $\tilde Y_{i,t}$ is how many standard deviations sensor $i$ is above its own mean.

\paragraph{Daily profile.} With $n_{\rm day}$ steps per day ($288$ at five minutes), we split the day into $48$ bins of 30 minutes and write $b(t)\in\{1,\dots,48\}$ for the bin of time $t$. Let $\mathcal B_i(b):=\{t : M'_{i,t}=1,\ b(t)=b\}$ be the observed times of sensor $i$ in bin $b$. The daily profile averages the standardized readings in each bin and then smooths over neighboring bins,
\begin{equation}
\begin{aligned}
\widehat P_i(b) &:= \frac{1}{|\mathcal B_i(b)|}\sum_{t\in\mathcal B_i(b)}\tilde Y_{i,t},\\
\bar P_i(b) &:= \tfrac14\,\widehat P_i(b-1) + \tfrac12\,\widehat P_i(b) + \tfrac14\,\widehat P_i(b+1),
\end{aligned}
\label{eq:profile}
\end{equation}
with bin indices taken cyclically, so that the day wraps around midnight. An empty bin ($\mathcal B_i(b)=\emptyset$) is first filled by interpolating its neighbors, and linear interpolation of $\bar P_i$ between bin centers gives the profile value $P_{i,t}$ at every time step.

\paragraph{Affinity.} To suppress five-minute noise, we first smooth each residual series with a one-hour moving average. With $\mathcal W_i(t):=\{t' : |t'-t|\le6,\ M'_{i,t'}=1\}$ the observed times within $\pm30$ minutes of $t$,
\begin{equation}
\bar R_{i,t} := \frac{1}{|\mathcal W_i(t)|}\sum_{t'\in\mathcal W_i(t)}R_{i,t'}, \qquad \text{defined if } |\mathcal W_i(t)|\ge 3 .
\label{eq:smooth}
\end{equation}
Let $\mathcal T_{ij}$ be the set of times at which both $\bar R_{i,t}$ and $\bar R_{j,t}$ are defined. The affinity is the Pearson correlation of the two smoothed residuals on this common support, centered by their means $\bar r_i$ and $\bar r_j$ over $\mathcal T_{ij}$:
\begin{equation}
\begin{aligned}
\bar r_i &:= \frac{1}{|\mathcal T_{ij}|}\sum_{t\in\mathcal T_{ij}}\bar R_{i,t}, \qquad
\bar r_j := \frac{1}{|\mathcal T_{ij}|}\sum_{t\in\mathcal T_{ij}}\bar R_{j,t},\\
C_{ij} &:= \frac{\sum_{t\in\mathcal T_{ij}}(\bar R_{i,t}-\bar r_i)(\bar R_{j,t}-\bar r_j)}
{\sqrt{\sum_{t\in\mathcal T_{ij}}(\bar R_{i,t}-\bar r_i)^2}\;\sqrt{\sum_{t\in\mathcal T_{ij}}(\bar R_{j,t}-\bar r_j)^2}}
\quad\text{if } |\mathcal T_{ij}|\ge n_{\min}, \qquad C_{ij}:=0 \ \text{otherwise},
\end{aligned}
\label{eq:affinity}
\end{equation}
with $C_{ii}:=0$ and $n_{\min}:=\max(30,\min(100,T/10))$.

\subsection{The estimator in full}
\label{app:backbone-full}

\paragraph{Temporal Laplacian.} Let $L_{\rm path}\in\R^{T\times T}$ be the Laplacian of the path graph that links each time step to the next, and $L_T:=\tfrac14L_{\rm path}$, scaled to spectral radius below one, so that
\begin{equation}
\langle X, XL_T\rangle_F = \frac14\sum_{i=1}^{N}\sum_{t=1}^{T-1}(X_{i,t+1}-X_{i,t})^2 .
\label{eq:temporal-dirichlet}
\end{equation}

We state the estimator for general nonnegative cell weights $W$ and an explicit spatial weight $\lambda_S$, with fidelity term $\tfrac12\sum_{(i,t)} W_{i,t}(X_{i,t} - Y_{i,t})^2$, which is that of \eqref{eq:tikh} for binary $W$. All experiments use the plain mask of the current solve, $W=M'$, and absorb $\lambda_S$ into $\alpha$. Setting the gradient of the objective to zero gives the normal equations
\begin{equation}
W \odot X + \lambda_S L_S X + \lambda_T X L_T + \mu X = W \odot Y,
\label{eq:opt}
\end{equation}
whose operator $Z \mapsto W\odot Z + \lambda_S L_S Z + \lambda_T Z L_T + \mu Z$ is symmetric positive definite for $\mu > 0$ (Lemma~\ref{lem:wellposed}); the solution is unique and is computed by conjugate gradients at $O(\mathrm{nnz}(L_S)\,T + NT)$ per iteration, warm-started and Jacobi-preconditioned. Section~\ref{sec:method-tying} instantiates \eqref{eq:opt} with $L_S(\alpha)$, $\lambda_S=1$, $W=M'$ and the residual data $\tilde Y-P$ in place of $Y$. The convergence bound of Lemma~\ref{lem:wellposed} is stated for plain conjugate gradients; with the Jacobi preconditioner, $\kappa$ is that of the diagonally scaled operator.

\subsection{Hyperedge Laplacians as shared-weight graphs}
\label{app:dirichlet}

Throughout, for $e \subseteq [N]$ with $s = |e| \ge 2$, $P_e := I_e - \tfrac1s \1_e\1_e^\top$ is the hyperedge Laplacian of Section~\ref{sec:setup-laplacian}, embedded in $\R^{N\times N}$ and zero outside $e \times e$. This is the convention of the implementation; the size-scaled operator $sP_e$ used in part of the hypergraph literature differs by a constant that the unit-spectral-radius normalization within an order removes.

\begin{lemma}[Hyperedge Laplacian as a shared-weight graph]
\label{lem:clique}
\label{lem:dirichlet}
$P_e$ is symmetric, idempotent and positive semidefinite, with $P_e\1_e = 0$ and $\operatorname{rank}(P_e) = s-1$, and for any $X \in \R^{N\times T}$
\[
\langle X, P_e X\rangle_F \;=\; \sum_{i\in e}\|X_{i,:} - \bar X_{e,:}\|_2^2 \;=\; \frac1s\sum_{\{i,j\}\subseteq e}\|X_{i,:} - X_{j,:}\|_2^2,
\qquad \bar X_{e,:} := \frac1s\sum_{i\in e}X_{i,:}.
\]
Let $L_{ij} := (e_i - e_j)(e_i-e_j)^\top$ and $A^{(e)} := \tfrac1s(\1_e\1_e^\top - I_e)$, the adjacency that places weight $1/s$ on every pair of $e$. Then
\[
P_e \;=\; \frac1s\sum_{\{i,j\}\subseteq e} L_{ij} \;=\; D_{A^{(e)}} - A^{(e)} .
\]
Consequently every hypergraph Laplacian $L_{\Hcal} = \sum_{e} w_e P_e$ is the graph Laplacian $D_{A_{\Hcal}} - A_{\Hcal}$ of the weighted adjacency $A_{\Hcal} = \sum_e w_e A^{(e)}$, that is $(A_{\Hcal})_{ij} = \sum_{e \ni i,j} w_e/|e|$ as stated in Section~\ref{sec:setup-laplacian}, and the mixture $L_S(\alpha)$ of \eqref{eq:mohl-family} has the weighted adjacency \eqref{eq:comembership}.
\end{lemma}
\begin{proof}
On the coordinates of $e$, $\tfrac1s\1_e\1_e^\top$ projects onto $\mathrm{span}(\1_e)$, so $P_e$ is the complementary projector, with the stated properties. Expanding the squares with $\sum_{i\in e}X_{i,:}=s\bar X_{e,:}$ gives both forms of $\langle X,P_eX\rangle_F$. Next, $\sum_{\{i,j\}\subseteq e}L_{ij}$ has diagonal entries $s-1$ on $e$ and off-diagonal entries $-1$ on $e\times e$, so it equals $sI_e - \1_e\1_e^\top = sP_e$; dividing by $s$ gives the first equality. Each vertex of $A^{(e)}$ has degree $(s-1)/s$, so $D_{A^{(e)}} - A^{(e)} = \tfrac{s-1}{s}I_e - \tfrac1s(\1_e\1_e^\top - I_e) = I_e - \tfrac1s\1_e\1_e^\top = P_e$. Both $D_{(\cdot)}$ and the identity map are linear in the adjacency, and a nonnegative combination of graph Laplacians is the graph Laplacian of the combined adjacency, which gives the consequence. For \eqref{eq:comembership}, the adjacency of $\sum_{e\in\Hcal_s}P_e$ places weight $1/s$ on a pair for each group of $\Hcal_s$ containing it, that is $\tfrac1s K^{(s)}$, and $L_{\Hcal_s}$ divides this sum by its spectral radius $\rho_s$.
\end{proof}

\subsection{The evidence-adaptive gate and its gradient}
\label{app:gate-adjoint}

\paragraph{Context vector.} The gate uses the context vector
\begin{equation}
c_{i,t} := \bigl(c^{\rm gap}_{i,t},\ c^{\rm dark}_{i,t},\ \sin 2\pi\phi_t,\ \cos 2\pi\phi_t,\ \sin 4\pi\phi_t,\ \cos 4\pi\phi_t,\ c^{\rm state}_{t}\bigr)^{\top} \in \R^{7}.
\label{eq:context}
\end{equation}
Each feature is computed from the mask and the observed cells only:
\begin{itemize}[leftmargin=0.25in,itemsep=1pt,topsep=2pt]
\item $c^{\rm gap}_{i,t}:=\log(1+\min(g_{i,t},n_{\rm day}))/\log(1+n_{\rm day})\in[0,1]$ is the temporal gap \eqref{eq:gap-evidence} computed on $M'$, capped at one day and log-scaled: $0$ on an observed cell and $1$ a day or more from the nearest reading.
\item $c^{\rm dark}_{i,t}:=\sum_jA_{ij}(1-M'_{j,t})/\sum_jA_{ij}\in[0,1]$ is the share of sensor $i$'s prior-graph neighbors that are unobserved at time $t$, weighted by the prior adjacency.
\item $\phi_t\in[0,1)$ is the time of day as a fraction of the day, so the four harmonics describe daily and half-daily patterns.
\item $c^{\rm state}_{t}\in[-1,1]$ is the network state: the mean standardized reading over the sensors observed under $M'$ at time $t$, averaged over one hour, clipped to $[-3,3]$ and divided by $3$.
\end{itemize}

\paragraph{Temporal gate.} The gate of the temporal link between $(i,t)$ and $(i,t+1)$ uses the slopes $b_T$ and the average context of the two times:
\begin{equation}
h^{T}_{i,t} := \exp\bigl(\mathrm{clip}(b_T^\top \bar c_{i,t},\,-\kappa,\,\kappa)\bigr), \qquad \bar c_{i,t} := \tfrac12\,(c_{i,t}+c_{i,t+1}).
\label{eq:tgate}
\end{equation}

\paragraph{Matrix form of the gated operator.} Let $L_k = D_k - A_k$ with $D_k = \operatorname{diag}(\deg_k)$ and $\deg_k = A_k\1$, and let $H_k = [h^k_{i,t}]$ hold the gate values of \eqref{eq:egate}. The gated spatial component acts as $S_k(Z; H_k) = \tfrac12[\,H_k\odot(L_kZ) + (A_kH_k)\odot Z - A_k(H_k\odot Z)\,]$. The temporal term carries the gate $h^{T}_{i,t}$ of \eqref{eq:tgate} on each temporal edge, $(ZL_T(h^{T}))_{i,t}=\tfrac14\bigl[h^{T}_{i,t}(Z_{i,t}-Z_{i,t+1})-h^{T}_{i,t-1}(Z_{i,t-1}-Z_{i,t})\bigr]$, with out-of-range terms omitted. The full operator is $\mathcal A Z = W\odot Z + \sum_k \alpha_k S_k(Z;H_k) + \lambda_T Z L_T(h^{T}) + \mu Z$. Half its quadratic form is the gated regularizer \eqref{eq:gated-reg} plus the fidelity and ridge terms, so it is symmetric and $\mathcal A\succeq\mu I$. For every $t$ and every gate value, the gated component is a graph Laplacian: symmetric, positive semidefinite and with zero row sums.

Since $\langle\Psi, S_k(\widehat X;H_k)\rangle_F = \sum_{i,t} h^k_{i,t}Q^k_{i,t}$
with $Q^k = \tfrac12[\Psi\odot\widehat X\odot\deg_k - \Psi\odot(A_k\widehat X) -
\widehat X\odot(A_k\Psi) + A_k(\Psi\odot\widehat X)]$, the single adjoint solve of
\eqref{eq:adjoint} gives $\partial\mathcal L/\partial b_k = -\alpha_k\sum_{i,t}
h^k_{i,t}Q^k_{i,t}\,c_{i,t}$ wherever the clip is inactive. For the temporal block,
with $Q^{T}_{i,t}:=\tfrac14(\Psi_{i,t}-\Psi_{i,t+1})(\widehat X_{i,t}-\widehat X_{i,t+1})$,
$\partial\mathcal L/\partial\log\lambda_T=-\lambda_T\sum_{i,t<T}h^{T}_{i,t}Q^{T}_{i,t}$ and
$\partial\mathcal L/\partial b_T=-\lambda_T\sum_{i,t<T}h^{T}_{i,t}Q^{T}_{i,t}\,\bar c_{i,t}$.
The ridge of \eqref{eq:mohl-bilevel} adds $\gamma_b b_k$ and $\gamma_b b_T$.

\subsection{Optimization and deployment}
\label{app:optimization}

We minimize $\mathcal J$ with the Adam optimizer and, after every step, clip each parameter to a box $\Theta$:
\begin{equation}
\alpha_k\in[10^{-6},10^{4}], \qquad \lambda_T\in[10^{-4},10^{4}], \qquad \mu\in[10^{-5},10^{2}], \qquad b_k,\,b_T\in[-8,8]^7 .
\label{eq:box}
\end{equation}
The lower bound on $\mu$ guarantees $\mathcal A(\theta)\succeq10^{-5}I$, so every solve stays well posed. The fit is warm-started in stages, each adding one part of the model: the profile and prior graph, then pairs ($s=2$), then orders $3$ to $S$ (MoHL-Static), then the gate (MoHL). Each new order starts at a quarter of the prior-graph weight, and each stage keeps its iterate with the lowest objective. Algorithm~\ref{alg:mohl} summarizes the complete procedure, from structure discovery through learning to imputation.

\begin{algorithm}[!ht]
\algrenewcommand\algorithmicrequire{\textbf{Input:}}
\algrenewcommand\algorithmicensure{\textbf{Output:}}
\caption{\textbf{MoHL}: imputation from an incomplete matrix.}
\label{alg:mohl}
\small
\begin{algorithmic}[1]
\Require readings $Y$ on $\Omega$, mask $M$, prior adjacency $A$, largest order $S$
\Ensure imputation $\widehat X^{\rm out}$ on the missing cells $\Omega^c$
\Statex \textit{Structure from a mask $M'$}
\Function{Structure}{$M'$}
\State profile $P$ and residual $R$ ($P^{\rm fit}$ and $R^{\rm fit}$ when $M'=M^{\rm fit}$)
\State affinity $C$ and groups $\Hcal_2,\dots,\Hcal_S$
\State components $L_{\rm prior}, L_{\Hcal_2},\dots,L_{\Hcal_S}$ and gate context $c$
\EndFunction
\Statex \textit{Learning}
\State standardize each sensor; split $\Omega$ into $\Omega_{\rm fit}$, $\Omega_{\rm learn}$, $\Omega_{\rm sel}$
\State \Call{Structure}{$M^{\rm fit}$}; initialize $\theta$
\For{stage $\in$ (profile and prior graph, pairs, orders $3$--$S$, gate)}
\Repeat
\State forward solve: $\mathcal A(\theta)\widehat X=M^{\rm fit}\odot R^{\rm fit}$
\State validation errors $u$ \eqref{eq:validation-error} and loss gradient $G$ \eqref{eq:adjoint}
\State adjoint solve: $\mathcal A(\theta)\Psi=G$
\State update $\theta$ by Adam with $\nabla_\theta\mathcal J$; clip to $\Theta$
\Until{step budget reached; keep the $\theta$ with the lowest $\mathcal J$}
\EndFor
\State score the kept $\widehat\theta$ on $\Omega_{\rm sel}$
\Statex \textit{Imputation}
\State \Call{Structure}{$M$}; with the kept parameters $\widehat\theta$, solve $\mathcal A(\widehat\theta)\widehat X=M\odot R$
\State \Return $\widehat X^{\rm out}_{i,t}=\bar Y_i+s_i(\widehat X_{i,t}+P_{i,t})$ for $(i,t)\in\Omega^c$
\end{algorithmic}
\end{algorithm}
\section{Proofs}
\label{app:mohl-proofs}

We use the following conventions.
\begin{itemize}[leftmargin=0.25in,itemsep=1pt,topsep=2pt]
\item $\langle X,Y\rangle_F=\operatorname{tr}(X^\top Y)$; $e_i$ is the $i$-th standard basis vector of $\R^N$; $X_{i,:}$ is row $i$ of $X$; $\rho(\cdot)$ is the spectral radius.
\item $\1_e$ and $I_e$ are the all-ones vector and the identity restricted to the coordinates of $e$.
\item $L_{ij}:=(e_i-e_j)(e_i-e_j)^\top$ is the Laplacian of a unit pair edge, with $\langle X,L_{ij}X\rangle_F=\|X_{i,:}-X_{j,:}\|_2^2$.
\item Spatial operators act by left multiplication and the temporal operator by right multiplication.
\item A linear map $\mathcal B$ on $\R^{N\times T}$ is symmetric if $\langle Z,\mathcal BY\rangle_F=\langle\mathcal BZ,Y\rangle_F$, and positive definite if also $\langle Z,\mathcal BZ\rangle_F>0$ for $Z\ne0$.
\end{itemize}
Subsections whose assumptions traffic data need not meet close with a check on the PEMS-BAY data and masks of Table~\ref{tab:tied}.

\subsection{Well-posedness and solver cost}
\label{app:proof-wellposed}

The statement, summarized in Section~\ref{sec:setup-backbone}:

\begin{lemma}[Well-posedness and solver cost]
\label{lem:wellposed}
Let $L_S, L_T \succeq 0$, $\lambda_S,\lambda_T\ge0$, $\mu > 0$, and $W \ge 0$ entrywise with $w_{\max} := \max_{i,t} W_{i,t}$. The operator $\mathcal A Z := W \odot Z + \lambda_S L_S Z + \lambda_T Z L_T + \mu Z$ of \eqref{eq:opt} is symmetric positive definite with spectrum in $[\mu,\; w_{\max} + \lambda_S\rho(L_S) + \lambda_T\rho(L_T) + \mu]$, so \eqref{eq:opt} has a unique solution for every mask, and conjugate gradients started from $X_0$ satisfy
\begin{equation}
\|X_k - \widehat X\|_{\mathcal A} \le 2\Bigl(\frac{\sqrt\kappa - 1}{\sqrt\kappa + 1}\Bigr)^{k}\|X_0 - \widehat X\|_{\mathcal A},
\qquad
\kappa \le \frac{w_{\max} + \lambda_S\rho(L_S) + \lambda_T\rho(L_T) + \mu}{\mu},
\label{eq:kappa}
\end{equation}
at a cost of $O(\mathrm{nnz}(L_S)\,T + NT)$ operations per iteration.
\end{lemma}

\begin{proof}
\emph{Step 1: symmetry.} The map $Z\mapsto W\odot Z$ is symmetric because
$\langle Y,W\odot Z\rangle_F=\sum_{it}W_{it}Y_{it}Z_{it}$ is symmetric in
$(Y,Z)$. For symmetric $L_S$,
$\langle Y,L_SZ\rangle_F=\operatorname{tr}(Y^\top L_SZ)
=\operatorname{tr}((L_SY)^\top Z)=\langle L_SY,Z\rangle_F$. For symmetric
$L_T$, $\langle Y,ZL_T\rangle_F=\operatorname{tr}(Y^\top ZL_T)
=\operatorname{tr}(L_TY^\top Z)=\operatorname{tr}((YL_T)^\top Z)
=\langle YL_T,Z\rangle_F$. The identity map is symmetric. Hence $\mathcal A$,
a nonnegative combination of symmetric maps, is symmetric.

\emph{Step 2: lower bound.} For any $Z$,
\[
\langle Z,\mathcal AZ\rangle_F
=\sum_{it}W_{it}Z_{it}^2
+\lambda_S\langle Z,L_SZ\rangle_F
+\lambda_T\langle Z,ZL_T\rangle_F
+\mu\|Z\|_F^2 .
\]
The first term is nonnegative because $W\ge0$. Writing $Z=[z_1,\dots,z_T]$ by
columns, $\langle Z,L_SZ\rangle_F=\sum_tz_t^\top L_Sz_t\ge0$ since
$L_S\succeq0$; writing $Z$ by rows $\check z_1,\dots,\check z_N$,
$\langle Z,ZL_T\rangle_F=\sum_i\check z_iL_T\check z_i^\top\ge0$ since
$L_T\succeq0$. Hence $\langle Z,\mathcal AZ\rangle_F\ge\mu\|Z\|_F^2$.

\emph{Step 3: upper bound.} By the same decompositions and the Rayleigh
quotient, $\sum_tz_t^\top L_Sz_t\le\rho(L_S)\|Z\|_F^2$ and
$\sum_i\check z_iL_T\check z_i^\top\le\rho(L_T)\|Z\|_F^2$, while
$\sum_{it}W_{it}Z_{it}^2\le w_{\max}\|Z\|_F^2$. Therefore
\[
\mu\|Z\|_F^2
\le\langle Z,\mathcal AZ\rangle_F
\le\bigl(w_{\max}+\lambda_S\rho(L_S)+\lambda_T\rho(L_T)+\mu\bigr)\|Z\|_F^2 .
\]

\emph{Step 4: spectrum, uniqueness, conditioning.} Identify $\R^{N\times T}$
with $\R^{NT}$; $\mathcal A$ is then a symmetric matrix whose Rayleigh
quotient lies in the interval of Step~3, so all of its eigenvalues do. In
particular $\mathcal A\succ0$, it is invertible, and \eqref{eq:opt} has the
unique solution $\widehat X=\mathcal A^{-1}(W\odot Y)$. Because the objective with
fidelity $\tfrac12\sum_{(i,t)}W_{i,t}(X_{i,t}-Y_{i,t})^2$ is a strictly convex quadratic with gradient
$\mathcal AX-W\odot Y$, this solution is its unique minimizer. The
condition number satisfies
$\kappa=\lambda_{\max}(\mathcal A)/\lambda_{\min}(\mathcal A)
\le(w_{\max}+\lambda_S\rho(L_S)+\lambda_T\rho(L_T)+\mu)/\mu$.

\emph{Step 5: conjugate gradients.} For a symmetric positive-definite system
with condition number $\kappa$, the $k$-th conjugate-gradient iterate
minimizes the $\mathcal A$-norm of the error over the $k$-th Krylov subspace,
and the classical Chebyshev-polynomial argument yields
$\|X_k-\widehat X\|_{\mathcal A}\le2\bigl((\sqrt\kappa-1)/(\sqrt\kappa+1)\bigr)^k
\|X_0-\widehat X\|_{\mathcal A}$. Consequently
$k\ge\frac12\sqrt\kappa\,\ln(2/\varepsilon)$ iterations suffice for a relative
error of $\varepsilon$ in the $\mathcal A$-norm.

\emph{Step 6: cost per iteration.} One application of $\mathcal A$ consists of
an entrywise product ($NT$ operations), a sparse left multiplication by $L_S$
($\operatorname{nnz}(L_S)\,T$ operations), a right multiplication by the
tridiagonal $L_T$ ($O(NT)$ operations), and a scaling ($NT$ operations), for a
total of $O(\operatorname{nnz}(L_S)\,T+NT)$.
\end{proof}

\paragraph{Spectral radii of the MoHL operator.} The path Laplacian
$L_{\rm path}$ on $T$ vertices has eigenvalues $2-2\cos(\pi k/T)$,
$k=0,\dots,T-1$, so $L_T=\tfrac14L_{\rm path}$ satisfies $\rho(L_T)<1$. For
the MoHL operator \eqref{eq:mohl-estimator}, in which $\lambda_S$ is absorbed
into the weights, the triangle inequality for the spectral norm and
$\rho(L_{\rm prior})=\rho(L_{\Hcal_s})=1$ give
$\rho(L_S(\alpha))\le\sum_k\alpha_k$, so
\[
\kappa\le\frac{w_{\max}+\sum_k\alpha_k+\lambda_T+\mu}{\mu}.
\]
Learning drives $\mu$ down, which improves accuracy but worsens conditioning. It is therefore $\mu$, not the size of the Laplacian dictionary, that governs the number of iterations. With the mask weights used in all experiments, $w_{\max}=1$.

\subsection{Pairwise versus group observation rate}
\label{app:proof-rate}

\begin{lemma}[Pairwise versus group observation rate]
\label{lem:pairwise-rate}
Under an i.i.d.\ Bernoulli observation model with rate $\pi$, the number of
jointly observed samples for a fixed pair is
$n_{ij}\sim\mathrm{Binomial}(T,\pi^2)$, whereas for a fixed group of size $s$
it is $n_e\sim\mathrm{Binomial}(T,\pi^s)$. For any $0<\delta<1$,
$\Pr[n_{ij}\le(1-\delta)T\pi^2]\le e^{-\delta^2T\pi^2/2}$ and
$\Pr[n_e\le(1-\delta)T\pi^s]\le e^{-\delta^2T\pi^s/2}$. The effective sample
sizes are therefore of order $T\pi^2$ and $T\pi^s$, and their ratio is
$\pi^{-(s-2)}$.
\end{lemma}

\begin{proof}[Proof of Lemma~\ref{lem:pairwise-rate}]
Suppose the mask entries $M_{it}$ are independent Bernoulli variables with
$\Pr[M_{it}=1]=\pi\in(0,1]$.

\emph{Pair counts.} For a fixed pair $(i,j)$ with $i\ne j$ define
$B_t=M_{it}M_{jt}\in\{0,1\}$. Since $M_{it}$ and $M_{jt}$ are independent,
$\Pr[B_t=1]=\pi^2$; and since $B_t$ depends only on the mask entries in
column $t$, the variables $B_1,\dots,B_T$ are independent. Therefore
\[
n_{ij}=\sum_{t=1}^TB_t\sim\mathrm{Binomial}(T,\pi^2),
\qquad
\E n_{ij}=T\pi^2 .
\]

\emph{Group counts.} For a fixed group $e$ of size $s$ define
$G_t=\prod_{i\in e}M_{it}$. By independence across sensors
$\Pr[G_t=1]=\pi^s$, and by independence across columns the $G_t$ are
independent, so
\[
n_e=\sum_{t=1}^TG_t\sim\mathrm{Binomial}(T,\pi^s),
\qquad
\E n_e=T\pi^s .
\]

\emph{Concentration.} Let $n$ be a sum of independent Bernoulli variables
with $\E n=\bar n$. The multiplicative Chernoff bounds state that for
$0<\delta<1$,
\[
\Pr[n\le(1-\delta)\bar n]\le\exp\bigl(-\delta^2\bar n/2\bigr),
\qquad
\Pr[n\ge(1+\delta)\bar n]\le\exp\bigl(-\delta^2\bar n/3\bigr).
\]
Applying the lower-tail bound with $\bar n=T\pi^2$ and with $\bar n=T\pi^s$ gives
the two inequalities of the lemma. In particular, with probability at least
$1-e^{-\delta^2T\pi^2/2}$ a pair retains at least $(1-\delta)T\pi^2$ jointly
observed samples, whereas for a group the exponent is $\delta^2T\pi^s/2$, smaller by the
factor $\pi^{s-2}$.

\emph{Ratio.} The expected counts satisfy
$\E n_{ij}/\E n_e=\pi^{2}/\pi^{s}=\pi^{-(s-2)}$. Equivalently, holding the
expected number of usable samples fixed at $n_0$ requires a window of length
$T=n_0\pi^{-2}$ for pairwise statistics and $T=n_0\pi^{-s}$ for statistics
that need all $s$ members simultaneously.
\end{proof}

\paragraph{Scope.} The Chernoff bounds require only independence, so they hold verbatim for cell-varying rates $\pi_{it}$, with $\E n_{ij}=\sum_t\pi_{it}\pi_{jt}\ge\E n_e=\sum_t\prod_{l\in e}\pi_{lt}$ for $\{i,j\}\subseteq e$. The experimental masks are not independent, so the lemma only explains the scaling gap. The dominance $n_e\le n_{ij}$ for $\{i,j\}\subseteq e$, however, holds for any mask, since a time at which all members of $e$ are observed is one at which $i$ and $j$ are.

\begin{figure}[!ht]
\centering
\includegraphics[width=0.69\linewidth]{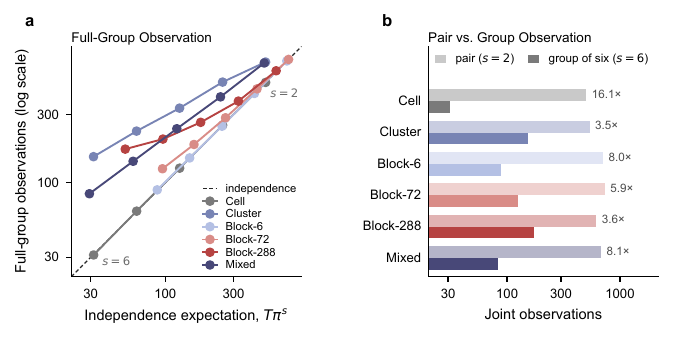}
\caption{\textbf{Why affinities are pairwise (Lemma~\ref{lem:pairwise-rate})}: (a) how often discovered groups are fully observed, against independent missingness, and (b) pairs against groups of six, on each regime's mask.}
\label{fig:check-rates}
\end{figure}

\paragraph{Observation rates on real masks.} Lemma~\ref{lem:pairwise-rate} assumes independent missingness, under which a pair is observed together far more often than a whole group. Real gaps cluster in space and time, which favors whole groups, and indeed groups are fully observed more often than independence predicts under the clustered, long-gap and mixed masks (Figure~\ref{fig:check-rates}a). Their pairs are nonetheless observed together $3.5$ to $16$ times as often as groups of six on every mask (Figure~\ref{fig:check-rates}b). Pairwise affinities thus remain the better-supported basis for discovering groups, even under strongly dependent missingness.

\subsection{Proof of Theorem~\ref{thm:comembership}}
\label{app:comembership}

Write
$\Hcal_s$ for the set of retained order-$s$ groups (Section~\ref{sec:method-groups}),
$L_s:=\sum_{e\in\Hcal_s}P_e$ for their unnormalized aggregate, so that the
order-$s$ component of \eqref{eq:mohl-family} is
$L_{\Hcal_s}=L_s/\rho_s$ with $\rho_s:=\rho(L_s)$, and
$L_{\rm prior}=(D_A-A)/\rho_0$ with $\rho_0:=\rho(D_A-A)$. For
$i\ne j$ and an order $s$ let
\[
K^{(s)}_{ij}=\bigl|\{e\in\Hcal_s:\{i,j\}\subseteq e\}\bigr|
\]
be the number of order-$s$ groups that contain both sensors, let $K^{(s)}$ be
the symmetric matrix of these counts with zero diagonal, and let
$\mathcal E=\{\{i,j\}:K^{(s)}_{ij}>0\text{ for some }s\}$ be the relational support, with
$m=|\mathcal E|$; let $\Phi$ be the matrix whose rows are the co-membership vectors $\varphi(i,j)$ of all
pairs with a nonzero vector (the pairs of $\mathcal E$ and the prior edges), and abbreviate the adjacency of \eqref{eq:comembership} as
$A(\alpha)$. An empty order has no component, so $\alpha$ ranges over the nonempty orders.

\begin{proof}[Proof of Theorem~\ref{thm:comembership}]
\emph{Part (i).} By Lemma~\ref{lem:clique},
$L_s=\tfrac1s\sum_{e\in\Hcal_s}\sum_{\{i,j\}\subseteq e}L_{ij}$, every group
of $\Hcal_s$ having size $s$. Exchanging the two sums, the pair Laplacian
$L_{ij}$ appears once for every group of $\Hcal_s$ that contains $\{i,j\}$,
so $L_s=\tfrac1s\sum_{i<j}K^{(s)}_{ij}L_{ij}$. By the identity
$D_B-B=\sum_{i<j}B_{ij}L_{ij}$ for a symmetric nonnegative $B$ with zero
diagonal (which holds entrywise: the diagonal of both sides is the row sum of
$B$, the off-diagonal entries are $-B_{ij}$), this equals
$\tfrac1s(D_{K^{(s)}}-K^{(s)})$. Substituting into \eqref{eq:mohl-family}, using
$L_{\rm prior}=(D_A-A)/\rho_0$ and the linearity of $B\mapsto D_B-B$, gives
\eqref{eq:comembership}.

\emph{Part (ii).} By \eqref{eq:comembership},
$A(\alpha)_{ij}=\langle\beta,\varphi(i,j)\rangle$ with
$\beta=(\alpha_0/\rho_0,\alpha_2/(2\rho_2),\dots,\alpha_S/(S\rho_S))$. Equal
co-membership vectors therefore give equal weights for every $\alpha$. Conversely, if two
vectors differ in some coordinate, choosing $\alpha$ supported on that
coordinate gives the two pairs different weights.

\emph{Part (iii).} Order the columns of $\Phi$ as $(S,S-1,\dots,2,0)$, where
$0$ is the prior column, and select the rows
$(p_S,p_{S-1},\dots,p_2,p_0)$, where $p_s$ is the pair that first appears at
order $s$ and $p_0$ is the prior edge contained in no group, omitting empty
orders. Row $p_s$ has a positive entry in column $s$, zeros in the columns
$s'<s$ by the first-appearance condition, and a zero in the prior column
because $p_s$ is off the prior graph; row $p_0$ is zero except for a positive
entry in the prior column. The selected square submatrix is therefore lower
triangular with a positive diagonal, hence invertible, so $\Phi$ has full
column rank. Then $A(\alpha)$ determines $\beta$ and, since every $\rho_k>0$
is known, determines $\alpha$.
\end{proof}

\begin{proposition}[Cone structure of the shared family]
\label{prop:cone}
The shared family $\{A(\alpha):\alpha\ge0\}$ is a polyhedral cone of
dimension $\mathrm{rank}(\Phi)\le S$, and its relational part is contained in
the closure of the free family $\{A_{\rm shared}\odot c:c\in(0,\infty)^{\mathcal E}\}$
of \eqref{eq:untied}, which has dimension $m$.
\end{proposition}

\begin{proof}
The map $\beta\mapsto\Phi\beta$ is linear from $\R^S$ to the
space of pair weights, and $\beta$ ranges over the nonnegative orthant as
$\alpha$ does, because each $\rho_k>0$. The image of an orthant under a
linear map is a polyhedral cone whose dimension is the rank of the map. For
the containment, let $B=\sum_{s\ge2}\beta_sK^{(s)}$ be the relational part of
a shared adjacency. It is supported on $\mathcal E$, and the initial shared adjacency
$A_{\rm shared}$ is strictly positive on $\mathcal E$, since it is of the same form with
all coefficients positive. Hence $B=A_{\rm shared}\odot c$ with
$c_{ij}=B_{ij}/(A_{\rm shared})_{ij}\ge0$ on $\mathcal E$, and $B$ is the limit of
$A_{\rm shared}\odot c^{(r)}$ with $c^{(r)}=c+1/r>0$. The free family is the
image of $(0,\infty)^{\mathcal E}$ under an injective linear map, so it has dimension
$m$.
\end{proof}

\paragraph{Reading Theorem~\ref{thm:comembership}.} Part (ii) is the precise meaning of structured weight sharing for the operator MoHL deploys. A free graph on the relational support has one parameter per pair, whereas MoHL has one per co-membership coordinate, so all pairs with the same co-membership vector share a weight. For a single chain, $K^{(s)}_{v_av_b}=\mathbf 1[s\ge\max(a,b)]$ and $\rho_s=1$, and \eqref{eq:comembership} reduces to Corollary~\ref{cor:spectral-more}(a). Proposition~\ref{prop:cone} supplies the inclusion used in Corollary~\ref{cor:crossover}, and part (iii) licenses reading the learned coupling profile (Figure~\ref{fig:check-identify}d).

\begin{figure}[!ht]
\centering
\includegraphics[width=0.76\linewidth]{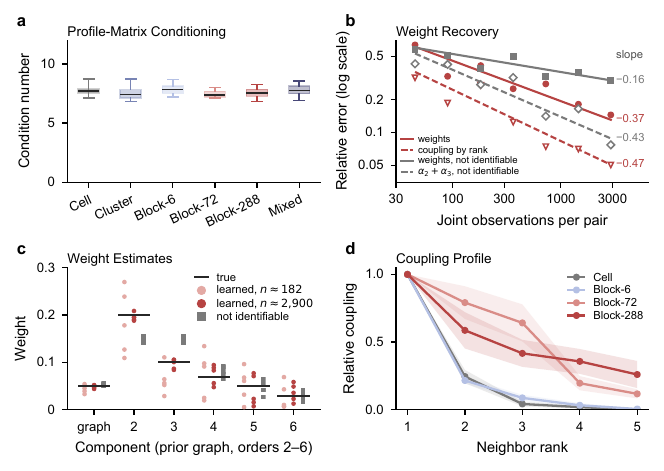}
\caption{\textbf{Learning the order weights (Theorem~\ref{thm:comembership}(iii))}: (a) conditioning of the co-membership matrix, (b, c) recovery of known weights with a non-identifiable control, and (d) the learned coupling of each neighbor relative to the first.}
\label{fig:check-identify}
\end{figure}

\paragraph{Identifiability of the coupling profile.} Reading the learned coupling profile is meaningful only if the data determine it. They determine the profile well, but single weights only coarsely. The condition of Theorem~\ref{thm:comembership}(iii) holds for every discovered group set, with a well-conditioned co-membership matrix (Figure~\ref{fig:check-identify}a). To measure how much data recovery takes, we fix known weights on the Cell groups, draw data from the Gaussian model of Theorem~\ref{thm:factor}(b) under a Cell mask, and learn the weights back over windows of $1.75$ to $112$ days. Single order weights converge slowly, because the components of orders four to six act on nearly the same pairs, so the data separate their weights only weakly, whereas the coupling profile $\Lambda_k$, which pools the orders, converges at about the inverse square root of the window length (Figure~\ref{fig:check-identify}b). When two orders share their groups, their weights stop converging while their sum still converges (Figure~\ref{fig:check-identify}c). On the real data, $\Lambda_4/\Lambda_2$ stays below $0.1$ under short gaps but exceeds $0.4$ under six-hour and day-long gaps, so the profile reaches further down each sensor's neighbor ranking as gaps lengthen (Figure~\ref{fig:check-identify}d).

\subsection{The chain spectrum}
\label{app:proof-spectral}

On one anchor's chain, the unit that Definition~\ref{def:rig} builds, the spectrum of the mixture is exact: nested sets give nested projectors, $P_eP_{e'} = P_{e'}P_e = P_e$ for $e \subseteq e'$, so one basis diagonalizes every mixture at once.

\begin{corollary}[Shared weights are a monotone coupling profile]
\label{thm:spectral}
Let $e_2 \subset \cdots \subset e_S$ be any sets with $|e_k| = k$, written $e_k = \{v_1, \dots, v_k\}$; let $L_{\rm ch}(\alpha) := \sum_{s=2}^{S} \alpha_s P_{e_s}$ with $\alpha \ge 0$, and let $h_k := \bigl(\tfrac{1}{k-1}\sum_{j<k} e_{v_j}\bigr) - e_{v_k}$ compare the $k$-th member against the mean of the earlier ones ($e_v$ the standard basis vector). Then:
\begin{itemize}[leftmargin=0.25in,itemsep=1pt,topsep=2pt]
\item[(i)] \textbf{One basis for the whole family.} For every $\alpha \ge 0$, $\{\1_{e_S}, h_2, \dots, h_S\}$ is an orthogonal eigenbasis of $L_{\rm ch}(\alpha)$ on the coordinates of $e_S$, with $L_{\rm ch}(\alpha)\1_{e_S} = 0$ and
\begin{equation}
L_{\rm ch}(\alpha)\,h_k = \Lambda_k\, h_k, \qquad \Lambda_k := \sum_{s=k}^S \alpha_s .
\label{eq:cumulative-spectrum}
\end{equation}
Each $P_{e_s}$ has spectral radius one, so $\Lambda_k$ is the coupling profile \eqref{eq:coupling-profile} of the chain.
\item[(ii)] \textbf{Exact range.} The family is precisely the set of operators that are diagonal in this basis, vanish on $\1_{e_S}$ and have a nonincreasing nonnegative spectrum on the $h_k$,
\[
\bigl\{L_{\rm ch}(\alpha) : \alpha \ge 0\bigr\}
=
\Bigl\{\,\textstyle\sum_{k=2}^{S} \Lambda_k\, \tfrac{h_k h_k^\top}{\|h_k\|_2^2} \;:\; \Lambda_2 \ge \dots \ge \Lambda_S \ge 0\,\Bigr\},
\]
and the correspondence is a bijection with $\alpha_s = \Lambda_s - \Lambda_{s+1}$ ($\Lambda_{S+1} := 0$): the shared weights are exactly the increments of a coupling profile that is nonincreasing in neighbor rank.
\end{itemize}
\end{corollary}

\begin{proof}[Proof of Corollary~\ref{thm:spectral}]
\emph{Nested projectors.} Let $e\subseteq e'$ with $|e|=s$ and $|e'|=s'$. The
inclusion gives $I_eI_{e'}=I_e$, $I_e\1_{e'}=\1_e$ and
$\1_e^\top\1_{e'}=s$, so $P_e\1_{e'}=\1_e-\tfrac1s\1_e\,s=0$ and hence
\[
P_eP_{e'}=P_eI_{e'}-\tfrac1{s'}\,(P_e\1_{e'})\1_{e'}^\top=P_e .
\]
The right-hand side is symmetric, so
$P_{e'}P_e=(P_eP_{e'})^\top=P_e=P_eP_{e'}$. Any two members of the chain
$e_2\subset\dots\subset e_S$ are nested, so the components of
$L_{\rm ch}(\alpha)$ commute pairwise.

\emph{Part (i): the contrasts are orthogonal.} Each $h_k$ is supported on
$e_k\subseteq e_S$ and satisfies
$\1^\top h_k=(k-1)\cdot\frac1{k-1}-1=0$, so $h_k\perp\1_{e_S}$. For $k<l$ the
support of $h_k$ lies in $e_k\subseteq e_{l-1}$, on which $h_l$ is constant
and equal to $1/(l-1)$; hence
$\langle h_k,h_l\rangle=\frac1{l-1}\,\1^\top h_k=0$. Finally
$\|h_k\|_2^2=(k-1)\cdot\frac1{(k-1)^2}+1=\frac{k}{k-1}>0$. Thus
$\{\1_{e_S},h_2,\dots,h_S\}$ consists of $S$ nonzero, mutually orthogonal
vectors in the $S$-dimensional coordinate subspace of $e_S$, and is an
orthogonal basis of it.

\emph{Part (i): action of one component.} Fix $s$ and $k$. If $s\ge k$, then
$\mathrm{supp}(h_k)\subseteq e_k\subseteq e_s$, so $I_{e_s}h_k=h_k$ and
$\1_{e_s}^\top h_k=\1^\top h_k=0$, giving $P_{e_s}h_k=h_k$. If $s<k$, then
$e_s\subseteq e_{k-1}$, on which $h_k$ equals the constant $1/(k-1)$; hence
$I_{e_s}h_k=\frac1{k-1}\1_{e_s}$ and $\1_{e_s}^\top h_k=\frac s{k-1}$, giving
$P_{e_s}h_k=\frac{1}{k-1}\1_{e_s}-\frac1s\cdot\frac{s}{k-1}\1_{e_s}=0$.
Summing over $s$,
\[
L_{\rm ch}(\alpha)\,h_k=\sum_{s=k}^{S}\alpha_s\,h_k=\Lambda_kh_k .
\]
Also $P_{e_s}\1_{e_S}=\1_{e_s}-\tfrac1s\1_{e_s}(\1_{e_s}^\top\1_{e_S})=0$ for
every $s$, so $L_{\rm ch}(\alpha)\1_{e_S}=0$, and $L_{\rm ch}(\alpha)$
vanishes on every coordinate outside $e_S$. Since $\alpha_s\ge0$,
$\Lambda_k-\Lambda_{k+1}=\alpha_k\ge0$ and $\Lambda_S=\alpha_S\ge0$, which is
the ordering of part (ii).

\emph{Part (ii).} The relations $\Lambda_S=\alpha_S$ and
$\Lambda_k-\Lambda_{k+1}=\alpha_k$ for $k<S$ define a triangular linear
system with unit diagonal, so $\alpha\mapsto\Lambda$ is a linear bijection
of $\R^{S-1}$ with the stated inverse. Under it, $\alpha\ge0$ holds if and
only if $\Lambda_S\ge0$ and $\Lambda_k\ge\Lambda_{k+1}$ for all $k<S$, which is
the isotonic cone. Because a symmetric matrix is determined by its eigenpairs,
\[
\bigl\{L_{\rm ch}(\alpha):\alpha\ge0\bigr\}
=\Bigl\{\sum_{k=2}^{S}\Lambda_k\frac{h_kh_k^\top}{\|h_k\|_2^2}:
\Lambda_2\ge\dots\ge\Lambda_S\ge0\Bigr\},
\]
the set of operators that are diagonal in the basis of part (i), vanish on
$\1_{e_S}$ and have a nonincreasing nonnegative spectrum on the $h_k$.
\end{proof}

\begin{corollary}[Further properties of a nested mixture]
\label{cor:spectral-more}
In the setting of Corollary~\ref{thm:spectral}, for every $\alpha\ge0$:
(a) the pairwise weight $L_{\rm ch}(\alpha)$ places between $v_a$ and $v_b$ is $\sum_{s\ge\max(a,b)}\alpha_s/s$;
(b) for a single fully observed snapshot ($\lambda_T=0$) with uniform weight $w>0$, the estimate $\hat x=((w+\mu)I+L_{\rm ch}(\alpha))^{-1}wx$ satisfies $\langle h_k,\hat x\rangle=\frac{w}{w+\mu+\Lambda_k}\langle h_k,x\rangle$ and $\langle\1_{e_S},\hat x\rangle=\frac{w}{w+\mu}\langle\1_{e_S},x\rangle$;
(c) the prior $\exp(-\tfrac12x^\top L_{\rm ch}(\alpha)x)$ is, for $\Lambda_k>0$, the sequential random-effects model $x_{v_k}\mid x_{v_1},\dots,x_{v_{k-1}}\sim\mathcal N\bigl(\bar x_{k-1},\tfrac{k}{(k-1)\Lambda_k}\bigr)$, with $\bar x_{k-1}$ the mean of the $k-1$ closer sensors and a flat prior on $x_{v_1}$.
\end{corollary}

\begin{proof}
\emph{Part (a).} The off-diagonal entry of $P_{e_s}$ at $(v_a,v_b)$, $a\ne b$,
is $-1/s$ if both vertices lie in $e_s$, that is, if $s\ge\max(a,b)$, and $0$
otherwise. The weight of the pair in the adjacency $-\mathrm{offdiag}
(L_{\rm ch}(\alpha))$ is therefore $\sum_{s\ge\max(a,b)}\alpha_s/s$.

\emph{Part (b).} On the coordinates of $e_S$ the matrix
$(w+\mu)I+L_{\rm ch}(\alpha)$ has the eigenvectors of Corollary~\ref{thm:spectral}(i) with eigenvalues
$w+\mu$ on $\1_{e_S}$ and $w+\mu+\Lambda_k$ on $h_k$; it is invertible because
$w>0$. Taking the inner product of
$((w+\mu)I+L_{\rm ch}(\alpha))\hat x=wx$ with each eigenvector and using the
symmetry of the matrix gives the two stated identities.

\emph{Part (c).} By Corollary~\ref{thm:spectral}(i) and the spectral
decomposition of a symmetric matrix,
\[
x^\top L_{\rm ch}(\alpha)x
=\sum_{k=2}^{S}\Lambda_k\frac{\langle h_k,x\rangle^2}{\|h_k\|_2^2}
=\sum_{k=2}^{S}\Lambda_k\,\frac{k-1}{k}\,\bigl(x_{v_k}-\bar x_{k-1}\bigr)^2,
\]
because $\langle h_k,x\rangle=\bar x_{k-1}-x_{v_k}$ with
$\bar x_{k-1}=\frac1{k-1}\sum_{j<k}x_{v_j}$ and $\|h_k\|_2^2=k/(k-1)$. Hence
\[
\exp\Bigl(-\frac12x^\top L_{\rm ch}(\alpha)x\Bigr)
=\prod_{k=2}^{S}\exp\Bigl(-\frac{(k-1)\Lambda_k}{2k}
\bigl(x_{v_k}-\bar x_{k-1}\bigr)^2\Bigr).
\]
For $\Lambda_k>0$ the $k$-th factor, viewed as a function of $x_{v_k}$ with
$x_{v_1},\dots,x_{v_{k-1}}$ fixed, is proportional to the
$\mathcal N(\bar x_{k-1},k/((k-1)\Lambda_k))$ density, and its normalizing
constant does not depend on $x$. A product of conditional densities, each
depending only on earlier coordinates, is the joint density of the
corresponding sequential model, here with a flat (improper) density on
$x_{v_1}$. If $\Lambda_k=0$ the factor is constant and $x_{v_k}$ is
unconstrained, which is the limit of infinite conditional variance.

\end{proof}

\paragraph{Scope.} Corollaries~\ref{thm:spectral} and~\ref{cor:spectral-more} are exact for a single chain, and for any chains with disjoint supports. In MoHL the chains of different anchors overlap, so their Laplacians need not commute; Theorem~\ref{thm:comembership} is the exact statement for the full operator, and Corollary~\ref{cor:spectral-more}(a) is its single-chain special case.

\subsection{Proof of Theorem~\ref{thm:factor}}
\label{app:proof-factor}

\begin{proof}
\emph{Part (a).} Let $\bar x_e=\frac1s\sum_{i\in e}x_i$. For every $\nu\in\R$,
$\sum_{i\in e}(x_i-\nu)^2=\sum_{i\in e}(x_i-\bar x_e)^2+s(\nu-\bar x_e)^2$, since
the cross term $2(\bar x_e-\nu)\sum_{i\in e}(x_i-\bar x_e)$ vanishes. With the
Gaussian integral
$\int_\R e^{-\alpha s(\nu-\bar x_e)^2/2}\,d\nu=(2\pi/(\alpha s))^{1/2}$,
\[
\int_\R\exp\Bigl(-\frac\alpha2\sum_{i\in e}(x_i-\nu)^2\Bigr)d\nu
=\Bigl(\frac{2\pi}{\alpha s}\Bigr)^{1/2}
\exp\Bigl(-\frac{\alpha}{2}\sum_{i\in e}(x_i-\bar x_e)^2\Bigr).
\]
By Lemma~\ref{lem:dirichlet} applied to a single column,
$x^\top P_ex=\sum_{i\in e}(x_i-\bar x_e)^2$, so the exponential on the right
equals $\exp(-\frac\alpha2x^\top P_ex)$, and rearranging gives
Theorem~\ref{thm:factor}(a). The integrand is, up to a constant, the joint
density of $(x_i)_{i\in e}$ given the level $\nu$ under
$x_i=\nu+\varepsilon_i$ with independent
$\varepsilon_i\sim\mathcal N(0,1/\alpha)$; integrating against the flat
prior on $\nu$ gives the marginal.

\emph{Part (b).} Write $\alpha^\circ$, $\lambda_T$ and $\mu$ for the true parameters, and let $\mathcal PZ=L_S(\alpha^\circ)Z+\lambda_TZL_T+\mu Z$. By
Lemma~\ref{lem:wellposed} and $\mu>0$, $\mathcal P$ is symmetric
positive definite, so $X\sim\mathcal N(0,\mathcal P^{-1})$ is a proper prior
with density proportional to $\exp(-\frac12\langle X,\mathcal PX\rangle_F)$.
With $Y$ standing for the residual data $\tilde Y-P$ of \eqref{eq:mohl-estimator}, the likelihood of the observed cells is proportional to
$\exp(-\frac12\sum_{it}W_{it}(X_{it}-Y_{it})^2)$ with $W_{it}=0$ on
unobserved cells. The log posterior is therefore, up to a constant,
\[
-\frac12\langle X,W\odot X\rangle_F-\frac12\langle X,\mathcal PX\rangle_F
+\langle X,W\odot Y\rangle_F
=-\frac12\langle X,\mathcal AX\rangle_F+\langle X,W\odot Y\rangle_F ,
\]
a Gaussian log density with precision $\mathcal A$ and mean
$\mathcal A^{-1}(W\odot Y)$, which is the MoHL estimate \eqref{eq:mohl-estimator}. The
posterior mean minimizes the posterior expected squared error, hence the
Bayes risk for squared error, over all measurable estimators. Each marginal
posterior is a univariate Gaussian, which is symmetric about its mean, so the
mean is also the posterior median and minimizes the posterior expected
absolute error in every cell; this is the loss reported in our experiments.
Any other estimator, including the resolvent estimator with an arbitrary
free-weight operator on any support, therefore has a Bayes risk that is at least
as large.
\end{proof}

\paragraph{Noise precision.} The estimator uses the mask $W=M'$, which is the noise precision of Theorem~\ref{thm:factor}(b) when the observed residuals have unit noise variance. For a noise precision $\tau$ on every observed cell, the posterior mean solves $\tau M'\odot X+\mathcal PX=\tau M'\odot Y$. Dividing by $\tau$ gives the MoHL estimate with $(\alpha,\lambda_T,\mu)$ divided by $\tau$, so the theorem holds up to a reparameterization that the learned parameters absorb.

\subsection{Shared versus free weights}
\label{app:tying-proof}

We prove Proposition~\ref{prop:tying} and Corollary~\ref{cor:scarcity} of Section~\ref{sec:method-untied}, in its notation. In the proof, $X\in\R^{n\times m}$ has rows
$x_l^\top$, $y = X\beta^* + \varepsilon$, and $Q\in\R^{m\times r}$ is an
orthonormal basis of $U$, so $\Pi_U = QQ^\top$.

\begin{proof}[Proof of Proposition~\ref{prop:tying}]
\emph{Step 1: a Wishart identity.} If $Z\in\R^{n\times k}$ has i.i.d.\
$\mathcal N(0,I_k)$ rows and $n>k+1$, then $Z^\top Z\sim\mathrm{Wishart}_k(n,I_k)$
and $\E[(Z^\top Z)^{-1}] = I_k/(n-k-1)$ \citep[Chapter~3]{muirhead1982}, so
$\E\operatorname{tr}\bigl((Z^\top Z)^{-1}\bigr) = k/(n-k-1)$.

\emph{Step 2: the free estimator.} With $X$ of full column rank almost
surely, $\hat\beta^{\rm free}-\beta^*=(X^\top X)^{-1}X^\top\varepsilon$.
Conditioning on $X$ and using $\E[\varepsilon\varepsilon^\top]=\sigma^2I_n$,
\[
\E\bigl[\|\hat\beta^{\rm free}-\beta^*\|_2^2\bigm|X\bigr]
=\sigma^2\operatorname{tr}\bigl((X^\top X)^{-1}\bigr),
\]
and Step~1 with $k=m$ gives $\E\mathcal R_{\rm ex}(\hat\beta^{\rm free})=m\sigma^2/(n-m-1)$.

\emph{Step 3: the shared estimator.} Least squares over $U$ is
$\hat\beta^{\rm shared}=Q\hat\gamma$ with
$\hat\gamma=(Z^\top Z)^{-1}Z^\top y$ and $Z:=XQ\in\R^{n\times r}$. Decompose
$\beta^*=Q\gamma^*+\beta^*_\perp$ with $\gamma^*=Q^\top\beta^*$
and $\beta^*_\perp=(I-QQ^\top)\beta^*$, $\|\beta^*_\perp\|_2^2=\zeta^2$.
Then $y=Z\gamma^*+\tilde\varepsilon$ with $\tilde\varepsilon_l=x_l^\top\beta^*_\perp+\varepsilon_l$.
Because $x_l$ is isotropic Gaussian, $Q^\top x_l$ and $(I-QQ^\top)x_l$ are
independent, so $\tilde\varepsilon$ is independent of $Z$ with i.i.d.\
$\mathcal N(0,\sigma^2+\zeta^2)$ entries: the part of $\beta^*$ that $U$
cannot express acts as additional noise. Hence, as in Step~2,
$\E\|\hat\gamma-\gamma^*\|_2^2=(\sigma^2+\zeta^2)\,r/(n-r-1)$ by Step~1 with
$k=r$. Finally $\hat\beta^{\rm shared}-\beta^*=Q(\hat\gamma-\gamma^*)-\beta^*_\perp$
with the two terms orthogonal, so
\[
\E\mathcal R_{\rm ex}(\hat\beta^{\rm shared})
=\E\|\hat\gamma-\gamma^*\|_2^2+\zeta^2
=\zeta^2+\frac{r(\sigma^2+\zeta^2)}{n-r-1}.
\]

\emph{Step 4: comparison.} Sharing wins when
$\zeta^2\bigl(1+\tfrac{r}{n-r-1}\bigr)<\sigma^2\bigl(\tfrac{m}{n-m-1}-\tfrac{r}{n-r-1}\bigr)$.
The left factor is $\zeta^2\,\tfrac{n-1}{n-r-1}$ and the right bracket equals
$\tfrac{(m-r)(n-1)}{(n-m-1)(n-r-1)}$, so after canceling $\tfrac{n-1}{n-r-1}>0$
the condition is exactly $\zeta^2<(m-r)\sigma^2/(n-m-1)$.
\end{proof}

\begin{proof}[Proof of Corollary~\ref{cor:scarcity}]
For $n>m+1$ the condition of Step~4 is $\zeta^2(n-m-1)<(m-r)\sigma^2$, i.e.\
$n<m+1+(m-r)\sigma^2/\zeta^2$ when $\zeta^2>0$, and it holds for every $n$ when
$\zeta^2=0$. For $n\le m+1$ the expected free risk is infinite or the free
least-squares solution is not unique, so the comparison is made for $n>m+1$.
For independent Bernoulli masks the number of time steps at which a fixed
pair is jointly observed is $\mathrm{Binomial}(T,\pi^2)$ and concentrates
around $T\pi^2$ (Lemma~\ref{lem:pairwise-rate}); identifying $n$ with this
count gives $T_{\rm cross}$.
\end{proof}

\paragraph{Reading Proposition~\ref{prop:tying}.} Equation~\eqref{eq:tying-risks} is the bias--variance decomposition of a restricted least-squares estimator under a random design. The restriction to $U$ costs an approximation error $\zeta^2$ that does not shrink with $n$, and it replaces a variance of order $m\sigma^2/(n-m)$ by one of order $r(\sigma^2+\zeta^2)/(n-r)$. Three monotonicities follow from \eqref{eq:scarcity}. The crossover $n_{\rm cross}$ grows linearly with the total number of couplings $m$, so the more pairs one weight governs, the longer sharing remains preferable. It grows with the noise level $\sigma^2$. It shrinks with the mismatch $\zeta^2$: if the true couplings within an order are nearly equal, sharing is nearly costless and wins at every realistic sample size.

\begin{figure}[!ht]
\centering
\includegraphics[width=0.67\linewidth]{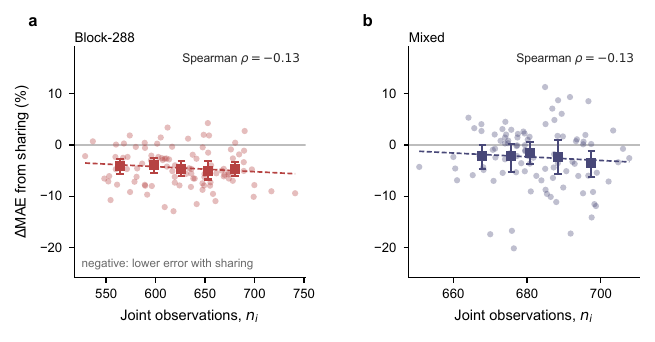}
\caption{\textbf{Sharing and joint evidence (Proposition~\ref{prop:tying})}: each sensor's change in MAE from free to shared weights against its joint evidence $n_i$, under (a) Block-288 and (b) Mixed.}
\label{fig:check-evidence}
\end{figure}

\paragraph{Sensor-level evidence.} Proposition~\ref{prop:tying} predicts that sharing helps most where the evidence per coupling is scarcest. Within a run, sensors differ in joint evidence while data, mask and training stay fixed, so we compare the sensors of the five Block-288 and five Mixed runs at $N=100$ of Table~\ref{tab:tied}, taking a sensor's joint evidence $n_i$ as the median number of joint observations with its discovered neighbors. Sharing lowers the error of the average sensor, but the size of its gain is not associated with $n_i$: the rank correlation is weak and changes sign across seeds (Figure~\ref{fig:check-evidence}), because $n_i$ varies little within one network. The predicted trend is borne out on data drawn from the model (Figure~\ref{fig:tied}b).

\subsection{The one-adjoint gradient}
\label{app:proof-adjoint}

Section~\ref{sec:method-learn} uses the adjoint gradients
\eqref{eq:mohl-adjoint} for every parameter. We record the statement and its
proof once, in the generality that also covers the free pair weights of the
free model and the gate parameters of Section~\ref{sec:gating}.

\begin{proposition}[One-adjoint gradient]
\label{prop:adjoint}
Let $\widehat X(\theta)$ solve $\mathcal A(\theta)\widehat X=f$, where
$\mathcal A(\theta)$ is symmetric positive definite and continuously
differentiable in $\theta$, and let $\mathcal L(\widehat X)$ be any
differentiable outer loss. Define the adjoint state $\Psi$ by
$\mathcal A(\theta)\Psi=\partial\mathcal L/\partial\widehat X$. Then for
any scalar parameter $\theta_r$ that enters $\mathcal A$ but not $f$,
\begin{equation}
\frac{\partial\mathcal L}{\partial\theta_r}
=-\Bigl\langle\Psi,\frac{\partial\mathcal A}{\partial\theta_r}\widehat X\Bigr\rangle_F .
\label{eq:adjoint-general}
\end{equation}
\end{proposition}

\begin{proof}
\emph{Step 1: differentiability of the solution map.} Define
$\mathcal F(\theta,X)=\mathcal A(\theta)X-f$. The map $\mathcal F$ is continuously
differentiable, and its partial derivative with respect to $X$ is the linear
map $\mathcal A(\theta)$, which is invertible because it is symmetric positive
definite. By the implicit function theorem the solution
$\widehat X(\theta)=\mathcal A(\theta)^{-1}f$ is continuously differentiable,
and differentiating $\mathcal A(\theta)\widehat X(\theta)=f$ with respect to a
scalar parameter $\theta_r$ that does not enter $f$ gives
\[
\frac{\partial\mathcal A}{\partial\theta_r}\widehat X
+\mathcal A\,\frac{\partial\widehat X}{\partial\theta_r}=0,
\qquad\text{hence}\qquad
\frac{\partial\widehat X}{\partial\theta_r}
=-\mathcal A^{-1}\Bigl(\frac{\partial\mathcal A}{\partial\theta_r}\widehat X\Bigr).
\]

\emph{Step 2: chain rule.} With $G=\partial\mathcal L/\partial\widehat X$,
$\partial\mathcal L/\partial\theta_r
=\langle G,\partial\widehat X/\partial\theta_r\rangle_F
=-\langle G,\mathcal A^{-1}(\partial_{\theta_r}\mathcal A\,\widehat X)\rangle_F$.

\emph{Step 3: moving the inverse.} The inverse of a symmetric invertible map
is symmetric: for any $U,V$, setting $U=\mathcal AU'$ and $V=\mathcal AV'$
gives $\langle U,\mathcal A^{-1}V\rangle_F=\langle\mathcal AU',V'\rangle_F
=\langle U',\mathcal AV'\rangle_F=\langle\mathcal A^{-1}U,V\rangle_F$. With
$\Psi=\mathcal A^{-1}G$ this turns Step~2 into
\eqref{eq:adjoint-general}. The adjoint state $\Psi$ does not depend on
$r$, so a single solve serves every parameter.
\end{proof}

\paragraph{Instantiation.} The static parameters give \eqref{eq:mohl-adjoint} directly. For the free model \eqref{eq:untied}, a single pair weight $(A_{\rm free})_{ij}$ enters
the spatial operator through $\alpha_1(A_{\rm free})_{ij}L_{ij}$, so
$\partial\mathcal L/\partial(A_{\rm free})_{ij}=-\alpha_1\langle\Psi,L_{ij}\widehat X\rangle_F
=-\alpha_1(\Xi_{ii}+\Xi_{jj}-\Xi_{ij}-\Xi_{ji})$ with $\Xi:=\Psi\widehat X^\top$,
followed by the derivative of the softplus reparameterization; all pair
gradients are read off one $N\times N$ matrix. The gate parameters enter
$\mathcal A$ through the time-varying edge weights \eqref{eq:egate} and are
differentiated in the same way, with the same $\Psi$
(Appendix~\ref{app:gate-adjoint}).

\paragraph{Right-hand side and cost.} For the objective \eqref{eq:mohl-bilevel}, $G$ has entries $\ell'_\delta(\widehat X_{it}-\tilde Y_{it}+P^{\rm fit}_{it})/|\Omega_{\rm learn}|$ on $\Omega_{\rm learn}$ and zero elsewhere, with $\ell'_\delta(u)=\max(-\delta,\min(\delta,u))$, and the $\ell_1$ term adds $\gamma\alpha_s$ to $\partial\mathcal J/\partial a_s$ directly. The adjoint right-hand side is $G$ itself (Proposition~\ref{prop:adjoint}), not $W\odot G$, which would vanish because $\Omega_{\rm learn}$ is disjoint from the fitting mask. One gradient evaluation costs the forward solve, one adjoint solve with the same symmetric operator, preconditioner and conditioning (Lemma~\ref{lem:wellposed}), and one inner product $\langle\Psi,\bar L_k\widehat X\rangle_F$ per component, instead of $p$ solves for $p$ parameters.

\subsection{Generalization of the learned operator}
\label{app:proof-generalization}

Apart from the per-sensor standardization, the structure and the solve depend only on $\Omega_{\rm fit}$, so the masked cells are a clean validation sample for the learned scalars. Because the estimate is a Lipschitz function of the parameters, the expected validation loss of the learned parameters, over fresh noise on the masked cells, comes close to the best penalized loss in the box.

\begin{theorem}[Generalization of the learned operator]
\label{thm:generalization}
Condition on the masks and on all data in $\Omega_{\rm fit}$, hold the per-sensor standardization fixed, and suppose the noise at the $n_{\rm learn} = |\Omega_{\rm learn}|$ masked cells is independent, mean-zero and $\sigma$-sub-Gaussian in standardized units. Let $\{\mathcal A(\theta) : \theta \in \Theta\}$ be an operator class on a box $\Theta \subset \R^p$ with $\mathcal A(\theta) \succeq \mu_{\min} I$ and $\|\mathcal A(\theta) - \mathcal A(\theta')\|_{\op} \le \Gamma_\Theta\|\theta - \theta'\|_1$, let $R(\theta)$ be the expected Huber validation loss of $\widehat X(\theta)$, and let $\hat\theta$ minimize the empirical validation loss plus any nonnegative penalty $\mathrm{pen}$ over $\Theta$, up to an optimization error $\varepsilon_{\rm opt}$. Then with probability at least $1 - \eta$,
\begin{equation}
R(\hat\theta) \le \inf_{\theta \in \Theta}\bigl\{R(\theta) + \mathrm{pen}(\theta)\bigr\} + C\,\delta\sigma\sqrt{\frac{p\log(1 + \kappa_\Theta) + \log(2/\eta)}{n_{\rm learn}}} + \varepsilon_{\rm opt},
\label{eq:generalization}
\end{equation}
where $C$ is an absolute constant and $\kappa_\Theta := p\,d_{\max} \Gamma_\Theta\,\|f\|_F/(\mu_{\min}^2\sigma)$, with $f$ the right-hand side of the solve and $d_{\max}$ the largest side of $\Theta$, enters only through the logarithm.
\end{theorem}

The right-hand side is the best risk in the class, an estimation error of order $\sqrt{p/n_{\rm learn}}$ and the optimization error, where $p$ counts the learned scalars and $\Gamma_\Theta$ bounds how fast the operator changes with them. For MoHL-Static, $p=S+2$, eight scalars at $S=6$, and $\mu_{\min}=10^{-5}$ on the box \eqref{eq:box} (Lemma~\ref{lem:classes}). The estimation error therefore depends on the number of learned scalars, not on the number of sensor pairs, and the data scale and conditioning enter only through the logarithm.

\begin{lemma}[Resolvent stability]
\label{lem:stability}
Let $\mathcal A,\mathcal A'$ be symmetric linear maps on $\R^{N\times T}$ with
$\mathcal A\succeq\mu I$ and $\mathcal A'\succeq\mu'I$, $\mu,\mu'>0$, and let
$f\in\R^{N\times T}$. Then
\[
\bigl\|\mathcal A^{-1}f-\mathcal A'^{-1}f\bigr\|_F
\le\frac{\|\mathcal A-\mathcal A'\|_{\rm op}}{\mu\mu'}\,\|f\|_F .
\]
\end{lemma}

\begin{proof}
Use the resolvent identity $\mathcal A^{-1}-\mathcal A'^{-1}=\mathcal A^{-1}(\mathcal A'-\mathcal A)\mathcal A'^{-1}$ with $\|\mathcal A^{-1}\|_{\rm op}\le1/\mu$ and $\|\mathcal A'^{-1}\|_{\rm op}\le1/\mu'$.
\end{proof}

\paragraph{Setting.} The masks and all data on $\Omega_{\rm fit}$ are fixed. The structure is built from the fitting mask only (Section~\ref{sec:method-validation}), and we treat the per-sensor standardization, which also uses the readings of the masked cells, as fixed. Then $\widehat X(\theta)=\mathcal A(\theta)^{-1}f$, with $f=M^{\rm fit}\odot R^{\rm fit}$, is deterministic for every $\theta$. On the masked cells the target is $Z_{it}:=\tilde Y_{it}-P^{\rm fit}_{it}=X^\circ_{it}+\xi_{it}$, where $X^\circ$ is the standardized, profile-removed truth and the noise is independent and mean-zero with $\|\xi_{it}\|_{\psi_2}\le\sigma$ in standardized units. The risks are $\widehat R(\theta)=n_{\rm learn}^{-1}\sum_{\Omega_{\rm learn}}\ell_\delta(\widehat X_{it}(\theta)-Z_{it})$ and $R=\E_\xi\widehat R$, and the box $\Theta=\prod_r[\theta_r^{\rm lo},\theta_r^{\rm hi}]$ has largest side $d_{\max}$.

\begin{proof}[Proof of Theorem~\ref{thm:generalization}]
\emph{Step 1: reduction.} Let $D_n=\sup_{\Theta}|\widehat R-R|$. For any $\theta^\circ\in\Theta$, since $\mathrm{pen}\ge0$ and $\hat\theta$ is $\varepsilon_{\rm opt}$-optimal,
\[
R(\hat\theta)\le\widehat R(\hat\theta)+\mathrm{pen}(\hat\theta)+D_n
\le\widehat R(\theta^\circ)+\mathrm{pen}(\theta^\circ)+\varepsilon_{\rm opt}+D_n
\le R(\theta^\circ)+\mathrm{pen}(\theta^\circ)+\varepsilon_{\rm opt}+2D_n ,
\]
so it suffices to bound $D_n$.

\emph{Step 2: concentration at a fixed $\theta$.} With $u_{it}=\widehat X_{it}(\theta)-X^\circ_{it}$, the summands of $\widehat R(\theta)$ are $\ell_\delta(u_{it}-\xi_{it})$. The Huber loss is $\delta$-Lipschitz, so each centered summand is sub-Gaussian with norm at most $C_0\delta\sigma$, and Hoeffding's inequality gives
\[
\Pr\bigl[|\widehat R(\theta)-R(\theta)|\ge t\bigr]\le2\exp\Bigl(-\frac{c\,n_{\rm learn}\,t^2}{\delta^2\sigma^2}\Bigr).
\]

\emph{Step 3: Lipschitz dependence on $\theta$.} Lemma~\ref{lem:stability} and the hypothesis on the class give $\|\widehat X(\theta)-\widehat X(\theta')\|_F\le L_\Theta\|\theta-\theta'\|_1$ with $L_\Theta:=\Gamma_\Theta\|f\|_F/\mu_{\min}^2$. Since $\sum_{\Omega_{\rm learn}}|v_{it}|\le\sqrt{n_{\rm learn}}\,\|v\|_F$, both $\widehat R$ and $R$ are $(\delta L_\Theta/\sqrt{n_{\rm learn}})$-Lipschitz in the $\ell_1$ norm.

\emph{Step 4: covering and union bound.} A grid with spacing $2\varepsilon/p$ on each side of $\Theta$ has at most $(1+p\,d_{\max}/(2\varepsilon))^p$ points and is an $\varepsilon$-net in $\ell_1$, so by Step~3 $D_n$ exceeds its maximum over the grid by at most $2\delta L_\Theta\varepsilon/\sqrt{n_{\rm learn}}$. With $\varepsilon=\sigma/(2L_\Theta)$ this excess is $\delta\sigma/\sqrt{n_{\rm learn}}$, and the grid has log-cardinality at most $p\log(1+\kappa_\Theta)$. A union bound of Step~2 over the grid gives, with probability at least $1-\eta$,
\[
D_n\le\delta\sigma\sqrt{\frac{p\log(1+\kappa_\Theta)+\log(2/\eta)}{c\,n_{\rm learn}}}+\frac{\delta\sigma}{\sqrt{n_{\rm learn}}}
\le\frac C2\,\delta\sigma\sqrt{\frac{p\log(1+\kappa_\Theta)+\log(2/\eta)}{n_{\rm learn}}},
\]
using $\log(2/\eta)\ge\log2$. Step~1 completes the proof.
\end{proof}

\begin{lemma}[Both model classes satisfy the hypotheses]
\label{lem:classes}
On any box on which $a_k\le\bar a$, $\log\lambda_T\le\bar l_T$ and
$\log\mu\in[\log\mu_{\min},\bar l_\mu]$, the MoHL-Static class
$\theta=(a_0,a_2,\dots,a_S,\log\lambda_T,\log\mu)$ satisfies the hypotheses of
Theorem~\ref{thm:generalization} with $p=S+2$ and
$\Gamma_\Theta=\max(e^{\bar a},e^{\bar l_T},e^{\bar l_\mu})$. On a box that in addition
bounds the pair logits, $|\vartheta_{ij}|\le\bar\vartheta$, the free class
$\theta=(a_0,a_1,(\vartheta_{ij})_{\{i,j\}\in\mathcal E},\log\lambda_T,\log\mu)$ satisfies
them with $p=m+4$ and a finite $\Gamma_\Theta$ depending on $\bar a$, $\bar l_T$, $\bar l_\mu$,
$\bar\vartheta$ and the row sums of $A_{\rm shared}$.
\end{lemma}

\begin{proof}
Both classes satisfy $\mathcal A(\theta)\succeq\mu_{\min}I$ by Lemma~\ref{lem:wellposed}. For MoHL-Static, $\mathcal A(\theta)-\mathcal A(\theta')$ is the sum of $(\alpha_k-\alpha_k')\bar L_k$, $(\lambda_T-\lambda_T')(\cdot)L_T$ and $(\mu-\mu')I$, whose operator norms are at most $|\alpha_k-\alpha_k'|$, $|\lambda_T-\lambda_T'|$ and $|\mu-\mu'|$ because $\rho(\bar L_k)=1$ and $\rho(L_T)<1$. The bound $|e^x-e^{x'}|\le e^{\bar x}|x-x'|$ for $x,x'\le\bar x$ gives the stated $\Gamma_\Theta$. For the free class, the spatial part is $\alpha_0\bar L_0+\alpha_1(D_B-B)$ with $B=A_{\rm free}(\vartheta)$. Gershgorin's theorem gives $\|D_B-B\|_{\rm op}\le2\max_i\sum_j|B_{ij}|$, and the softplus is $1$-Lipschitz, so $D_B-B$ is Lipschitz in $\vartheta$ and bounded on the box, with constants set by $\bar\vartheta$ and the entries and row sums of $A_{\rm shared}$. The product rule for $\alpha_1(D_B-B)$ then gives a finite $\Gamma_\Theta$.
\end{proof}

\begin{corollary}[Comparison of the shared and free guarantees]
\label{cor:crossover}
Let $R^{\min}_{\rm shared}$ and $R^{\min}_{\rm free}$ be the infima of $R$ over
the two classes on the same relational support, let
$\Delta=R^{\min}_{\rm shared}-R^{\min}_{\rm free}$, and let $\kappa_{\rm shared}$ and
$\kappa_{\rm free}$ be the constants $\kappa_\Theta$ of the two classes. If the closure of
the free class contains the shared class, that is, the free box contains the shared
ranges of $a_0$, $\log\lambda_T$ and $\log\mu$ and its pair weights reach the ratios
$c_{ij}$ of Proposition~\ref{prop:cone}, then $\Delta\ge0$. Writing
$\omega(p,\kappa)=C\delta\sigma\sqrt{(p\log(1+\kappa)+\log(2/\eta))/n_{\rm learn}}$
for the estimation term of \eqref{eq:generalization} and neglecting the
penalties and with a common optimization error, the bound of Theorem~\ref{thm:generalization} for the shared class is
smaller than the bound for the free class if and only if
\[
\Delta<\omega(m+4,\kappa_{\rm free})-\omega(S+2,\kappa_{\rm shared}).
\]
If moreover $(m+4)\log(1+\kappa_{\rm free})\ge4(S+2)\log(1+\kappa_{\rm shared})+3\log(2/\eta)$,
this holds for all $n_{\rm learn}<n_{\rm bound}:=c_0^2/\Delta^2$ when $\Delta>0$, with
$c_0:=\tfrac C2\,\delta\sigma\sqrt{(m+4)\log(1+\kappa_{\rm free})}$, and for every
$n_{\rm learn}$ when $\Delta=0$.
\end{corollary}

\begin{proof}
Every shared operator is a limit of free operators on the same support (Proposition~\ref{prop:cone}), and $R$ is continuous in the operator (Lemma~\ref{lem:stability}), so under the stated containment $\Delta\ge0$. Subtracting the two bounds $R^{\min}_{\rm shared}+\omega(S+2,\kappa_{\rm shared})$ and $R^{\min}_{\rm free}+\omega(m+4,\kappa_{\rm free})$ gives the condition. Its right-hand side is $C\delta\sigma(\sqrt a-\sqrt b)/\sqrt{n_{\rm learn}}$ with $a=(m+4)\log(1+\kappa_{\rm free})+\log(2/\eta)$ and $b=(S+2)\log(1+\kappa_{\rm shared})+\log(2/\eta)$. The extra condition is $a\ge4b$, so $\sqrt a-\sqrt b\ge\sqrt a/2$ and the right-hand side is at least $c_0/\sqrt{n_{\rm learn}}$. Finally, $\Delta<c_0/\sqrt{n_{\rm learn}}$ is equivalent to $n_{\rm learn}<c_0^2/\Delta^2$.
\end{proof}

Corollary~\ref{cor:crossover} is the counterpart of Proposition~\ref{prop:tying} at the level of guarantees: the shared bound is the tighter one while the number of masked cells is small, and the free bound can win only once that number is large.

\end{document}